\documentclass[sigconf]{acmart}

\usepackage{amsmath}       
\usepackage{optidef}       
\usepackage{algorithm}     
\usepackage[noend]{algpseudocode} 
\usepackage{subfigure}     
\usepackage{appendix}      
\usepackage{bm}            
\usepackage{amsfonts}      
\usepackage{xcolor}        
\usepackage{soul}          
\usepackage{enumitem}      
\usepackage{float}     
\usepackage{caption}   
\usepackage{subcaption}    
\usepackage{mwe}           
\usepackage{amsmath}       
\usepackage{amssymb}       
\usepackage{amsthm}        
\usepackage{pifont}        
\usepackage{multirow}      
\usepackage{graphicx}      
\newcommand{\randAE}{\emph{Adversarial Distribution}}

\newtheorem{thm}{Theorem}

\newtheorem{Lem}{Lemma}

\newtheorem{corollary}{Corollary}[thm] 

\usepackage{hyperref} 

\AtBeginDocument{%
  }

\copyrightyear{2024} 
\acmYear{2024} 
\setcopyright{acmlicensed}\acmConference[CCS '24]{Proceedings of the 2024 ACM SIGSAC Conference on Computer and Communications Security}{October 14--18, 2024}{Salt Lake City, UT, USA}
\acmBooktitle{Proceedings of the 2024 ACM SIGSAC Conference on Computer and Communications Security (CCS '24), October 14--18, 2024, Salt Lake City, UT, USA}
\acmDOI{10.1145/3658644.3690343}
\acmISBN{979-8-4007-0636-3/24/10}

\begin{document}

\title{Certifiable Black-Box Attacks with Randomized Adversarial Examples: Breaking Defenses with Provable Confidence}


\author{Hanbin Hong}
\affiliation{%
  \institution{University of Connecticut}
  \city{Storrs}
  \state{Connecticut}
  \country{USA}
}

\author{Xinyu Zhang}
\affiliation{%
  \institution{Zhejiang University}
  \city{Hangzhou}
  \state{Zhejiang}
  \country{China}
}

\author{Binghui Wang}
\affiliation{%
  \institution{Illinois Institute of Technology}
  \city{Chicago}
  \state{Illinois}
  \country{USA}
}

\author{Zhongjie Ba}
\affiliation{%
  \institution{Zhejiang University}
  \city{Hangzhou}
  \state{Zhejiang}
  \country{China}
}

\author{Yuan Hong}
\affiliation{%
  \institution{University of Connecticut}
  \city{Storrs}
  \state{Connecticut}
  \country{USA}
}

\begin{abstract}
Black-box adversarial attacks have demonstrated strong potential to compromise machine learning models by iteratively querying the target model or leveraging transferability from a local surrogate model.
Recently, such attacks can be effectively mitigated by state-of-the-art (SOTA) defenses, e.g., detection via the pattern of sequential queries, or injecting noise into the model. To our best knowledge, we take the first step to study a new paradigm of black-box attacks with provable guarantees -- certifiable black-box attacks that can guarantee the attack success probability (ASP) of adversarial examples before querying over the target model. This new black-box attack unveils significant vulnerabilities of machine learning models, compared to traditional empirical black-box attacks, e.g., breaking strong SOTA defenses with provable confidence, constructing a space of (infinite) adversarial examples with high ASP, and the ASP of the generated adversarial examples is theoretically guaranteed without verification/queries over the target model. Specifically, we establish a novel theoretical foundation for ensuring the ASP of the black-box attack with randomized adversarial examples (AEs). Then, we propose several novel techniques to craft the randomized AEs while reducing the perturbation size for better imperceptibility. Finally, we have comprehensively evaluated the certifiable black-box attacks on the CIFAR10/100, ImageNet, and LibriSpeech datasets, while benchmarking with 16 SOTA black-box attacks, against various SOTA defenses in the domains of computer vision and speech recognition. Both theoretical and experimental results have validated the significance of the proposed attack.\footnote{The code and all the benchmarks are available at \url{https://github.com/datasec-lab/CertifiedAttack}, and the full version can be accessed at \url{https://arxiv.org/abs/2304.04343}.}
\end{abstract}



\maketitle

\noindent \textbf{ACM Reference Format:}\\
Hanbin Hong, Xinyu Zhang, Binghui Wang, Zhongjie Ba, and Yuan Hong. 2024. Certifiable Black-Box Attacks with Randomized Adversarial Examples: Breaking Defenses with Provable Confidence. In \textit{Proceedings of the 2024 ACM SIGSAC Conference on Computer and Communications Security (CCS '24)}. ACM, Salt Lake City, UT, USA, 15 pages. https://doi.org/10.1145/3658644.3690343

\section{Introduction}
\label{sec:Introduction}
Machine learning (ML) models have achieved unprecedented success and have been widely integrated into many practical applications. However, 
it is well known that minor perturbations injected into the input data are sufficient to induce model misclassification~\cite{DBLP:conf/iclr/MadryMSTV18}. 
Many state-of-the-art (SOTA) adversarial attacks \cite{DBLP:conf/iclr/MadryMSTV18,DBLP:conf/sp/Carlini017,DBLP:conf/sp/ChenJW20,xie2022universal,DBLP:conf/iclr/KurakinGB17,DBLP:journals/corr/Moosavi-Dezfooli16,DBLP:conf/ccs/ChenZSYH17,DBLP:conf/icml/IlyasEAL18,DBLP:conf/eccv/AndriushchenkoC20,XieYH23,DBLP:conf/iclr/BrendelRB18,codebackdoor} have been proposed to explore the vulnerabilities of a variety of ML models. Wherein, the stringent \emph{black-box} attack is believed to be closer to real-world security practice \cite{DBLP:conf/ccs/PapernotMGJCS17,DBLP:conf/sp/ChenJW20}. 

In black-box attacks, the adversary only has access to the target ML model's outputs (either prediction scores or hard labels). Through iteratively querying the target model, the adversary progressively updates the perturbation until convergence. Existing black-box attack methods primarily utilize gradient estimation \cite{DBLP:conf/ccs/ChenZSYH17,DBLP:conf/icml/IlyasEAL18,DBLP:conf/iclr/ChengLCZYH19,DBLP:conf/eccv/BhagojiHLS18,DBLP:conf/ccs/DuFYCT18}, surrogate models \cite{DBLP:conf/ccs/PapernotMGJCS17,DBLP:conf/cvpr/ShiWH19,DBLP:conf/cvpr/DongPSZ19,DBLP:conf/cvpr/NaseerKHKP20}, or heuristic algorithms \cite{DBLP:conf/iclr/BrendelRB18,DBLP:conf/iccv/BrunnerDTK19,DBLP:conf/icml/LiLWZG19,DBLP:conf/icml/GuoGYWW19,DBLP:conf/eccv/AndriushchenkoC20} to generate adversarial perturbations. Although these attack algorithms can empirically achieve relatively high attack success rates (e.g., on CIFAR-10 \cite{krizhevsky2009learning}), their query process is shown to be easy to detect or interrupt due to the minor perturbation changes and high reliance on the previous perturbation~\cite{li2022blacklight,qin2021RAND,chandrasekaran2020RANDpost,chen2022AAA}. For example, ``Blacklight''~\cite{li2022blacklight} can achieve $100\%$ detection rate on most of the existing black-box attacks by checking the similarity of queries; some ``randomized defense'' methods~\cite{qin2021RAND,chen2022AAA,he2019parametric,liu2018towards,chandrasekaran2020RANDpost} inject random noise to the inputs, outputs, intermediate features or model parameters such that the performance of existing black-box attacks can be significantly degraded (since the query results are obfuscated to be unpredictable). 

To break such types of SOTA defenses \cite{li2022blacklight,qin2021RAND,chen2022AAA,he2019parametric,liu2018towards}, it is challenging to design an effective attack equipped with both \emph{high degree of randomness to bypass the strong detection} (e.g., Blacklight~\cite{li2022blacklight}) and \emph{high robustness to resist randomized defense}. A feasible solution is to add random noise to the adversarial example by the adversary, but it will make the query intractable. Therefore, an innovative method is desirable to carefully craft the adversarial example based on feedback from queries using randomly generated inputs. 

To this end, we propose a novel attack paradigm, termed \emph{Certifiable Attack}, that ensures a provable attack success probability (ASP) on the randomized adversarial examples against the equipped defenses (or no defense). Specifically, our attack strategy integrates random noise into the queries while preserving the adversarial efficacy of these queries. In particular, we model the adversarial examples as a random variable in the input space following an underlying noise distribution $\varphi$, namely ``\randAE''. Then, we design a novel query strategy and establish the theoretical foundation to guarantee the ASP of the distribution throughout the crafting process. A novel framework is also developed to find the initial \randAE, optimize it, and use it to sample the adversarial examples.

\begin{figure*}[!h]
    \centering
    \includegraphics[width=0.95\linewidth]{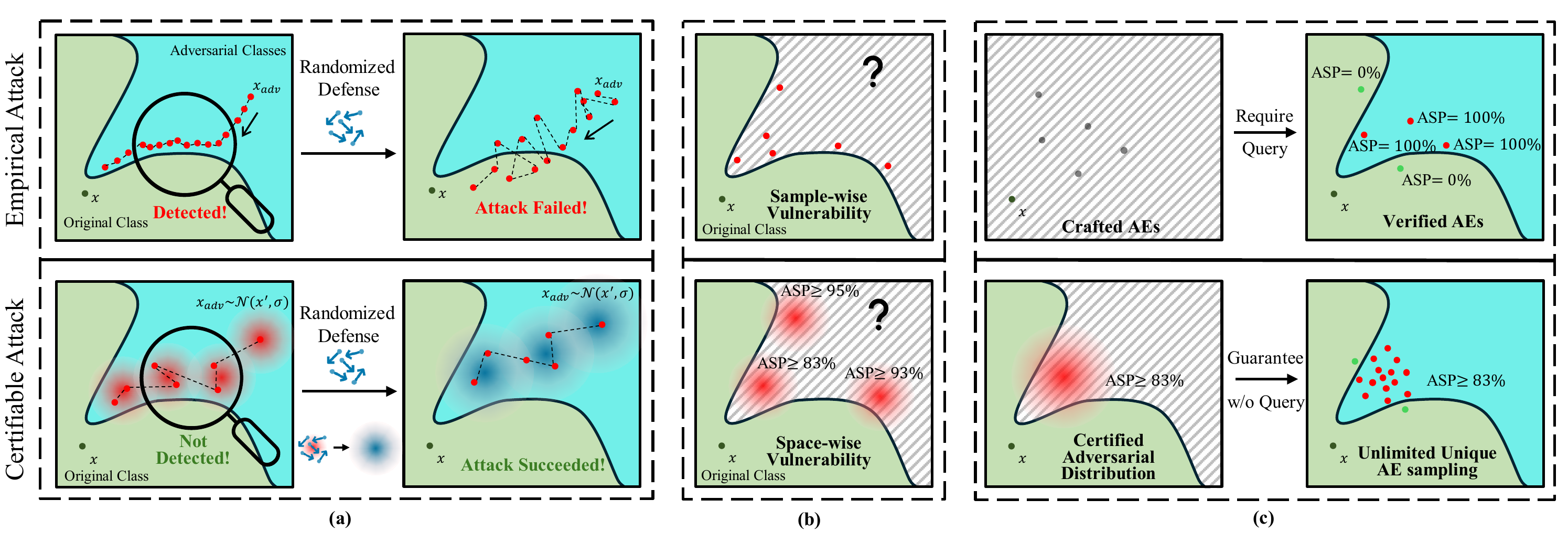}

    \caption{Empirical attacks vs. Certifiable attacks. (a) Certifiable attack can break the SOTA AE detection and randomized defenses. (b) Certifiable attack uncovers space-wise vulnerability rather than sample-wise vulnerability. (c) Once certified, Certifiable Attack can generate unlimited unique AEs with a guaranteed minimum ASP without querying the model for verification, while the empirical attack requires verifying the attack result of crafted AE by query. }
    \label{fig:certifiable attack}
\end{figure*}

\subsection{Certifiable Attacks vs. Empirical Attacks}

Compared with existing empirical black-box attacks, the
Certifiable Attack demonstrates multi-faceted advantages 
(also see Figure \ref{fig:certifiable attack}): 

\begin{enumerate}[label=(\alph*),leftmargin=*]
\setlength\itemsep{0.4em}

\item \textbf{Strong attack to break SOTA defenses.} The randomness in the certifiable attack allows it to effectively bypass detection methods that rely on the similarity between the attacker's sequence of queries (e.g., Blacklight \cite{li2022blacklight}), while traditional empirical attacks often create a suspicious trajectory of highly similar perturbations. The certifiable attack also provides a provable guarantee of success for attacks using randomized inputs, by taking into account the equipped defense and target model, enhancing its resistance to randomized defense \cite{qin2021RAND,chandrasekaran2020RANDpost}. 

\item {\bf Adversarial space vs. Adversarial example (AE).} Distinct from traditional empirical adversarial attacks, which uncover model vulnerabilities with sample-wise inputs, the Certifiable Attack seeks to explore an adversarial input space constructed by an \randAE. This continuous space facilitates the generation of numerous (potentially infinite) adversarial examples with a high ASP, thus revealing a more consistent and severe vulnerability of the target model.

\item \textbf{Adversarial Examples (AEs) sampled from the adversarial distribution are verification-free.} 
Empirical attacks search AEs by iteratively querying the target model and verifying the query outputs (\emph{the final successful AE is also used to query over the target model; then it will be verified and recorded by the defender/target model}). Instead, the certifiable attack crafts the \emph{adversarial distribution} with a guaranteed lower bound of the ASP. Due to the highly dimensional and continuous input space, AEs sampled from the adversarial distribution can be considered unique (with noise in all the dimensions) and have a negligible probability of being recorded by the defender/target model after verification. The ASP of such AEs are theoretically guaranteed (verification-free), and they are new to the defender, posing more challenges for mitigation. 
    
\end{enumerate}

\vspace{-0.1in}

\begin{table*}[]
\centering
\caption{Comparison of state-of-the-art empirical black-box attacks with certifiable attack}
\vspace{-0.1in}
\label{tab:black-box comparison}
\resizebox{\textwidth}{!}{%
\begin{tabular}{lcccccc}
\hline
Black-box Attacks &
  \begin{tabular}[c]{@{}c@{}}Query\\ Type\end{tabular} &
  \begin{tabular}[c]{@{}c@{}}Perturbation\\ Type\end{tabular} &
  \begin{tabular}[c]{@{}c@{}}ASP\\ Guarantee\end{tabular} &
  \begin{tabular}[c]{@{}c@{}}vs. Detection on\\ Attacker's Queries\end{tabular} &
  \begin{tabular}[c]{@{}c@{}}vs. Randomized\\ Pre-process. Defense\end{tabular} &
  \begin{tabular}[c]{@{}c@{}}vs. Randomized\\ Post-process. Defense\end{tabular} \\ \hline
Bandit \cite{Bandit}, NES \cite{DBLP:conf/icml/IlyasEAL18}, Parsimonious \cite{Parsimonious}, Sign \cite{Sign}, Square \cite{DBLP:conf/eccv/AndriushchenkoC20}, ZOSignSGD \cite{ZOSignSGD} &
  Score-based &
  $\ell_\infty$-bounded &
  $\times$ &
  $\times$ &
  $\times$ &
  $\checkmark$ \\
GeoDA \cite{rahmati2020geoda}, HSJ \cite{DBLP:conf/sp/ChenJW20}, Opt \cite{DBLP:conf/iclr/ChengLCZYH19}, RayS \cite{RayS}, SignFlip \cite{SignFlip}, SignOPT \cite{SignOpt} &
  Label-based &
  $\ell_\infty$-bounded &
  $\times$ &
  $\times$ &
  $\times$ &
  $\times$ \\
Bandit \cite{Bandit}, NES \cite{DBLP:conf/icml/IlyasEAL18}, Simple \cite{DBLP:conf/icml/GuoGYWW19}, Square \cite{DBLP:conf/eccv/AndriushchenkoC20}, ZOSignSGD \cite{ZOSignSGD} &
  Score-based &
  $\ell_2$-bounded &
  $\times$ &
  $\times$ &
  $\times$ &
  $\checkmark$ \\
Boundary \cite{DBLP:conf/iclr/BrendelRB18}, GeoDA \cite{rahmati2020geoda}, HSJ \cite{DBLP:conf/sp/ChenJW20}, Opt \cite{DBLP:conf/iclr/ChengLCZYH19}, SignOPT \cite{SignOpt} &
  Label-based &
  $\ell_2$-bounded &
  $\times$ &
  $\times$ &
  $\times$ &
  $\times$ \\
PointWise \cite{PointWise}, SparseEvo \cite{SparseEvo} &
  Label-based &
  Optimized &
  $\times$ &
  $\times$ &
  $\times$ &
  $\times$ \\ \hline
Certifiable Attack (ours) &
  Label-based &
  Optimized &
  $\checkmark$ &
  $\checkmark$ &
  $\checkmark$ &
  $\checkmark$ \\ \hline
\end{tabular}%
}\vspace{-0.1in}
\end{table*}

\subsection{Randomization for Certifiable Attacks}

To pursue certifiable attacks, we theoretically bound the ASP of \randAE~
based on a novel way of utilizing randomized smoothing \cite{DBLP:conf/icml/CohenRK19}, a technique achieving great success in the certified defenses with probabilistic guarantees.

The design for the randomization-based certifiable attack follows an intuitive goal, i.e., \emph{ensuring that the classification results are consistently wrong over the distribution}. However, many new significant challenges should be addressed. First, existing theories (randomization for certified defenses, e.g., \cite{DBLP:conf/icml/CohenRK19}) cannot be directly adapted to certifiable attacks since they have completely different goals and settings. Second, how to efficiently craft the \randAE~that can ensure the ASP is challenging since it requires maintaining the wrong prediction over a large number of randomized samples drawing from the distribution. Third, how to make the \randAE~ as imperceptible as possible is also challenging due to their randomness. By addressing these new challenges, in this paper, we make the following significant contributions:

\begin{enumerate}[label=\arabic*),leftmargin=*]
\setlength\itemsep{0.4em}
    
    \item To our best knowledge, we introduce the first \emph{certifiable attack theory} based on randomization for the black-box setting, which universally guarantees the attack success probability of AEs drawn from different noise distributions, e.g., Gaussian, Laplace, and Cauthy distributions, enabling a novel transition from deterministic to probabilistic adversarial attacks.
    
    \item We propose a novel \emph{certifiable attack framework} that can efficiently craft certifiable \randAE~ with provable ASP and imperceptibility. Specifically, we design a novel \emph{randomized parallel query method} to efficiently collect probabilistic query results 
    from any target model, which supports the certifiable attack theory. We propose a novel \emph{self-supervised localization} method as well as a binary-search localization method to efficiently generate certifiable \randAE. We design a novel \emph{geometric shifting} method to reduce the perturbation size for better imperceptibility while ensuring the ASP. Finally, we have validated that diffusion models \cite{ho2020denoising} can be used to further denoise the randomized AEs with guaranteed ASP.   
    
    \item We comprehensively evaluate the performance of the certifiable attack with different settings on 4 datasets, while benchmarking with 16 SOTA empirical black-box attacks, against various defenses. Experimental results consistently demonstrate that our certifiable attack effectively breaks the SOTA defenses, including adversarial detection, randomized pre-processing and post-processing defenses, as well as adversarial training defenses (Also, Table \ref{tab:black-box comparison} shows a summary of the certifiable attack vs. SOTA black-box attacks).
    
\end{enumerate}

\section{Problem Definition}

\label{sec:threat model}

\noindent {\bf Threat Model:} We consider designing a certifiable attack where the target model may or may not be protected by a defense mechanism. 

\begin{itemize}[leftmargin=*]
\setlength\itemsep{0.4em}

 \item \textbf{Adversary}: We focus on the hard-label black-box attack, where the adversary only knows the predicted label by querying the target ML model. The adversary's objective is to craft adversarial examples to fool the model based on the query results. 

 \item \textbf{Model Owner}: The model owner pursues the model utility. 
 We consider three different levels of the model owner's knowledge and capability: 1) The model owner has no awareness of the adversarial attacks and is not equipped with any defense; 2) The model owner is aware of the adversarial attack but has no knowledge of the attack method. The model owner can deploy general defense methods such as adversarial training \cite{DBLP:conf/iclr/MadryMSTV18}; 3) The model owner is aware of the adversarial attack and has knowledge about the attack method. The model owner can deploy adaptive defenses that are specifically designed for the attack.
 
\end{itemize}

\noindent {\bf Problem Formulation:} We first briefly review adversarial examples, and then formally define our problem.  Given an ML classifier $f$ and a testing data $x \in \mathbb{R}^d$ with label $y$ from a label set $\mathcal{Y}=[1, \cdots C]$ (where $C$ is the number of classes). An adversary carefully crafts a perturbation 
on the data $x$ 
such that the classifier $f$ misclassifies the perturbed data $x_{adv}$, i.e., $f(x_{adv})\neq y$ under $x_{adv}\in [\Pi_a,\Pi_b]^d$, 
where $[\Pi_a,\Pi_b]^d$ is the valid input space. The perturbed data $x_{adv}$ is called \emph{adversarial example}. Imperceptibility is usually achieved by restricting the $\ell_2$ or $\ell_\infty$ norm of the perturbation $x_{adv} - x$, or by minimizing the magnitude of this perturbation. 

In the black-box setting, an adversary can use \emph{empirical} black-box attack techniques (details in Section~\ref{sec:RelatedWork}) to iteratively query the classifier $f$ and progressively update the perturbation until finding a successful adversarial example for a testing example.  
However, such attack strategies have key limitations: 1) query inefficient, usually $>100$ queries per adversarial example; 2) easy to be detected by observing the query trajectory~\cite{li2022blacklight,qin2021RAND,chandrasekaran2020RANDpost,chen2022AAA}; and 3) lack of guaranteed attack performance, i.e., cannot provably guarantee a (un)successful adversarial example under a given budget.  

We aim to address all these limitations and design an efficient and effective certifiable black-box attack in the paper.
Particularly, instead of inefficiently searching adversarial \emph{examples} one-by-one, we want to certifiably find the underlying adversarial \emph{distribution} that the adversarial examples lie on. 

\begin{definition}[Certifiable black-box attack]
Given a classifier $f: \mathbb{R}^d \rightarrow \mathcal{Y}$, a clean input $x \in \mathbb{R}^d$ with label $y\in \mathcal{Y}$, and an Attack Success Probability Threshold $p$, the certifiable attack is to find an \randAE~ $\varphi(x', \boldsymbol{\kappa})$ with mean $x'$ and parameters $\boldsymbol{\kappa}$\footnote{If $\varphi$ is a Gaussian distribution, $\kappa$ is the standard deviation of $\varphi$. If $\varphi$ is a Generalized normal distribution, $\kappa=(a,b)$, with $a$ and $b$ the scale and shape parameters of $\varphi$, respectively. Notice that, the distribution will be applied to all the dimensions in the input, and \randAE~ is a noise distribution over the input space. }, such that data sampled from $\varphi$ have at least $p$ probability of being misclassified (i.e., adversarial examples). That is, 
\begin{align}
\label{eq:target}
   & \mathbb{P}_{x_{adv}\sim\varphi(x',\boldsymbol{\kappa})}[f(x_{adv})\neq y]\geq p \\
    &\text{s.t.}\, \, x_{adv} \in [\Pi_a,\Pi_b]^d. 
\end{align}

\end{definition}

\noindent {\bf Design Goals:} We expect our attack to achieve the below goals. 

\begin{enumerate}[label=\arabic*),leftmargin=*]

\item {\bf Certifiable:} It can provide provable guarantees on the minimum attack success probability of the crafted adversarial examples. 

\item {\bf Verification free:} 
 
It can not only verify examples to be adversarial \emph{after} querying the model, but also verify examples \emph{before} the query by giving its ASP. This significantly boosts the effectiveness of adversarial examples generation. 

\item {\bf Query efficient:} It needs as few number of queries as possible. Fewer queries can definitely save the adversary's cost. 

\item {\bf Bypass defenses:} It can generate imperceptible adversarial perturbations that can bypass the existing detection and pre/post-processing based defenses~\cite{li2022blacklight,qin2021RAND,chandrasekaran2020RANDpost,chen2022AAA}. 

\end{enumerate}

\begin{figure*}[!h]
    \centering
    \includegraphics[width=0.85\linewidth]{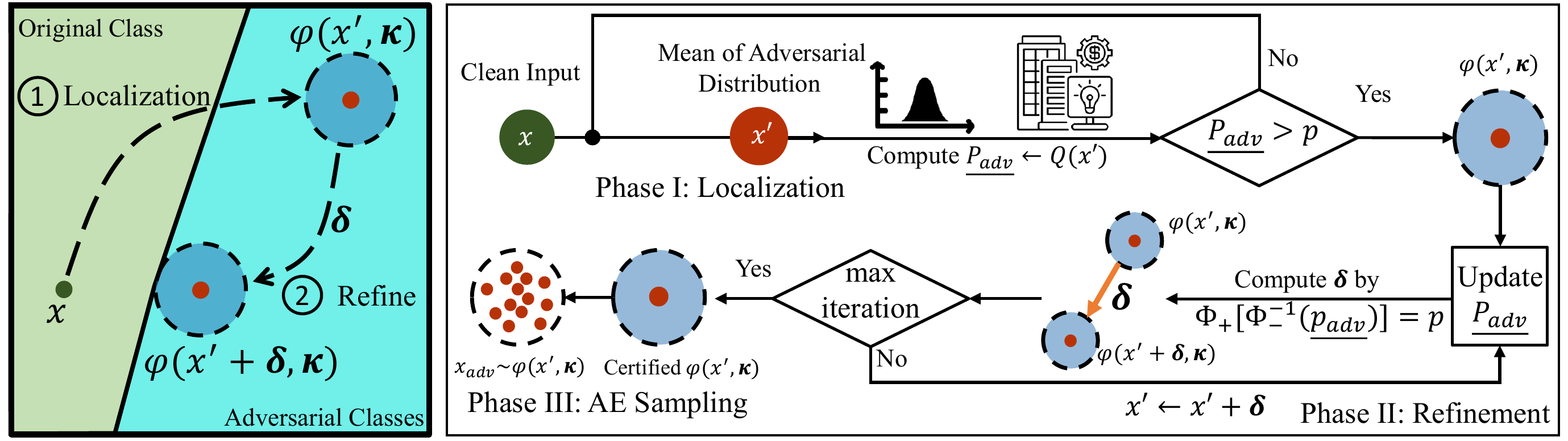}
    \caption{Overview of our certifiable black-box attack to generate certified adversarial distribution.  }
    \label{fig:framework}
\end{figure*}

\section{Attack Overview} 
At a high level, our certifiable black-box attack can be divided into three phases. 
The overview of our attack is depicted in Figure \ref{fig:framework}. 

\vspace{0.05in}

\noindent {\bf Phase I: Adversarial Distribution Localization.} This phrase initially locates a feasible \randAE~ $\varphi$ with guarantees on the lower bound of attack success probabilities (i.e., satisfying Eq. (\ref{eq:target})). There are a few challenges. 
First, computing the exact probability $\mathbb{P}[f(x_{adv})\neq y]$ is intractable due to the high-dimensional continuous input space. 
Second, due to the black-box nature, there exists no gradient information that can be used. 
To address the first challenge and ensure query efficiency, we propose a Randomized Parallel Query (RPQ) strategy that can approximate the probability and ensure multiple queries are implemented in parallel. 
To address the second challenge, we design two localization strategies to enable learning a feasible adversarial distribution. The first strategy adapts the existing self-supervised perturbation (SSP) technique \cite{DBLP:conf/cvpr/NaseerKHKP20}, which facilitates designing a classifier-unknown loss on a pretrained feature extractor such that the adversarial examples/perturbations can be optimized. 
The second one is based on binary search. It first randomly initializes a qualified \randAE~, and then reduces the perturbation size using the binary search algorithm. 
See Section~\ref{sec:initialization} for more details.  

\vspace{0.05in}

\noindent {\bf Phase II: Adversarial Distribution Refinement.} While successfully generating the adversarial distribution, the adversarial examples from it often induce relatively large perturbation sizes. This phrase further refines the adversarial distribution by reducing the perturbation size and maintains the guarantee of attack success probability as well. 
Particularly, we propose to shift the adversarial distribution close to the decision boundary of the classifier. This problem can be solved by two steps: \emph{the first step finds the shifting direction, and the second step derives the shifting distance and maintains the guarantee}. We design a novel shifting method to find the local-optimal \randAE~ by considering the geometric relationship between the decision boundary and \randAE. Deciding the shifting distance can then be converted to an optimization problem. We then propose a binary search algorithm to achieve the goal. See Section~\ref{sec:shifting} for more details.  

\vspace{0.05in}

\noindent {\bf Phase III: Adversarial Example Sampling.} Phases I and II craft an \randAE~ with guaranteed attack success probability, called ``certifiable attack''. To transform the \randAE~ into concrete AEs, we need to sample the AE from the \randAE. The sampled AEs naturally maintain the certified ASP without the need for additional model queries. Optionally, the adversary can verify the success of these sampled AEs to ensure a successful attack, turning the certifiable attack into an empirical attack. Specifically, the adversary can sequentially sample the adversarial examples from \randAE~ and query the target model until finding the successful adversarial example(s).

\section{Certifiable Black-box Attack}
\label{sec:Method}
In this section, we present our certifiable black-box attack in detail. We first introduce the Randomized Parallel Query strategy that estimates the lower bound probability of being the adversarial example (Section~\ref{sec: batch query}). We then  
develop two algorithms to locate the feasible \randAE~ (Section~\ref{sec: locate alg}). 
Next, we propose our refinement method to reduce the perturbation size, while maintaining the guarantees of attack success probability (Section~\ref{sec:shifting}).
We also provide the theoretical analysis of the convergence and confidence bound of the Shifting method. 

\begin{figure}[!h]
    \centering
    \includegraphics[width=0.9\linewidth]{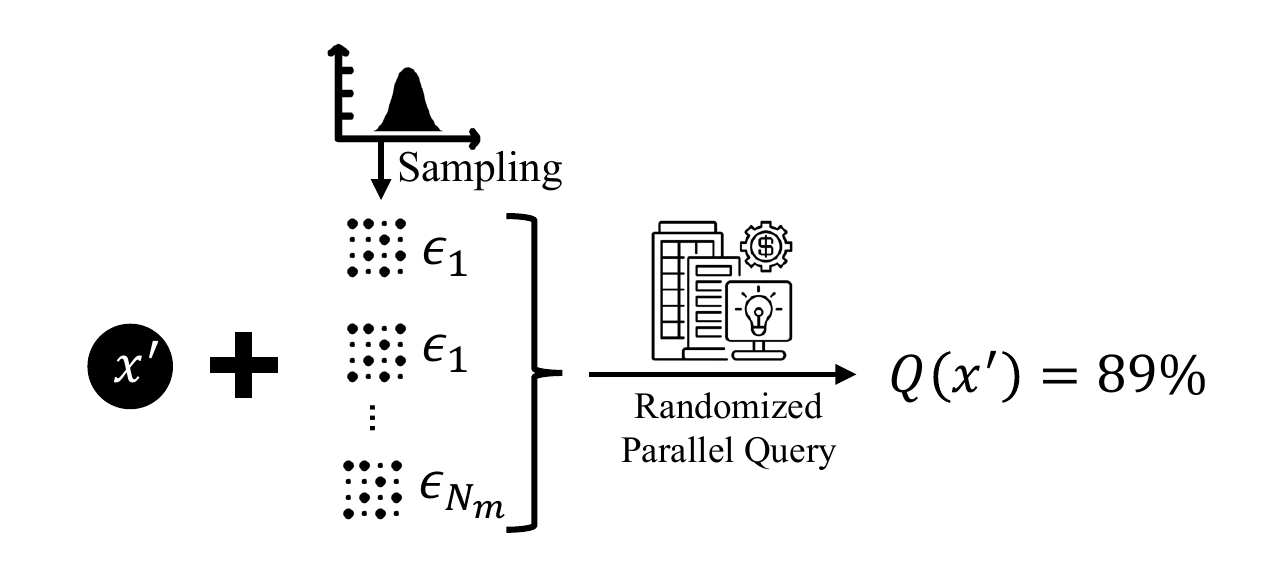}
    \caption{Illustration of randomized parallel query (returning the probability $Q(x')$ that $x'+\epsilon$ is an adversarial example). }
    \label{fig:RQB}
\end{figure}

\subsection{Adversarial Distribution Localization}
\label{sec:initialization}

\subsubsection{Randomized Parallel Query}
\label{sec: batch query}

As stated, computing the exact probability $\mathbb{P}[f(x_{adv})\neq y]$ with $x_{adv}\sim\varphi(x',\boldsymbol{\kappa})$ is intractable. Here, we propose to estimate its low bound probability by the Monte Carlo method. This requires the adversary to query the classifier with random instances sampled from an \randAE. By noting that random instances can be queried efficiently in parallel, we propose the Randomized Parallel Query (RPQ) to compute the lower bound of the attack success probability 
as below:
\begin{align}
 Q(x')=\underline{p_{adv}} & \leq \mathbb{P}_{x_{adv}\sim\varphi(x',\boldsymbol{\kappa})}[f(x_{adv})\neq y] \nonumber \\
 & = \mathbb{P}_{\epsilon \sim \varphi(0,\boldsymbol{\kappa})}[f(x'+\epsilon) \neq y]. 
\end{align}
With a given $x'$, the lower bound probability $\underline{p_{adv}}$ can be estimated via the Binomial testing on a zero-mean distribution $\varphi(0,\boldsymbol{\kappa})$ using Clopper-Pearson confidence interval \cite{lecuyer2019certified} following the Algorithm \ref{alg:Monte Carlo}, where the $\textsc{LowerConfBound}(k,N_m,1-\alpha)$ returns the one-sided $(1-\alpha)$ lower confidence interval.

\begin{algorithm}[!t]
\small
\caption{Lower Bound of Attack Success Probability}
\label{alg:Monte Carlo}
\begin{algorithmic}[1]
\Require Mean $x'$ of the \randAE~ $\varphi$, classifier $f$, confidence level $\alpha$, Monte Carlo samples 
$N_m$, ground truth label $y$.
\Ensure The lower bound of attack success probability $\underline{p_{adv}}$
\State $\epsilon_1, \epsilon_2,...,\epsilon_{N_m} \sim \varphi(0, \boldsymbol{\kappa})$
\State Incorrect prediction count $k \leftarrow \sum^{i=1}_{N_m}\mathbf{1} [f(x'+\epsilon_i)\neq y]$
\State \Return $\underline{p_{adv}} \leftarrow \textsc{LowerConfBound}(k,N_m,1-\alpha)$
\end{algorithmic}
\end{algorithm}

\setlength{\textfloatsep}{2mm}

Now we can estimate $\underline{p_{adv}}$  
given an \randAE~ with known mean/location $x'$. The next question is how to decide $x'$ to 
satisfy Eq. (\ref{eq:target}), i.e., locating the adversarial distribution that includes certifiable adversarial examples (with probability at least $p$). 

The simplest way is random localization, where the input $x$ is uniformly sampled from the input space $[\Pi_a,\Pi_b]^d$, e.g., $[0,\ 1]^d$, followed by the RPQ to check if $\underline{p_{adv}}$ is larger than $p$. However, random localization could not generate a good initial adversarial distribution due to the high-dimensional input space. Below we propose two practical localization methods to mitigate the issue. 

\begin{algorithm}[!t]
\small
\caption{Smoothed Self-Supervised Perturbation (SSSP)}
\label{alg:sssp}
\begin{algorithmic}[1]
\Require
Clean input $x$, feature extractor $\mathcal{F}$, 
noise distribution $\varphi(0,\boldsymbol{\kappa})$, maximum iterations $n_{max}$,  perturbation budget $\pi$, step size $\eta$, and noise sampling number $N_s$.
\Ensure Updated mean $x'$ of \randAE~ 
\State $x'=x$
\For {$n=1$ to $n_{max}$ }
\State $\mathcal{L}(x')\leftarrow \frac{1}{N_s}\sum_i^{N_s}[||\mathcal{F}(x'+\epsilon_i)-\mathcal{F}(x+\epsilon_i)||_2], \ \epsilon_i \sim \varphi$
\State $x' \leftarrow  x' + \eta\ sgn(\nabla_{x'}\mathcal{L})$
\State $x' \leftarrow Clip(x',\ x-\pi,\ x+\pi)$ 
\State $x' \leftarrow Clip(x',\ 0.0,\ 1.0)$ (if $x$ is an image)
\EndFor
\State \Return $x'$
\end{algorithmic}
\end{algorithm}

\setlength{\textfloatsep}{2mm}

\begin{algorithm}[!t]
\small
\caption{Smoothed SSP for Certifiable Attack Localization}
\label{alg:initialization}
\begin{algorithmic}[1]
\Require Clean input $x$, feature extractor $\mathcal{F}(\cdot)$, RPQ function $Q(\cdot)$,  smoothed SSP algorithm $SSSP(\cdot)$ (Algorithm \ref{alg:sssp}), initial perturbation budget $\pi_{init}$,  
step size $\gamma$, ASP Threshold $p$, maximum iterations $N_{max}$.
\Ensure Mean $x'$ of \randAE~ $\varphi$, number of RPQs $q$.
\State $x'=x$, $\pi=\pi_{init}$, $N=0$, $q=0$
\While {$Q(x')<p$ and $N<N_{max}$}
\State $N\leftarrow N+1$, $q\leftarrow q+1$, $\pi \leftarrow \pi+\gamma$
\State $x'\leftarrow SSSP(x',\ \mathcal{F},\ \pi)$
\EndWhile
\If{$Q(x')<p$}
\State \Return \textbf{Abstain}
\Else
\State \Return $x'$ and $q$ 
\EndIf
\end{algorithmic}
\end{algorithm}

\setlength{\textfloatsep}{2mm}

\subsubsection{Proposed Localization Algorithms}
\label{sec: locate alg}
We notice the adversarial distribution localization is similar to empirical black-box attacks on generating adversarial examples. Here, we propose to adapt these empirical attack algorithms and design two localization algorithms.

\vspace{0.05in}

\noindent {\bf Smoothed Self-Supervised Localization:}
To better locate the \randAE, we propose to adapt the self-supervised perturbation (SSP) technique \cite{DBLP:conf/cvpr/NaseerKHKP20}. Specifically, SSP generates generic adversarial examples by distorting the features extracted by a pre-trained feature extractor on a large-scale dataset in a self-supervised manner. The rationale is that the extracted (adversarial) features  
can be transferred to other classifiers as well.

As our attack uses RPQ, we compute the feature distortion over \emph{a set of random samples} from the \randAE. Formally, 

{
\vspace{-4mm}
\begin{align}
   & x'= \arg \max_{x'} ~ \mathbb{E}_{\epsilon \sim \varphi(0,\boldsymbol{\kappa})}[\| \mathcal{F}(x' +\epsilon) -\mathcal{F}(x+\epsilon) \|_2] \label{eqn:xprime} \nonumber\\
    & \text{s.t.} \ \|x'-x\|_\infty\leq \pi 
\end{align}
}
where $\mathcal{F}$ is a pre-trained feature extractor. The perturbation budget $\pi$ is initially set to a small value and later increased in multiple attempts of localization, ensuring that smaller perturbations are identified first. This optimization problem can be solved via the Projected Gradient Ascent method \cite{DBLP:conf/iclr/MadryMSTV18}. Let the adversarial loss be $\mathcal{L}(x')\equiv \mathbb{E}_{\epsilon \sim \varphi}[\|\mathcal{F}(x'+\epsilon)-\mathcal{F}(x+\epsilon)\|_2]$. Then we can locate the \randAE~ via iteratively update $x'$ with $x'=x'+\eta\ sgn(\nabla_{x'} \mathcal{L})$, where $sgn(\cdot)$ is the sign function, and $\eta$ denotes the step size. The details for localizing the \randAE~ are summarized in Algorithms \ref{alg:sssp} and \ref{alg:initialization}. 

\vspace{0.05in}

\noindent {\bf Binary Search Localization:} 
Another method is to randomly initialize the location of  \randAE~ such that $\underline{p_{adv}}\geq p$, and then reduce the gap between $\underline{p_{adv}}$ and $p$, as well as the perturbation via binary search. The algorithm is presented in Algorithm \ref{alg:bin search}. This method is efficient in reducing the perturbation size once the feasible \randAE~ is found by random search. Figure \ref{fig:visual_binary} and \ref{fig:visual_sssp} visualize some $x_{adv}$ during the crafting process for both Binary Search Localization and SSSP Localization.

\begin{algorithm}[!t]
\small
\caption{Binary Search for Certifiable Attack Localization}
\label{alg:bin search}
\begin{algorithmic}[1]
\Require Clean input $x$, RPQ function $Q(\cdot)$, ASP Threshold $p$, random search iterations $N_r$, and binary search iteration $N_b$, error tolerance $\Omega$.
\Ensure Mean of initial \randAE~ $x'$, number of RPQs $q$.
\State $n=0$, $m=0$, $q=0$, $x^*=x$
\While{$Q(x')<p$ and $n\leq N_r$}
\State $x' \sim \text{Uniform}([0,1]^d)$
\State $q\leftarrow q+1$, 
\EndWhile
\If{$n>N_r$} \Return \textbf{Abstain}
\EndIf
\While{$m<N_b$ and $\|x'-x^*\|_2\leq \Omega$}
\If{$Q(\frac{x^*+x'}{2})\geq p$}
\State $x'=\frac{x^*+x'}{2}$
\Else
\State $x^*=\frac{x^*+x'}{2}$
\EndIf
\EndWhile
\State \Return $x'$
\end{algorithmic}
\end{algorithm}

\subsection{Adversarial Distribution Refinement}
\label{sec:shifting}

Though our localization algorithms can find an effective \randAE, our empirical results found the perturbation size can be large (See Table~\ref{sec:ablation}). This occurs possibly because the pretrained feature extractor is too generic and the generated adversarial perturbation is suboptimal for our target classifier. To mitigate the issue, we propose to reduce the perturbation by refining the \randAE~ while still maintaining the condition Eq. (\ref{eq:target}).

Our key observation is that the optimal perturbation is achieved when the adversarial example is close to the decision boundary of the target classifier. Hence, we propose to shift the \randAE~ until intersecting the decision boundary, thereby locating the locally optimal point on that boundary.

\subsubsection{Certification for \randAE~ Shifting} 
We propose a theory on shifting the \randAE~ while maintaining the attack success probability. 
We denote $\varphi(x'+\delta,\boldsymbol{\kappa})$ as a shifted distribution for the \randAE~ $\varphi(x', \boldsymbol{\kappa})$ by a shifting vector $\delta$. Then, the shifted \randAE~ ensures the ASP if $\delta$ satisfies the condition presented in Theorem \ref{thm1}.

\begin{thm}{(\textbf{Certifiable Adversarial Distribution Shifting})} 
Let $f$ be a classifier, $\epsilon$ be the noise drawn from any continuous probability density function  $\varphi(0,\boldsymbol{\kappa})$. Let $p$ be the predefined attack success possibility threshold. Denote $\underline{p_{adv}}$ as the lower bound of the attack success probability. For any $x'$ satisfies 

\begin{equation}
    \mathbb{P}[f(x'+\epsilon)\neq y]\geq \underline{p_{adv}}=Q(x')\geq p, 
\label{thm1 condition}
\end{equation}
$\mathbb{P}[f(x'+\delta+\epsilon)\neq y]\geq p$ is guaranteed for any 
shifting vector $\delta$ when
\begin{equation}
    \Phi_+[\Phi_-^{-1}(\underline{p_{adv}})]\geq p
\label{thm1 condition2}
\end{equation}
where $\Phi_-^{-1}$ is the inverse cumulative density function (CDF) of the random variable $\frac{\varphi(\epsilon-\delta, \boldsymbol{\kappa})}{\varphi(\epsilon,\boldsymbol{\kappa})}$, and $\Phi_+$ the CDF of random variable $\frac{\varphi(\epsilon, \boldsymbol{\kappa})}{\varphi(\epsilon+\delta, \boldsymbol{\kappa})}$.

\label{thm1}
\end{thm}

\begin{proof}
See detailed proof in Appendix \ref{apd:thm1 proof}.
\end{proof}

Theorem \ref{thm1} ensures the minimum attack success probability if Eq. (\ref{thm1 condition}) and Eq. (\ref{thm1 condition2}) hold while without querying $\varphi(x'+\delta,\boldsymbol{\kappa})$. Eq. (\ref{thm1 condition}) requires finding a $x'$ such that the RPQ on samples of $\varphi(x',\boldsymbol{\kappa})$ returns a $\underline{p_{adv}} \geq p$, and Eq. (\ref{thm1 condition2}) ensures any $\delta$ meeting this condition will not reduce the attack success probability of the shifted \randAE~ below $p$. Further, Theorem \ref{thm1} works for any continuous noise distributions, e.g., Gaussian, Laplace, Exponential, and mixture PDFs. 
We also present the case when the noise is Gaussian in Corollary \ref{thm2} in Appendix \ref{apd:thm2}. It shows the shifting perturbation $\delta$ should satisfy $||\delta||_2 \le \sigma [\Phi^{-1}(\underline{p_{adv}})-\Phi^{-1}(p)]$, where $\Phi^{-1}$ is the inverse of Gaussian CDF.

\begin{figure}[!t]
    \centering
    \includegraphics[width=\linewidth]{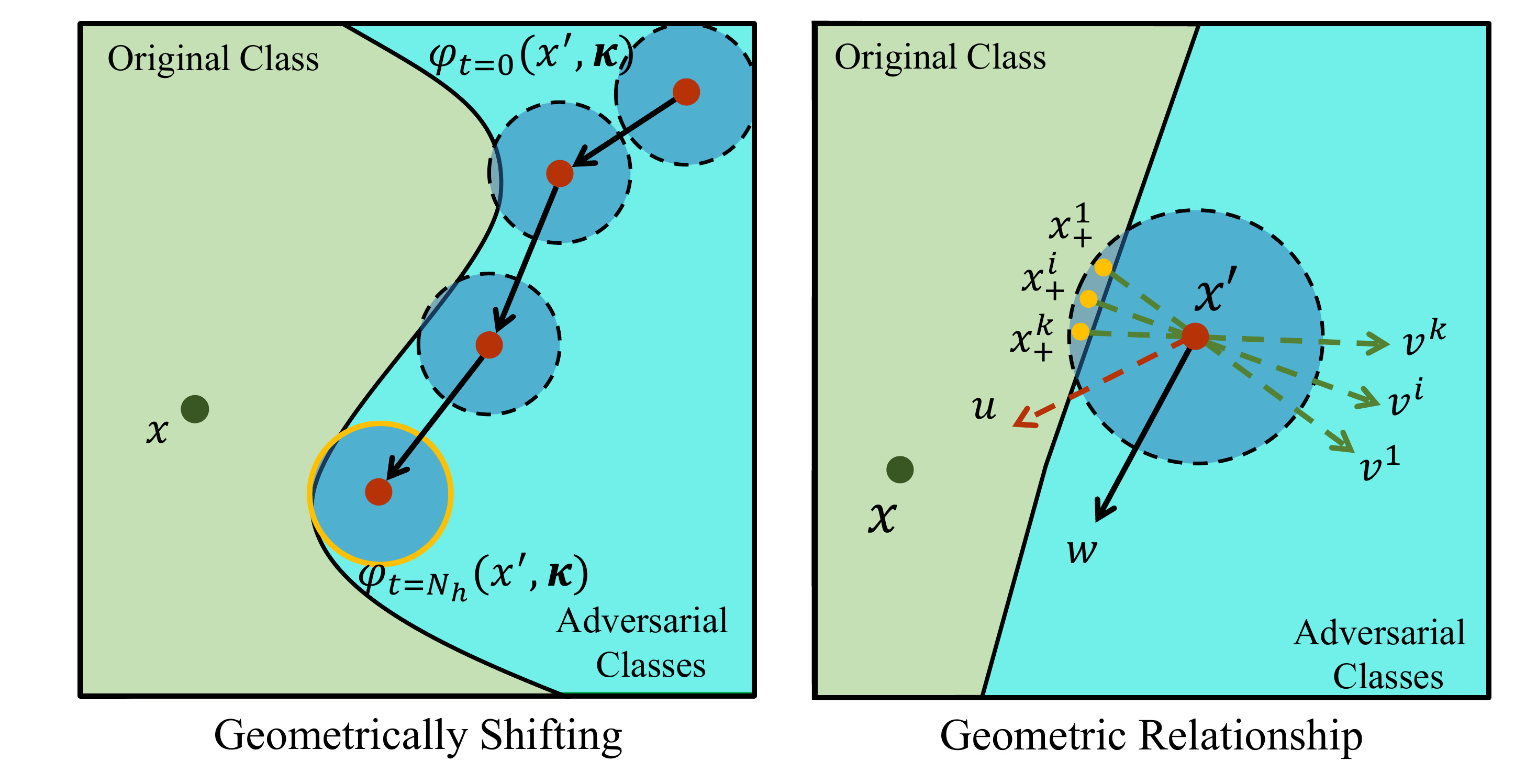}
    \caption{Illustration of geometrically shifting.}
    \label{fig:shifting method}
\end{figure}

\subsubsection{Obtaining Refined \randAE}

Since the \randAE~ can be shifted by any $\delta$ satisfying Eq. (\ref{thm1 condition2}), we propose to shift it toward the clean input with the maximum $\delta$ that does not break the guarantee. By iteratively executing the RPQ and applying the Theorem \ref{thm1}, the \randAE~ can be repeatedly shifted with a guarantee until approaching the decision boundary (where $\underline{p_{adv}}=p$). 

The problem of shifting the \randAE~ to reduce the perturbation can be solved by two steps: \emph{first finding the shifting direction, and then deriving the shifting distance while maintaining the guarantee}. Here, we design a novel shifting method to find the locally optimal \randAE~ by considering the geometric relationship between the decision boundary and the \randAE, which is called ``Geometrical Shifting''. Specifically, through using the noisy samples of \randAE~ to ``probe'' the decision boundary, we shift the RandAE~ along the decision boundary and towards the clean input until finding the local optimal point on the decision boundary (see Figure \ref{fig:shifting method} for the illustration). If none of the noisy samples can approach the decision boundary, we simply shift the \randAE~ directly toward the clean input without considering the decision boundary.  

\begin{algorithm}[!t]
\small
\caption{Shifting Direction}
\label{alg:shifting direction}
\begin{algorithmic}[1]
\Require Mean of the \randAE~ $x'$, clean input $x$,  vectors $\{v^i\}$, a vector $u$, maximum iteration $M$, updating step size $\eta'$.
\Ensure The shifting direction $w$
\State Initialize $w$ with random noise
\If{ $\{v^i\}$ is empty}
\State $w=x-x'$
\Else
\For {$j=1$ to $M$}
\State $w\leftarrow w+\eta'\ sgn[\nabla_w({\sum_{i=1}^k \sin(v^i,w) + \cos(u,w)})]$
\State $w\leftarrow \frac{w}{||w||_2}$
\EndFor
\EndIf
\State \Return $w$
\end{algorithmic}
\end{algorithm}

\setlength{\textfloatsep}{2mm}

\begin{algorithm}[!t]
\small
\caption{Shifting Distance}
\label{alg:shifting distance}
\begin{algorithmic}[1]
\Require Mean of \randAE~ $x'$,  noise distribution $\varphi$, randomized query function $Q(\cdot)$, the shifting direction algorithm $SD(\cdot)$ (Algorithm \ref{alg:shifting direction}), error threshold $e$, ASP Threshold $p$.
\Ensure The shifting perturbation $\delta$
\State $w \leftarrow SD(x')$, $\underline{p_{adv}} \leftarrow Q(x')$
\State find a scalar $a$ such that $\delta=a w$ and $ \Phi_+[\Phi_-^{-1}(\underline{p_{adv}})] > p$
\State find a scalar $b$ such that $\delta=b w$ and $ \Phi_+[\Phi_-^{-1}(\underline{p_{adv}})] < p$
\While {$\Phi_+[\Phi_-^{-1}(\underline{p_{adv}})]<p$ or $ > p+e$ and $n\leq N_k$}
\If {$\Phi_+[\Phi_-^{-1}(\underline{p_{adv}})] >p$}
\State $a\leftarrow \frac{(a+b)}{2}$
\Else
\State $b\leftarrow \frac{(a+b)}{2}$
\EndIf
\State $\delta\leftarrow \frac{(a+b)}{2}w$, $n\leftarrow n+1$
\EndWhile
\State \Return $\delta$
\end{algorithmic}
\end{algorithm}

\begin{algorithm}[!t]
\small
\caption{Certifiable Attack Shifting}
\label{alg:shifting}
\begin{algorithmic}[1]
\Require Mean of \randAE~ $x'$,  noise distribution $\varphi$, randomized query function $Q(\cdot)$, shifting distance algorithm $\textsc{Shift}(\cdot)$ (Algorithm \ref{alg:shifting distance}), distance threshold $e_s$, ASP Threshold $p$, max iteration $N_h$.
\Ensure The shifted mean $x'$
\State $\underline{p_{adv}}\leftarrow Q(x')$, $\delta\leftarrow\textsc{Shift}(x')$, $n=0$
\While{$\underline{p_{adv}}>p$ and $\|\delta\|_2\geq e_s$ and $n\leq N_h$}

\State $x'\leftarrow x'+\delta$, $\underline{p_{adv}}\leftarrow Q(x')$, $\delta\leftarrow\textsc{Shift}(x')$, $n\leftarrow n+1$
\If{$\|x'-x\|_2\leq \|\delta\|_2$}
\Return $x$
\EndIf
\EndWhile
\State \Return $x'$
\end{algorithmic}
\end{algorithm}

\vspace{+0.05in}
\noindent {\bf Finding the Shifting Direction:} 
The geometrical relationship is presented on the right-hand side of Figure \ref{fig:shifting method}.  
Denote $x'$ as the mean of the current \randAE. When sampling the adversarial examples from the \randAE, we mark the failed adversarial examples as $x^1_+, x^2_+,...,x^i_+, ..., x^k_+$, aka., ``samples fell into the original class''. The normalized vector from $x_+^i$ to $x'$ is denoted as $v^i$. The normalized vector from $x'$ to $x$ is denoted as $u$. If the \randAE~ has no samples crossing the decision boundary, then we can shift the \randAE~ straight toward the clean input (along the direction of $u$) until it intersects the decision boundary, otherwise, the shifting should be along the decision boundary but not cross it
(without changing the certifiable attack guarantee). Note that the input space is high-dimensional, thus there could be many directions along the decision boundary. To reduce the perturbation, the direction should be similar to the vector $u$ as much as possible. Based on these geometric analyses, the goals of the geometrical shifting can be summarized as: \emph{The shifting direction should lie relatively parallel to the direction of $u$; and be relatively vertical to the vectors $v^i$.}

Formally, denoting the shifting direction as $w$, then the goal of finding the shifting direction can be formulated as:
\begin{equation}
    w =\arg \max \sum\nolimits_{i=1}^k \sin(v^i,w) + \cos(u,w)
\label{eq:geo eq}
\end{equation}
where $\sin(\cdot)$ and $\cos(\cdot)$ denote the sine and cosine function, and Eq. (\ref{eq:geo eq}) can be solved via the gradient ascent algorithm. 

\vspace{0.05in}

\noindent {\bf Calculating the Shifting Distance:}  The shifting distance can be determined by maximizing $||\delta||_2$ that satisfies the constraint of Eq. (\ref{thm1 condition2}), i.e., when the equality holds. We use binary search to approach the equality and the Monte Carlo method to estimate the CDF of random variable $\frac{\varphi(\epsilon-\delta, \boldsymbol{\kappa})}{\varphi(\epsilon, \boldsymbol{\kappa})}$ and $\frac{\varphi(\epsilon, \boldsymbol{\kappa})}{\varphi(\epsilon+\delta, \boldsymbol{\kappa})}$, similar to \cite{hong2022unicr}. Algorithm \ref{alg:shifting direction}, \ref{alg:shifting distance}, and \ref{alg:shifting} show the details of finding the shifting direction, shifting distance, and the shifting process, respectively.

\vspace{0.05in}

\noindent\textbf{Convergence Guarantee and Confidence Bound:}
Any $\delta$ computed by Algorithm \ref{alg:shifting distance} will satisfy the certifiable attack guarantee since it strictly ensures $\Phi_+[\Phi^{-1}_-(\underline{p_{adv}})] \ge p$. Further, with a centralized noise distribution, the shifting algorithm is guaranteed to converge once the located \randAE~ is feasible.

\begin{thm}
\label{prop:bound}
If the PDF of noise distribution $\varphi(x)$ decreases as the $|x|$ increases, with the satisfaction of Eq. (\ref{thm1 condition}), given any  direction vector $w$, the Shifting Distance algorithm guarantees to find $\delta$ such that $\Phi_+[\Phi_-^{-1}(\underline{p_{adv}})]= p$ with confidence $(1-\alpha)(1-2e^{-2N_m\Delta^2})^2$, where $(1-\alpha)$ is the confidence for of estimating $\underline{p_{adv}}$, $N_m$ is the Monte Carlo samples, and $\Delta$ is the error bound for the CDF estimation.

\end{thm}
\begin{proof}
See detailed proof in Appendix \ref{apd:theoretical bound}.
\end{proof}

\subsection{Discussions on Our Attack}
\label{sec:Discussion_ATK}

\noindent\textbf{Realizing Our Certifiable Attack:} Our certifiable attack does not have extra requirements on realization compared to empirical black-box attacks. 
To implement our attack, we only need to predefine a continuous noise distribution and a threshold of certified attack success probability. The adversary then adds the noise sampled from the distribution to the inputs and queries the target model. Then, \randAE~ can be crafted by RPQ and our theory. 

\vspace{0.05in}

\noindent\textbf{Randomized Query vs. Deterministic Query:} The proposed randomized query returns a \emph{probability} over a batch of inputs with injected random noises, while the traditional query returns a \emph{deterministic} output (score or hard label)  from the target model. This probability return may provide more information that better guides the attack. In addition, the randomized queries can be executed in parallel for query acceleration. 
See results in Section~\ref{section:exp against defense}. 

\vspace{0.05in}

\noindent {\bf Imperceptibility with Diffusion Denoiser:}
The certifiable adversarial examples sampled from \randAE~ are noise-injected inputs that still might be perceptible when the noise is large. We can further leverage the recent innovation for image synthesis, i.e., diffusion model \cite{ho2020denoising}, to denoise the adversarial examples for better imperceptibility. The key idea is to consider the noise-perturbed adversarial examples as the middle sample in the forward process of the diffusion model \cite{carlini2022certified,zhang2023diffsmooth}. This is shown to improve the imperceptibility and the diversity of the adversarial examples. 
More technical details are shown in Appendix~\ref{sec:DM} and results in Table~\ref{tab:denoise} in Appendix~\ref{sec:exp diffusion}. 

\vspace{0.05in}

\noindent\textbf{Extension to Certifiable White-Box Attack:} 
Our certifiable attack can be readily extended to the white-box setting by adapting/designing a white-box localization method. Specifically, the Smoothed SSP localization method can directly compute the gradients of the noise-perturbed examples rather than leveraging the feature extractor, which may significantly improve the certified accuracy of the certifiable attack. In our experiments, when leveraging the PGD-like white-box attacks as the localization method, the certified accuracy can be increased to $100\%$ for ResNet and CIFAR10, compared to the $92.54\%$ certified accuracy in the black-box setting.

\vspace{0.05in}

\noindent\textbf{Extension to Targeted Certifiable Attack:} Our attack design focuses on the untargeted certifiable attack. It can also be generalized to the targeted attack setting, where we require the majority of the noise-perturbed inputs to be \emph{certifiably} misclassified to a specific \emph{target} label. However, we admit it would be more challenging to find a 
successful \randAE~ in this scenario.  

\vspace{0.05in}

\noindent\textbf{Attacks under Adaptive Blacklight:} The defender might design an adaptive countermeasure, such as an adaptive blacklight defense, to mitigate certified attacks. For instance, the defender could attempt to eliminate randomness by assuming the noise distribution is known. However, this approach presents several challenges: 1) The defender would need detailed knowledge about the attack’s design, including the noise distribution, which is often an impractical assumption. 2) Even if the noise distribution were known, the sampled adversarial examples would remain random, making it difficult to accurately estimate the center of the noise distribution.

\section{Evaluations}
\label{sec:Experiment}

\begin{table}[]
\centering
\caption{Summary of Experiments}
\vspace{-2mm}
\label{tab:exp summarize}
\resizebox{0.44\textwidth}{!}{%
\begin{tabular}{|clll|}
\hline
Experiments                                    & Dataset     & Model                     & Reference                                   \\ \hline
\multicolumn{1}{|c|}{\multirow{9}{*}{\begin{tabular}[c]{@{}c@{}}Comparison with \\ empirical attacks \\ against Blacklight detection\end{tabular}}} &
  CIFAR10 &
  VGG16 &
  Table \ref{tab:blacklight_cifar10_vgg} \\
\multicolumn{1}{|c|}{}                         & CIFAR10     & ResNet110                 & Table \ref{tab:blacklight_cifar10_resnet}   \\
\multicolumn{1}{|c|}{}                         & CIFAR10     & ResNext29                 & Table \ref{tab:blacklight_cifar10_resnext}  \\
\multicolumn{1}{|c|}{}                         & CIFAR10     & WRN28                     & Table \ref{tab:blacklight_cifar10_wrn}      \\
\multicolumn{1}{|c|}{}                         & CIFAR100    & VGG16                     & Table \ref{tab:blacklight_cifar100_vgg}     \\
\multicolumn{1}{|c|}{}                         & CIFAR100    & ResNet110                 & Table \ref{tab:blacklight_cifar100_resnet}  \\
\multicolumn{1}{|c|}{}                         & CIFAR100    & ResNext29                 & Table \ref{tab:blacklight_cifar100_resnext} \\
\multicolumn{1}{|c|}{}                         & CIFAR100    & WRN28                     & Table \ref{tab:blacklight_cifar100_wrn}     \\
\multicolumn{1}{|c|}{}                         & ImageNet    & ResNet18                  & Table \ref{tab:blacklight_imagenet_resnet}  \\ \hline
\multicolumn{1}{|c|}{\multirow{9}{*}{\begin{tabular}[c]{@{}c@{}}Comparison with\\ empirical attacks against \\ RAND pre-processing defense\end{tabular}}} &
  CIFAR10 &
  VGG16 &
  Table \ref{tab:cifar10_RAND_vgg} \\
\multicolumn{1}{|c|}{}                         & CIFAR10     & ResNet110                 & Table \ref{tab:cifar10_RAND_resnet}         \\
\multicolumn{1}{|c|}{}                         & CIFAR10     & ResNext29                 & Table \ref{tab:cifar10_RAND_resnext}        \\
\multicolumn{1}{|c|}{}                         & CIFAR10     & WRN28                     & Table \ref{tab:cifar10_RAND_wrn}            \\
\multicolumn{1}{|c|}{}                         & CIFAR100    & VGG16                     & Table \ref{tab:cifar100_RAND_vgg}           \\
\multicolumn{1}{|c|}{}                         & CIFAR100    & ResNet110                 & Table \ref{tab:cifar100_RAND_resnet}        \\
\multicolumn{1}{|c|}{}                         & CIFAR100    & ResNext29                 & Table \ref{tab:cifar100_RAND_resnext}       \\
\multicolumn{1}{|c|}{}                         & CIFAR100    & WRN28                     & Table \ref{tab:cifar100_RAND_wrn}           \\
\multicolumn{1}{|c|}{}                         & ImageNet    & ResNet18                  & Table \ref{tab:imagenet_RAND_resnet}        \\ \hline
\multicolumn{1}{|c|}{\multirow{9}{*}{\begin{tabular}[c]{@{}c@{}}Comparison with \\ empirical attacks against \\ RAND post-processing defense\end{tabular}}} &
  CIFAR10 &
  VGG16 &
  Table \ref{tab:cifar10_post_RAND_vgg} \\
\multicolumn{1}{|c|}{}                         & CIFAR10     & ResNet110                 & Table \ref{tab:cifar10_post_RAND_resnet}    \\
\multicolumn{1}{|c|}{}                         & CIFAR10     & ResNext29                 & Table \ref{tab:cifar10_post_RAND_resnext}   \\
\multicolumn{1}{|c|}{}                         & CIFAR10     & WRN28                     & Table \ref{tab:cifar10_post_RAND_wrn}       \\
\multicolumn{1}{|c|}{}                         & CIFAR100    & VGG16                     & Table \ref{tab:cifar100_post_RAND_vgg}      \\
\multicolumn{1}{|c|}{}                         & CIFAR100    & ResNet110                 & Table \ref{tab:cifar100_post_RAND_resnet}   \\
\multicolumn{1}{|c|}{}                         & CIFAR100    & ResNext29                 & Table \ref{tab:cifar100_post_RAND_resnext}  \\
\multicolumn{1}{|c|}{}                         & CIFAR100    & WRN28                     & Table \ref{tab:cifar100_post_RAND_wrn}      \\
\multicolumn{1}{|c|}{}                         & ImageNet    & ResNet18                  & Table \ref{tab:imagenet_post_RAND_resnet}   \\ \hline
\multicolumn{1}{|c|}{\multirow{2}{*}{\begin{tabular}[c]{@{}c@{}}Comparison with empirical attack \\ against adversarial training\end{tabular}}} &
  CIFAR10 &
  ResNet110 ($\ell_2$) &
  \multirow{2}{*}{Table \ref{tab:AT_cifar10_resnet}} \\
\multicolumn{1}{|c|}{}                         & CIFAR10     & ResNet110 ($\ell_\infty$) &                                             \\ \hline
\multicolumn{1}{|c|}{\multirow{3}{*}{\begin{tabular}[c]{@{}c@{}}Ablation: CA vs. \\ different noise variance\end{tabular}}} &
  CIFAR10 &
  ResNet110 &
  \multirow{2}{*}{Table \ref{tab:diff variance}} \\
\multicolumn{1}{|c|}{}                         & ImageNet    & ResNet50                  &                                             \\ \cline{2-4} 
\multicolumn{1}{|c|}{}                         & LibriSpeech & ECAPA-TDNN                & Table \ref{tab:diff variance audio}         \\ \hline
\multicolumn{1}{|c|}{\multirow{3}{*}{\begin{tabular}[c]{@{}c@{}}Ablation: CA vs. \\ different  p\end{tabular}}} &
  CIFAR10 &
  ResNet110 &
  \multirow{2}{*}{Table \ref{tab:diff p}} \\
\multicolumn{1}{|c|}{}                         & ImageNet    & ResNet50                  &                                             \\ \cline{2-4} 
\multicolumn{1}{|c|}{}                         & LibriSpeech & ECAPA-TDNN                & Table \ref{tab:diff p audio}                \\ \hline
\multicolumn{1}{|c|}{\multirow{3}{*}{\begin{tabular}[c]{@{}c@{}}Ablation: CA vs. different \\ Localization/Shifting\end{tabular}}} &
  CIFAR10 &
  ResNet110 &
  \multirow{2}{*}{Table \ref{tab:ablation study}} \\
\multicolumn{1}{|c|}{}                         & ImageNet    & ResNet50                  &                                             \\ \cline{2-4} 
\multicolumn{1}{|c|}{}                         & LibriSpeech & ECAPA-TDNN                & Table \ref{tab:ablation study audio}        \\ \hline
\multicolumn{1}{|c|}{\multirow{2}{*}{\begin{tabular}[c]{@{}c@{}}Ablation: CA\\ vs. different noise PDF\end{tabular}}} &
  CIFAR10 &
  ResNet110 &
  \multirow{2}{*}{Table \ref{tab:diff pdf}} \\
\multicolumn{1}{|c|}{}                         & ImageNet    & ResNet50                  &                                             \\ \hline
\multicolumn{1}{|c|}{\multirow{2}{*}{\begin{tabular}[c]{@{}c@{}}Ablation: CA w/ and w/o \\ Diffusion Denoise\end{tabular}}} &
  CIFAR10 &
  ResNet110 &
  \multirow{2}{*}{Table \ref{tab:denoise}} \\
\multicolumn{1}{|c|}{}                         & ImageNet    & ResNet50                  &                                             \\ \hline
\multicolumn{1}{|c|}{CA vs. Feature Squeezing} &
  CIFAR10 &
  ResNet110 &
  Figure \ref{fig:ablation and ROC curve} \\ \hline
\multicolumn{1}{|c|}{CA vs. Adaptive Denoiser} & CIFAR10     & ResNet110                 & Table \ref{tab:WB defense}                  \\ \hline
\multicolumn{1}{|c|}{CA vs. Rand. Smoothing}   & CIFAR10     & ResNet110                 & Table \ref{tab:RS defense}                  \\ \hline
\end{tabular}%
}
\vspace{+0.05in}
\end{table}

We comprehensively evaluate our certifiable black-box attack in various experimental settings. Particularly, we would like to study the following research questions: 

\vspace{-0.1in}

\begin{itemize}[leftmargin=*]
\item {\bf RQ1:} 
How effective is the learnt \randAE? Particularly, how large is the probability of samples from it being successful adversarial examples?   
\vspace{0.05in}

\item {\bf RQ2:} Can our certifiable attack outperform empirical attacks in terms of attack effectiveness and query efficiency? 
\vspace{0.05in}

\item {\bf RQ3:} How effective is our attack to break SOTA defenses? 
\vspace{0.05in}

\item {\bf RQ4:} What is the impact of the design components and their hyperparameters on our attack? 
\end{itemize}

\vspace{-0.1in}

Accordingly, we first assess the empirical attack success possibility of the \randAE~ in Section \ref{sec: exp verification free}. Then, we evaluate our certifiable attack on various models with defenses while benchmarking with empirical black-box attacks in Section \ref{section:exp against defense}. In Section \ref{section:exp ablation}, we conduct ablation studies to explore in-depth our certifiable attack. 
\emph{All sets of experiments are summarized in Table \ref{tab:exp summarize} for reference.}

\vspace{-0.05in}
\subsection{Experimental Setup}

\noindent\textbf{Datasets and Models}. We use three benchmark datasets for image classification: CIFAR10/CIFAR100 \cite{krizhevsky2009learning} and ImageNet \cite{ILSVRC15}.
CIFAR10 and CIFAR100, both consisting of $60,000$ $32x32$ color images split into $10$ and $100$ classes, respectively. ImageNet is a large-scale dataset with $1,000$ classes. The training set contains $1,281,167$ images and the validation set contains $50,000$ images (resized to $3\times224\times 224$). 
we use VGG \cite{VGG}, ResNet \cite{he2016deep}, ResNext \cite{ResNext}, and WRN \cite{WRN} as the target model. We use a pre-trained ResNet34 on ImageNet as the feature extractor (in the Smoothed SSP localization). 
We also test our attacks on the audio dataset LibriSpeech~\cite{korvas_2014} for the speaker verification task, and results are shown in Appendix~\ref{apd:audio}.  

\vspace{0.05in}

\noindent {\bf Baseline Attacks.}
We compare our certifiable (hard label-based) black-box attack with SOTA black-box attacks including 7 \emph{hard label-based} black-box attacks: GeoDA \cite{rahmati2020geoda}, HSJ \cite{DBLP:conf/sp/ChenJW20}, Opt \cite{DBLP:conf/iclr/ChengLCZYH19}, RayS \cite{RayS}, SignFlip \cite{SignFlip}, SignOPT \cite{SignOpt}, and Boundary \cite{DBLP:conf/iclr/BrendelRB18}; and 7 \emph{score-based} black-box attacks: Bandit \cite{Bandit}, NES \cite{DBLP:conf/icml/IlyasEAL18}, Parsimonious \cite{Parsimonious}, Sign \cite{Sign}, Square \cite{DBLP:conf/eccv/AndriushchenkoC20}, ZOSignSGD \cite{ZOSignSGD}, Simple attack \cite{DBLP:conf/icml/GuoGYWW19}. As our method does not constrain the perturbation budget but minimizing the perturbation, 
we also compare with two similar attacks: SparseEvo \cite{SparseEvo} and PointWise \cite{PointWise}. We evaluate our attack with both SSSP localization and binary-search localization. For optimized-based attacks, we limit the AEs in the valid image space. For a fair comparison with optimization-based attacks, the perturbation budget for $\ell_p$-bounded attacks are set to $0.1$ for $\ell_\infty$ and $5$ for $\ell_2$ on CIFAR10 and CIFAR100, while on ImageNet, they are $\ell_\infty = 0.1$ and $\ell_2 = 40$. The maximum query limits are $10,000$ for CIFAR10 and CIFAR100 and $1,000$ for ImageNet. We evaluate $1,000$ randomly selected images for each dataset.

\vspace{0.05in}

\noindent {\bf Defenses.}
We select 4 SOTA defenses against black-box attacks for evaluation: Blacklight detection \cite{li2022blacklight}, Randomized pre-processing defense (RAND-Pre) \cite{qin2021RAND}, Randomized post-processing defense (RAND-Post) \cite{chandrasekaran2020RANDpost}, and Adversarial Training based TRADES \cite{TRADES}. Blacklight has recently proposed to mitigate query-based black-box attacks by utilizing the similarity among queries. It has been shown to detect $100\%$ adversarial examples generated in multiple attacks. RAND-Pre and RAND-Post respectively add noise to the inputs and prediction logits to obfuscate the gradient estimation or local search. TRADES has demonstrated SOTA robustness performance against adversarial attacks by training on adversarial examples.

\vspace{0.05in}

\noindent\textbf{Metrics}. We use the below metrics to evaluate all compared attacks. 
\begin{itemize}[leftmargin=*]
\vspace{-0.03in}
\item {\bf Model Accuracy:} the model accuracy under attack and defense.

\item {\bf Number of RPQ (\# RPQ):} the number of the randomized parallel query for certifiable attack.

\item {\bf Number of Query (\# Q):}  the total number of queries for empirical attack. For our method, it is equal to Monte Carlo Sampling Number $\times$ \# RPQ + additional queries for sampling from the \randAE.

\item {\bf Certified Accuracy@p:} the certified accuracy at the ASP Threshold $p$. 
It is the percentage of the testing samples that have the certified ASP at least $p$, e.g., a $95\%$ certified accuracy with ASP Threshold $p=90\%$ means the adversary can guarantee to have $90\%$ probability to attack successfully for $95\%$ testing samples.

\item {\bf $\ell_2$ Perturbation Size (Dist. $\ell_2$):} $\ell_2$ distance between the adversarial example $x_{adv}$ and the clean input $x$, i.e., $\|x_{adv}-x\|_2$.

\item {\bf $\ell_2$ Mean Distance (Mean Dist. $\ell_2$):} $\ell_2$ distance between the mean $x'$ of \randAE~  and clean input $x$, i.e., $\|x'-x\|_2$.

\item {\bf Detection Success Rate (Det. Rate):} the detection success rate of Blacklight detection.

\item {\bf Average \# Queries for Detection (\# Q to Det.):} the average number of queries before Blacklight detects an AE.

\item {\bf Detection Coverage (Det. Cov.):} the percent of queries in an attack's query sequence that Blacklight identified as attack queries. 

\end{itemize}

\noindent\textbf{Parameters Settings}. There exist many parameters that may affect the performance of our certifiable attack. For instance, the Monte Carlo sampling number, the attack success probability $p$, and the family of the adversarial distribution and its parameters. If not specified, we set Monte Carlo sampling number to be $50$, $p=10\%$, and use Gaussian distribution with variance $\sigma=0.025$. We will also study the impact of these parameters in Section \ref{section:exp against defense}. All the parameter details are summarized in Table \ref{tab:parameters} in Appendix \ref{apd:exp settings}. 

\noindent\textbf{Experimental Environment}. We implemented a PyTorch library\footnote{The codes are available at \url{https://github.com/datasec-lab/CertifiedAttack}
} including 16 black-box attacks, 4 defenses, 6 datasets, and 9 models by integrating several open-source libraries\footnote{\href{https://github.com/SCLBD/BlackboxBench.git}{BlackboxBench}, \href{https://github.com/hysts/pytorch_image_classification}{pytorch image classification}, \href{https://github.com/huiying-li/blacklight}{Blacklight}, \href{https://github.com/SparseEvoAttack/SparseEvoAttack.github.io.git}{SparseEvo}, and \href{https://github.com/yaodongyu/TRADES}{TRADES}
}. The experiments were run on a server with AMD EPYC Genoa 9354 CPUs (32 Core, 3.3GHz), and NVIDIA H100 Hopper GPUs (80GB each).

\vspace{-0.05in}
\subsection{Verifying the Adversarial Distribution}
\label{sec: exp verification free}

\begin{figure}
    \centering
    \centering
    \includegraphics[width=0.23\textwidth]{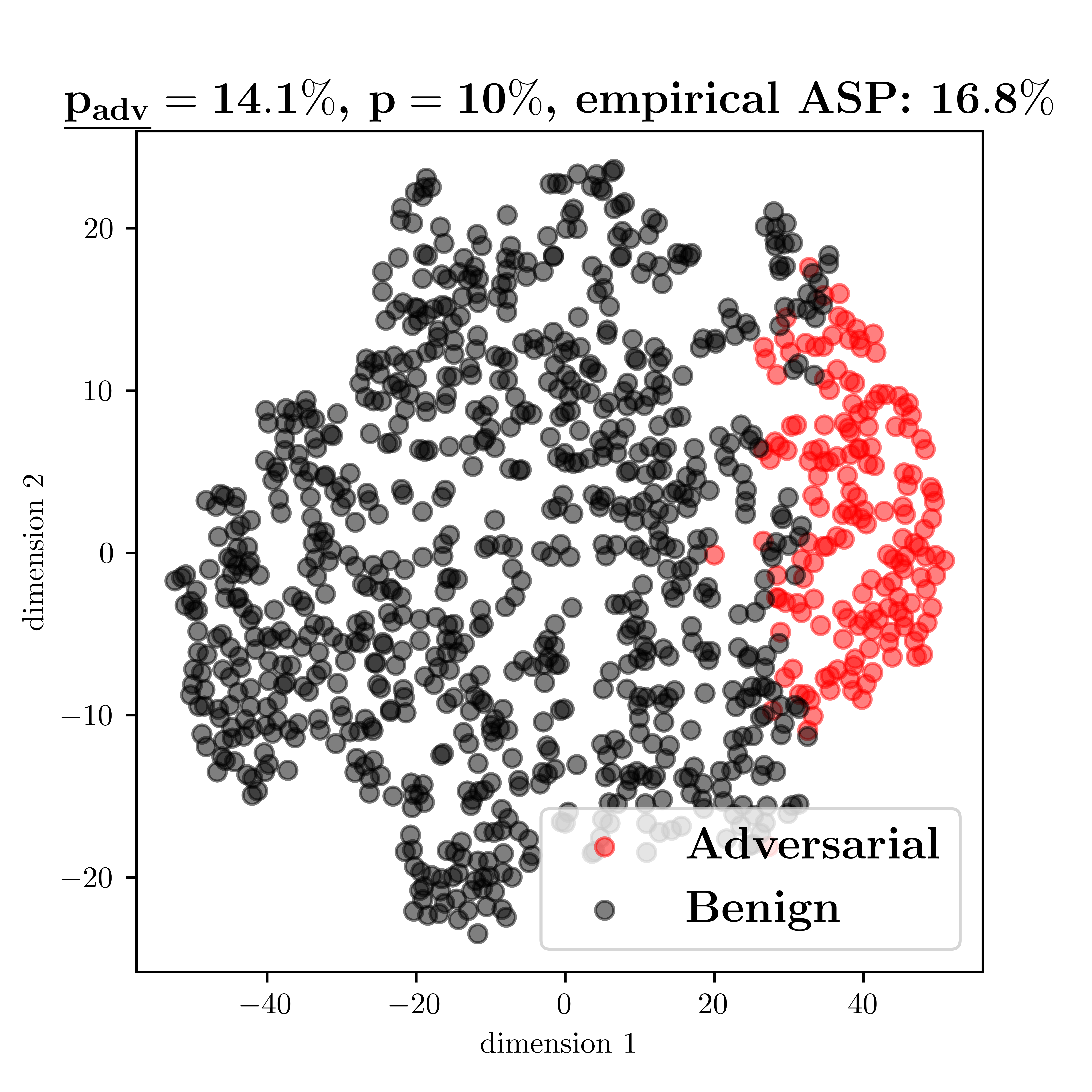}  
    \label{fig:mean and std of net14}
    \vspace{-4mm}
    \hfill
    \centering 
    \includegraphics[width=0.23\textwidth]{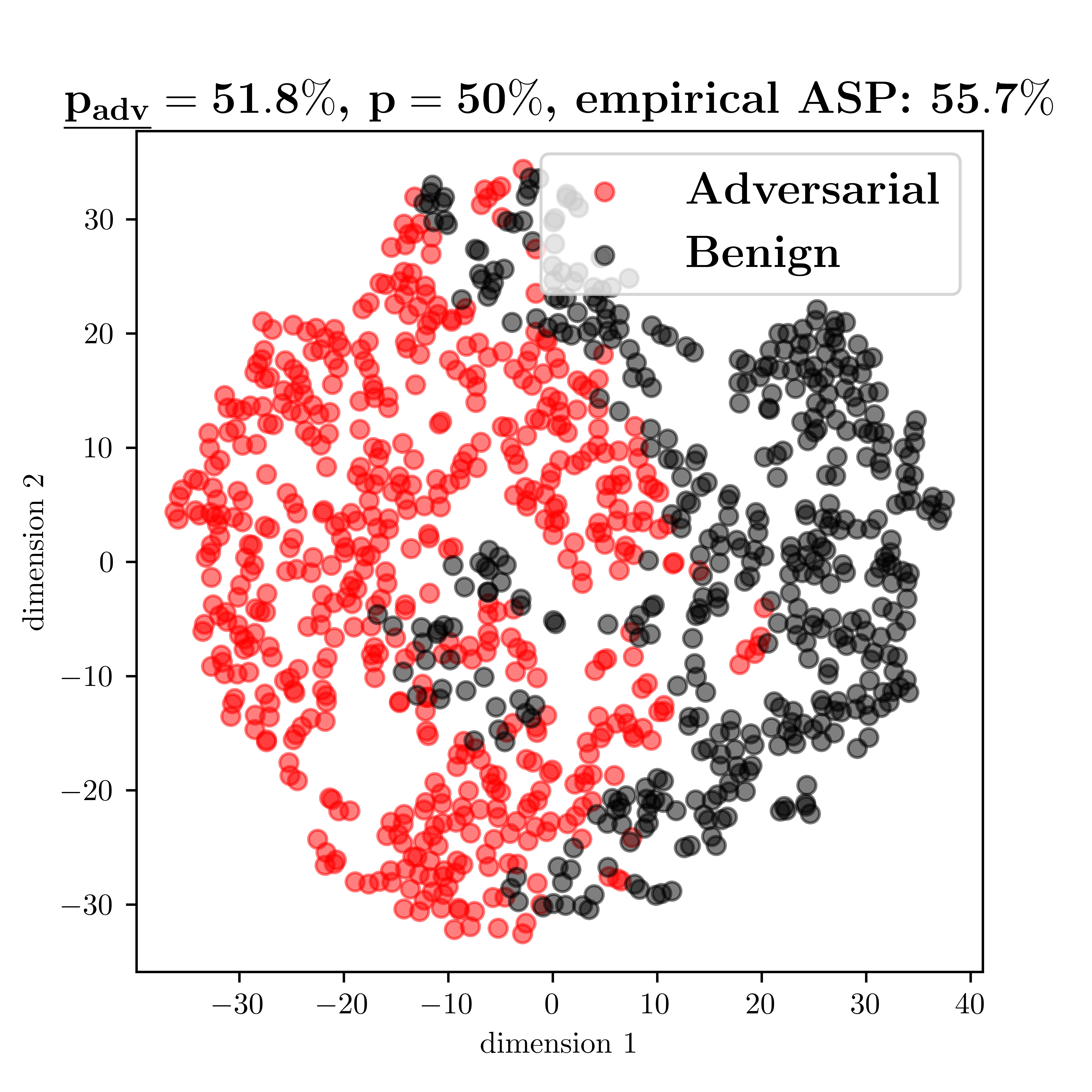}
    \label{fig:mean and std of net24}
    \vspace{-4mm}
    
    \vskip\baselineskip
    \centering 
    \includegraphics[width=0.23\textwidth]{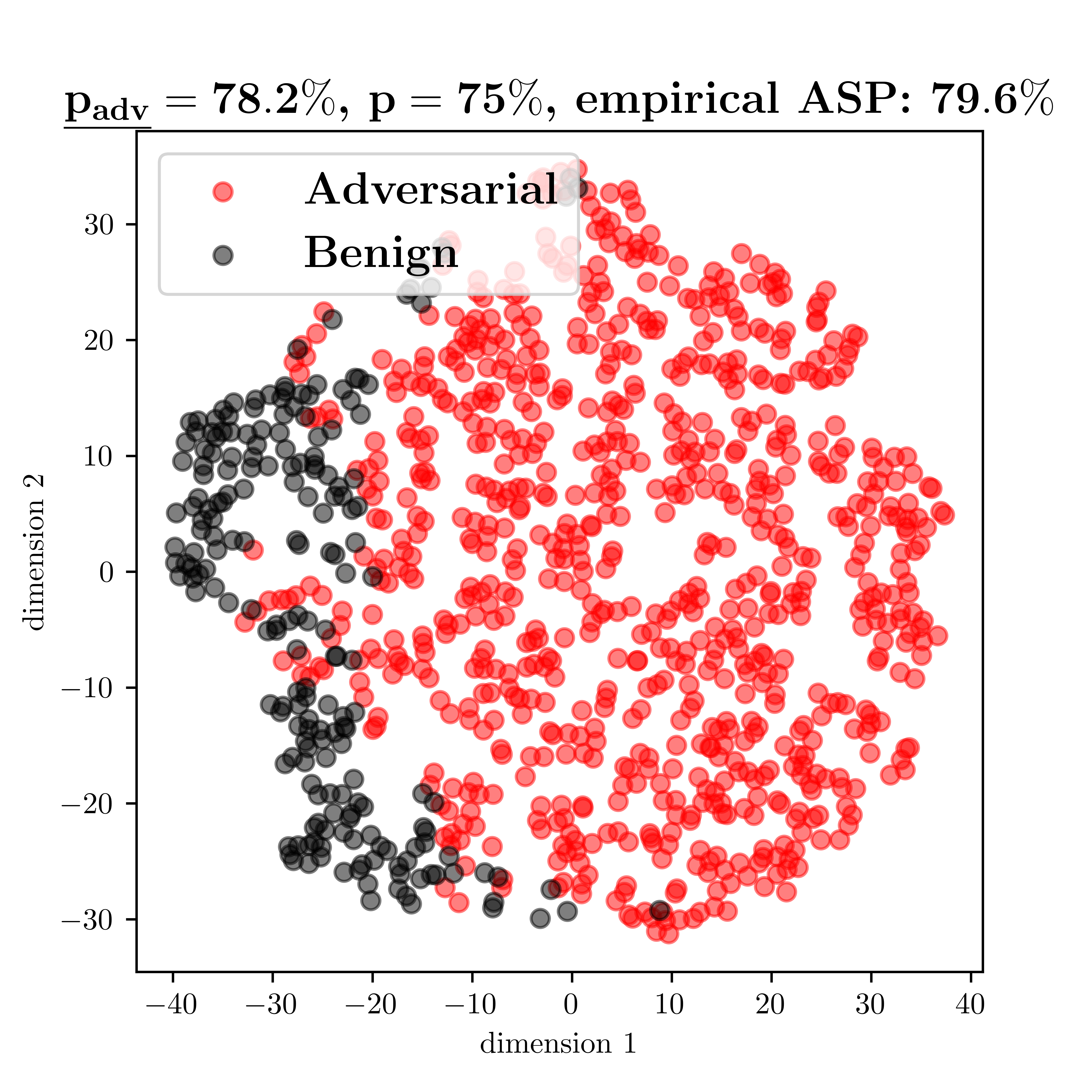}
    \label{fig:mean and std of net34}
    \hfill
    \centering 
    \includegraphics[width=0.23\textwidth]{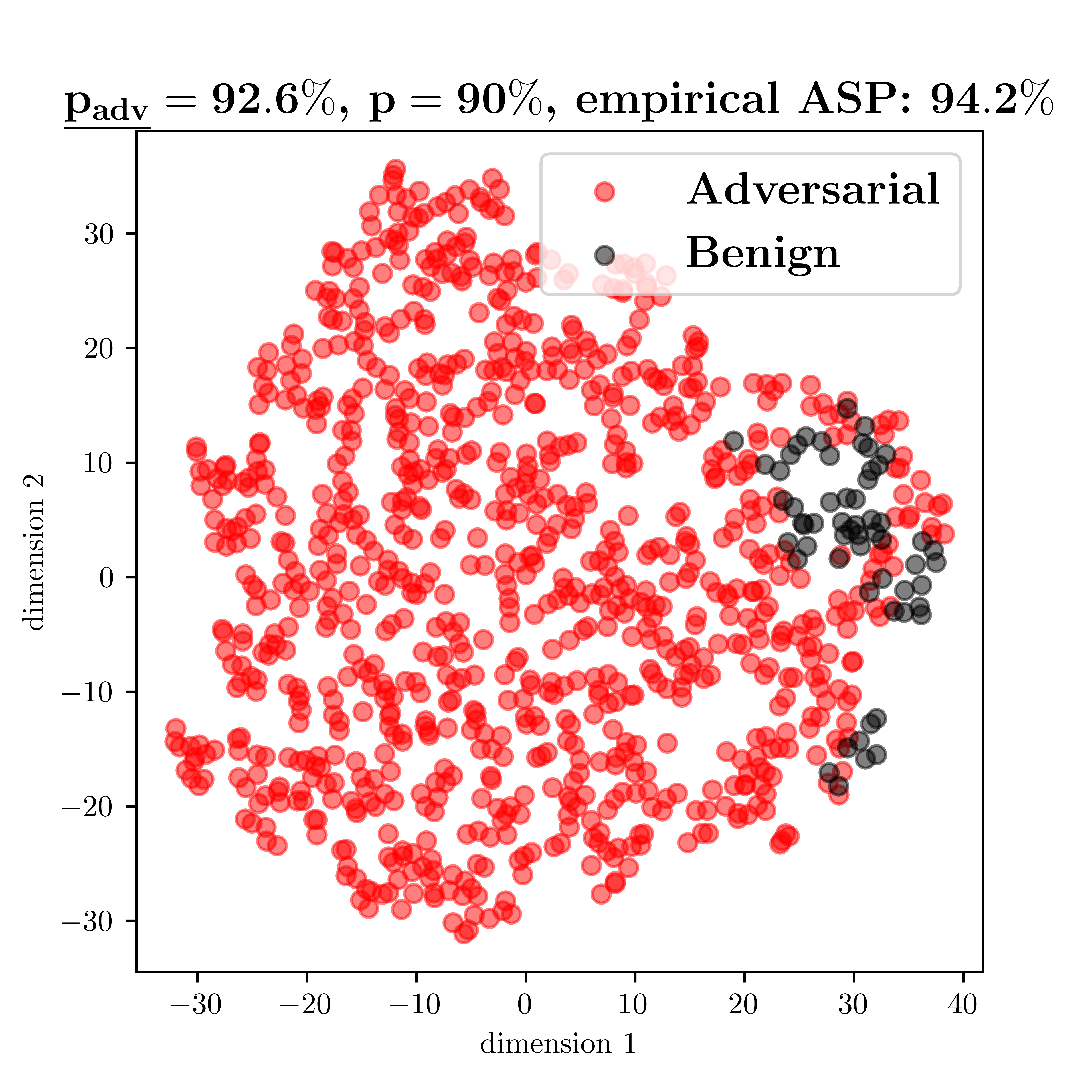}   
    \label{fig:mean and std of net44}
    \vspace{-4mm}
    \caption[]
    {t-SNE visualization of adversarial example sampling from the adversarial distribution.} 
    \label{fig:t-sne}
\end{figure}

We first assess the ASP of the crafted \randAE. Specifically, given an ASP Threshold $p$ and the certified \randAE~ $\varphi(x',\boldsymbol{\kappa})$, we randomly sample $1,000$ examples $x_{adv}\sim \varphi(x',\boldsymbol{\kappa})$, and query the model. We visualize the query results for 4 certified \randAE s with different $p$ using 2D t-SNE\footnote{t-SNE reduces the prediction logits of the random samples to 2-dimension.} \cite{t-SNE}. We also report the provable lower bound of ASP $\underline{p_{adv}}$ and the empirical ASP in Figure \ref{fig:t-sne}. It validates that the sampled AEs ensure the minimum ASP via the \randAE, and the \randAE~lies on the decision boundary.

\subsection{Attack Performance against SOTA Defenses}
\label{section:exp against defense}

In this section, we evaluate our certifiable attack and empirical attacks against the 4 studied SOTA defenses.

\subsubsection{Attack Performance under Blacklight Detection \cite{li2022blacklight}}
\label{sec: exp blacklight}
\begin{table}[]
\centering
\caption{Attack performance under Blacklight detection on ResNet and ImageNet (Clean Accuracy: $67.9\%$)}\vspace{-0.1in}

\label{tab:blacklight_imagenet_resnet}
\resizebox{0.47\textwidth}{!}{%
\begin{tabular}{lcccccccc}
\hline
Attack &
  \begin{tabular}[c]{@{}c@{}}Query\\ Type\end{tabular} &
  \begin{tabular}[c]{@{}c@{}}Pert.\\ Type\end{tabular} &
  \begin{tabular}[c]{@{}c@{}}Det. \\ Rate \%\end{tabular} &
  \begin{tabular}[c]{@{}c@{}}\# Q\\ to Det.\end{tabular} &
  \begin{tabular}[c]{@{}c@{}}Det.\\ Cov. \%\end{tabular} &
  \begin{tabular}[c]{@{}c@{}}Model\\ Acc.\end{tabular} &
  \# Q &
  \begin{tabular}[c]{@{}c@{}}Dist.\\ $\ell_2$\end{tabular} \\ \hline
Bandit                  & Score & \multicolumn{1}{c|}{$\ell_\infty$} & 100.0 & 1.0      & \multicolumn{1}{c|}{64.2} & 1.9  & 25   & 25.42 \\
NES                     & Score & \multicolumn{1}{c|}{$\ell_\infty$} & 100.0 & 10.3     & \multicolumn{1}{c|}{17.3} & 7.0  & 337  & 8.28  \\
Parsimonious            & Score & \multicolumn{1}{c|}{$\ell_\infty$} & 100.0 & 2.0      & \multicolumn{1}{c|}{96.7} & 3.8  & 282  & 25.24 \\
Sign                    & Score & \multicolumn{1}{c|}{$\ell_\infty$} & 100.0 & 2.0      & \multicolumn{1}{c|}{91.5} & 0.5  & 126  & 25.50 \\
Square                  & Score & \multicolumn{1}{c|}{$\ell_\infty$} & 100.0 & 2.0      & \multicolumn{1}{c|}{66.9} & 0.0  & 14   & 25.54 \\
ZOSignSGD               & Score & \multicolumn{1}{c|}{$\ell_\infty$} & 100.0 & 2.0      & \multicolumn{1}{c|}{50.2} & 12.5 & 322  & 8.53  \\
GeoDA                   & Label & \multicolumn{1}{c|}{$\ell_\infty$} & 100.0 & 1.0      & \multicolumn{1}{c|}{88.9} & 5.1  & 151  & 17.99 \\
HSJ                     & Label & \multicolumn{1}{c|}{$\ell_\infty$} & 100.0 & 7.3      & \multicolumn{1}{c|}{94.9} & 35.6 & 212  & 9.82  \\
Opt                     & Label & \multicolumn{1}{c|}{$\ell_\infty$} & 99.9  & 8.4      & \multicolumn{1}{c|}{81.4} & 61.2 & 646  & 0.98  \\
RayS                    & Label & \multicolumn{1}{c|}{$\ell_\infty$} & 100.0 & 4.4      & \multicolumn{1}{c|}{83.5} & 4.2  & 260  & 29.63 \\
SignFlip                & Label & \multicolumn{1}{c|}{$\ell_\infty$} & 100.0 & 8.5      & \multicolumn{1}{c|}{70.3} & 4.4  & 148  & 27.64 \\
SignOPT                 & Label & \multicolumn{1}{c|}{$\ell_\infty$} & 99.9  & 8.4      & \multicolumn{1}{c|}{69.8} & 55.9 & 570  & 1.32  \\
Bandit                  & Score & \multicolumn{1}{c|}{$\ell_2$}      & 100.0 & 1.0      & \multicolumn{1}{c|}{99.5} & 1.7  & 431  & 9.60  \\
NES                     & Score & \multicolumn{1}{c|}{$\ell_2$}      & 100.0 & 10.2     & \multicolumn{1}{c|}{32.8} & 61.2 & 571  & 0.45  \\
Simple                  & Score & \multicolumn{1}{c|}{$\ell_2$}      & 100.0 & 1.0      & \multicolumn{1}{c|}{99.9} & 53.6 & 883  & 0.88  \\
Square                  & Score & \multicolumn{1}{c|}{$\ell_2$}      & 100.0 & 2.0      & \multicolumn{1}{c|}{68.8} & 0.0  & 16   & 26.30 \\
ZOSignSGD               & Score & \multicolumn{1}{c|}{$\ell_2$}      & 100.0 & 2.0      & \multicolumn{1}{c|}{52.4} & 65.1 & 531  & 0.30  \\
Boundary                & Label & \multicolumn{1}{c|}{$\ell_2$}      & 100.0 & 7.2      & \multicolumn{1}{c|}{76.3} & 37.9 & 60   & 11.63 \\
GeoDA                   & Label & \multicolumn{1}{c|}{$\ell_2$}      & 100.0 & 1.0      & \multicolumn{1}{c|}{89.3} & 3.9  & 181  & 19.14 \\
HSJ                     & Label & \multicolumn{1}{c|}{$\ell_2$}      & 100.0 & 7.3      & \multicolumn{1}{c|}{93.4} & 11.4 & 255  & 22.21 \\
Opt                     & Label & \multicolumn{1}{c|}{$\ell_2$}      & 100.0 & 8.5      & \multicolumn{1}{c|}{67.9} & 41.2 & 610  & 16.71 \\
SignOPT                 & Label & \multicolumn{1}{c|}{$\ell_2$}      & 99.9  & 8.4      & \multicolumn{1}{c|}{62.9} & 36.7 & 485  & 17.54 \\
PointWise               & Label & \multicolumn{1}{c|}{Opt.}          & 100.0 & 1.0      & \multicolumn{1}{c|}{99.8} & 0.0  & 920  & 13.53 \\
SparseEvo               & Label & \multicolumn{1}{c|}{Opt.}          & 100.0 & 1.0      & \multicolumn{1}{c|}{99.9} & 0.0  & 1000 & 7.68  \\ \hline
\textbf{CA (sssp)}       & Label & \multicolumn{1}{c|}{Opt.}          & 0.0   & $\infty$ & \multicolumn{1}{c|}{0.0}  & 1.4  & 148  & 13.74 \\
\textbf{CA (bin search)} & Label & \multicolumn{1}{c|}{Opt.}          & 0.0   & $\infty$ & \multicolumn{1}{c|}{0.0}  & 0.0  & 603  & 33.14 \\ \hline
\end{tabular}%
}
\end{table}

We use the default setting from \cite{li2022blacklight} with a threshold of $25$. The results are presented in Table \ref{tab:blacklight_imagenet_resnet} and Tables \ref{tab:blacklight_cifar10_vgg}-\ref{tab:blacklight_cifar100_wrn} in Appendix \ref{apd: additional exp}. We have the following key observations: 1) Our certifiable attack consistently circumvents Blacklight with $0\%$ detection success rate, and $0\%$ detection coverage on all settings.  This indicates that none of the queries from our attack are detected. In contrast, existing black-box attacks are highly susceptible to Blacklight, with most achieving a $100\%$ detection success rate on various datasets and models. Even the most resilient attack, as shown in Appendix \ref{apd: additional exp}, Table \ref{tab:blacklight_cifar100_vgg}, attains an $86.5\%$ detection success rate on the CIFAR100 dataset using the VGG16 model. 2) With the strong ability to bypass the detection, our certifiable attack still maintains top attack performance on all the datasets and models such that the model accuracy can be attacked to $0\%$ with moderate $\ell_2$ perturbation size and few queries. 
The high attack accuracy and low detection rate of certifiable attacks stem from the randomness of \randAE~ and the guarantee of the attack success probability.

\subsubsection{Attack Performance under RAND-Pre \cite{qin2021RAND}}
\label{sec: exp rand pre-processing}
\begin{table}[]
\centering
\caption{Attack performance under RAND Pre-processing Defense on ResNet and ImageNet (Clean Accuracy: $67.0\%$)}\vspace{-0.1in}

\label{tab:imagenet_RAND_resnet}
\resizebox{0.47\textwidth}{!}{%
\begin{tabular}{lccccc}
\hline
Attack &
  \begin{tabular}[c]{@{}c@{}}Query\\ Type\end{tabular} &
  \begin{tabular}[c]{@{}c@{}}Perturbation\\ Type\end{tabular} &
  \# Query &
  Model Acc. &
  Dist. $\ell_2$ \\ \hline
\multicolumn{1}{l|}{Bandit}            & Score & \multicolumn{1}{c|}{$\ell_\infty$} & 10   & 6.7  & 25.26 \\
\multicolumn{1}{l|}{NES}               & Score & \multicolumn{1}{c|}{$\ell_\infty$} & 428  & 49.8 & 10.26 \\
\multicolumn{1}{l|}{Parsimonious}      & Score & \multicolumn{1}{c|}{$\ell_\infty$} & 243  & 62.7 & 25.12 \\
\multicolumn{1}{l|}{Sign}              & Score & \multicolumn{1}{c|}{$\ell_\infty$} & 116  & 40.6 & 25.20 \\
\multicolumn{1}{l|}{Square}            & Score & \multicolumn{1}{c|}{$\ell_\infty$} & 27   & 10.4 & 24.96 \\
\multicolumn{1}{l|}{ZOSignSGD}         & Score & \multicolumn{1}{c|}{$\ell_\infty$} & 428  & 49.4 & 10.36 \\
\multicolumn{1}{l|}{GeoDA}             & Label & \multicolumn{1}{c|}{$\ell_\infty$} & 150  & 40.0 & 18.08 \\
\multicolumn{1}{l|}{HSJ}               & Label & \multicolumn{1}{c|}{$\ell_\infty$} & 232  & 58.5 & 8.76  \\
\multicolumn{1}{l|}{Opt}               & Label & \multicolumn{1}{c|}{$\ell_\infty$} & 905  & 69.4 & 0.44  \\
\multicolumn{1}{l|}{RayS}              & Label & \multicolumn{1}{c|}{$\ell_\infty$} & 235  & 47.9 & 28.14 \\
\multicolumn{1}{l|}{SignFlip}          & Label & \multicolumn{1}{c|}{$\ell_\infty$} & 46   & 52.8 & 13.06 \\
\multicolumn{1}{l|}{SignOPT}           & Label & \multicolumn{1}{c|}{$\ell_\infty$} & 394  & 59.1 & 0.39  \\
\multicolumn{1}{l|}{Bandit}            & Score & \multicolumn{1}{c|}{$\ell_2$}      & 583  & 58.2 & 12.99 \\
\multicolumn{1}{l|}{NES}               & Score & \multicolumn{1}{c|}{$\ell_2$}      & 341  & 66.8 & 0.43  \\
\multicolumn{1}{l|}{Simple}            & Score & \multicolumn{1}{c|}{$\ell_2$}      & 258  & 67.2 & 0.10  \\
\multicolumn{1}{l|}{Square}            & Score & \multicolumn{1}{c|}{$\ell_2$}      & 18   & 13.6 & 25.96 \\
\multicolumn{1}{l|}{ZOSignSGD}         & Score & \multicolumn{1}{c|}{$\ell_2$}      & 249  & 67.3 & 0.28  \\
\multicolumn{1}{l|}{Boundary}          & Label & \multicolumn{1}{c|}{$\ell_2$}      & 38   & 49.2 & 15.12 \\
\multicolumn{1}{l|}{GeoDA}             & Label & \multicolumn{1}{c|}{$\ell_2$}      & 149  & 47.4 & 17.70 \\
\multicolumn{1}{l|}{HSJ}               & Label & \multicolumn{1}{c|}{$\ell_2$}      & 225  & 55.7 & 14.30 \\
\multicolumn{1}{l|}{Opt}               & Label & \multicolumn{1}{c|}{$\ell_2$}      & 1000 & 58.2 & 12.28 \\
\multicolumn{1}{l|}{SignOPT}           & Label & \multicolumn{1}{c|}{$\ell_2$}      & 406  & 52.2 & 15.41 \\
\multicolumn{1}{l|}{PointWise}         & Label & \multicolumn{1}{c|}{Optimized}  & 942  & 54.6 & 16.90 \\
\multicolumn{1}{l|}{SparseEvo}         & Label & \multicolumn{1}{c|}{Optimized}  & 1000 & 61.7 & 11.33 \\ \hline
\multicolumn{1}{l|}{\textbf{CA (sssp)}} & Label & \multicolumn{1}{c|}{Optimized}  & 154  & 1.7  & 13.98 \\
\multicolumn{1}{l|}{\textbf{CA (bin search)}} &
  Label &
  \multicolumn{1}{c|}{Optimized} &
  603 &
  0.0 &
  32.16 \\ \hline
\end{tabular}%
}
\end{table}

We follow \cite{qin2021RAND} to inject the Gaussian noise with standard deviation $0.02$ to the query (in the input space). The experimental results are presented in Table \ref{tab:imagenet_RAND_resnet} and Tables \ref{tab:cifar10_RAND_vgg}-\ref{tab:cifar100_RAND_wrn} in Appendix \ref{apd: additional exp}. Based on a comprehensive analysis of all results, it is evident that the RAND-Pre consistently reduces the attack success rate of existing black-box attacks. Specifically, the defense reduces the average attack success rate of empirical black-box attacks from $92\%$ to $30\%$ on CIFAR10, from $95\%$ to $29\%$ on CIFAR100, and from $69\%$ to $25\%$ on ImageNet, respectively. However, our attack still achieves the average attack success rate of $93\%$, $99\%$, and $99\%$ respectively on the three datasets under RAND-Pre. Further, we highlight that, with RAND-Pre applied across all datasets and models, 
the average $\ell_2$ perturbation size and number of queries in our certifiable attack \emph{decrease by $4.2\%$ and $2.1\%$}, respectively. This intriguing observation matches our findings in Section \ref{sec: exp diff variance} where a larger variance leads to smaller $\ell_2$ mean distance and \# RPQ in the certifiable attack. This is because the Gaussian noise injected by the defense  (e.g., $\epsilon_1 \sim \mathcal{N}(0,v)$) is added to the adversary's noise (e.g., $\epsilon_2 \sim \mathcal{N}(0,u)$), leading to a larger variance $v+u$ and hence further enhancing our attack. 

\vspace{-0.25in}

\subsubsection{Attack Performance under RAND-Post \cite{chandrasekaran2020RANDpost}}
\label{sec: exp rand post-processing}

\vspace{0.2in}

\begin{table}[]
\centering
\caption{Attack performance under RAND Post-processing Defense on ResNet and ImageNet (Clean Accuracy: $68.0\%$)}\vspace{-0.1in}
\label{tab:imagenet_post_RAND_resnet}
\resizebox{0.47\textwidth}{!}{%
\begin{tabular}{lccccc}
\hline
Attack &
  \begin{tabular}[c]{@{}c@{}}Query\\ Type\end{tabular} &
  \begin{tabular}[c]{@{}c@{}}Perturbation\\ Type\end{tabular} &
  \# Query &
  Model Acc. &
  Dist. $\ell_2$ \\ \hline
\multicolumn{1}{l|}{Bandit}            & Score & \multicolumn{1}{c|}{$\ell_\infty$} & 17   & 2.7  & 25.51 \\
\multicolumn{1}{l|}{NES}               & Score & \multicolumn{1}{c|}{$\ell_\infty$} & 378  & 18.6 & 9.53  \\
\multicolumn{1}{l|}{Parsimonious}      & Score & \multicolumn{1}{c|}{$\ell_\infty$} & 253  & 47.9 & 25.46 \\
\multicolumn{1}{l|}{Sign}              & Score & \multicolumn{1}{c|}{$\ell_\infty$} & 124  & 8.1  & 25.81 \\
\multicolumn{1}{l|}{Square}            & Score & \multicolumn{1}{c|}{$\ell_\infty$} & 18   & 0.8  & 25.44 \\
\multicolumn{1}{l|}{ZOSignSGD}         & Score & \multicolumn{1}{c|}{$\ell_\infty$} & 376  & 21.4 & 9.71  \\
\multicolumn{1}{l|}{GeoDA}             & Label & \multicolumn{1}{c|}{$\ell_\infty$} & 143  & 38.6 & 17.62 \\
\multicolumn{1}{l|}{HSJ}               & Label & \multicolumn{1}{c|}{$\ell_\infty$} & 212  & 52.7 & 8.82  \\
\multicolumn{1}{l|}{Opt}               & Label & \multicolumn{1}{c|}{$\ell_\infty$} & 1000 & 65.3 & 0.67  \\
\multicolumn{1}{l|}{RayS}              & Label & \multicolumn{1}{c|}{$\ell_\infty$} & 243  & 43.9 & 28.09 \\
\multicolumn{1}{l|}{SignFlip}          & Label & \multicolumn{1}{c|}{$\ell_\infty$} & 86   & 47.2 & 15.44 \\
\multicolumn{1}{l|}{SignOPT}           & Label & \multicolumn{1}{c|}{$\ell_\infty$} & 412  & 63.6 & 0.64  \\
\multicolumn{1}{l|}{Bandit}            & Score & \multicolumn{1}{c|}{$\ell_2$}      & 596  & 6.0  & 13.96 \\
\multicolumn{1}{l|}{NES}               & Score & \multicolumn{1}{c|}{$\ell_2$}      & 344  & 59.7 & 0.44  \\
\multicolumn{1}{l|}{Simple}            & Score & \multicolumn{1}{c|}{$\ell_2$}      & 241  & 58.9 & 0.10  \\
\multicolumn{1}{l|}{Square}            & Score & \multicolumn{1}{c|}{$\ell_2$}      & 23   & 0.4  & 26.46 \\
\multicolumn{1}{l|}{ZOSignSGD}         & Score & \multicolumn{1}{c|}{$\ell_2$}      & 275  & 61.4 & 0.29  \\
\multicolumn{1}{l|}{Boundary}          & Label & \multicolumn{1}{c|}{$\ell_2$}      & 24   & 48.0 & 12.64 \\
\multicolumn{1}{l|}{GeoDA}             & Label & \multicolumn{1}{c|}{$\ell_2$}      & 146  & 40.6 & 16.89 \\
\multicolumn{1}{l|}{HSJ}               & Label & \multicolumn{1}{c|}{$\ell_2$}      & 238  & 49.5 & 14.59 \\
\multicolumn{1}{l|}{Opt}               & Label & \multicolumn{1}{c|}{$\ell_2$}      & 1000 & 53.9 & 12.62 \\
\multicolumn{1}{l|}{SignOPT}           & Label & \multicolumn{1}{c|}{$\ell_2$}      & 411  & 46.3 & 15.96 \\
\multicolumn{1}{l|}{PointWise}         & Label & \multicolumn{1}{c|}{Optimized}  & 969  & 55.1 & 16.01 \\
\multicolumn{1}{l|}{SparseEvo}         & Label & \multicolumn{1}{c|}{Optimized}  & 1000 & 66.7 & 9.10  \\ \hline
\multicolumn{1}{l|}{\textbf{CA (sssp)}} & Label & \multicolumn{1}{c|}{Optimized}  & 147  & 1.4  & 13.70 \\
\multicolumn{1}{l|}{\textbf{CA (bin search)}} &
  Label &
  \multicolumn{1}{c|}{Optimized} &
  603 &
  0.0 &
  32.67 \\ \hline
\end{tabular}%
}
\end{table}

We follow \cite{chandrasekaran2020RANDpost} to inject the Gaussian noise with standard deviation $0.2$ to the output logits of each query (applied to both hard label-based and score-based attacks). The experimental results are presented in Table \ref{tab:imagenet_post_RAND_resnet}, and Tables \ref{tab:cifar10_post_RAND_vgg}-\ref{tab:cifar100_post_RAND_wrn}.  
Similarly, we find that  RAND-Post can strongly degrade the average attack success rate of hard label-based empirical attacks from $84\%$ to $41\%$ on CIFAR10, from $89\%$ to $45\%$ on CIFAR100, and from $60\%$ to $24\%$ on ImageNet, respectively. On the other hand, it moderately degrades the average attack success rate of score-based empirical attacks from $100\%$ to $91\%$ on CIFAR10, from $100\%$ to $95\%$ on CIFAR100, and from $72\%$ to $62\%$ on ImageNet. The discrepancy between label-based and score-based empirical attacks may stem from variations in the richness and smoothness of the query information. The loss value (score), providing a smoother evaluation, is less susceptible to noise interference and discloses finer-grained details. In contrast, labels are more likely to be impacted by injected noise, resulting in more randomized query outcomes.  However, our hard-label certifiable attack shows strong resilience against RAND-Post, by maintaining the average attack success rate at $93\%$, $99\%$, and $99\%$ on CIFAR10, CIFAR100, and ImageNet, respectively. This advantage over empirical attacks, particularly the label-based ones, originates from Randomized Parallel Querying---It precisely assesses query results with a lower bound of the ASP.

\subsubsection{Attack Performance under TRADES \cite{TRADES}}
\label{sec: exp adversarial training}
\begin{table}[]
\centering
\caption{Attack performance under TRADES Adversarial Training on ResNet and CIFAR10}\vspace{-0.1in}

\label{tab:AT_cifar10_resnet}
\resizebox{0.47\textwidth}{!}{%
\begin{tabular}{clccccc}
\hline
\multicolumn{1}{l}{Defense} &
  Attack &
  \begin{tabular}[c]{@{}c@{}}Query\\ Type\end{tabular} &
  \begin{tabular}[c]{@{}c@{}}Pert.\\ Type\end{tabular} &
  \# Query &
  Model Acc. &
  \begin{tabular}[c]{@{}c@{}}Dist.\\ $\ell_2$\end{tabular} \\ \hline
\multicolumn{1}{c|}{\multirow{16}{*}{\begin{tabular}[c]{@{}c@{}}\rotatebox[origin=c]{90}{\begin{minipage}{4cm}
\centering \bf $\mathbf{\ell_\infty}$ Adversarial Training \\ (Clean Accuracy: $\mathbf{80.9\%}$) \end{minipage}}\end{tabular}}} &
  \multicolumn{1}{l|}{Bandit} &
  Score &
  \multicolumn{1}{c|}{$\ell_\infty$} &
  1601 &
  10.2 &
  4.32 \\
\multicolumn{1}{c|}{} & \multicolumn{1}{l|}{NES}                     & Score & \multicolumn{1}{c|}{$\ell_\infty$} & 1474 & 27.4 & 2.27 \\
\multicolumn{1}{c|}{} & \multicolumn{1}{l|}{Parsimonious}            & Score & \multicolumn{1}{c|}{$\ell_\infty$} & 630  & 5.3  & 4.35 \\
\multicolumn{1}{c|}{} & \multicolumn{1}{l|}{Sign}                    & Score & \multicolumn{1}{c|}{$\ell_\infty$} & 439  & 4.4  & 4.37 \\
\multicolumn{1}{c|}{} & \multicolumn{1}{l|}{Square}                  & Score & \multicolumn{1}{c|}{$\ell_\infty$} & 854  & 5.7  & 4.39 \\
\multicolumn{1}{c|}{} & \multicolumn{1}{l|}{ZOSignSGD}               & Score & \multicolumn{1}{c|}{$\ell_\infty$} & 1196 & 37.0 & 2.21 \\
\multicolumn{1}{c|}{} & \multicolumn{1}{l|}{GeoDA}                   & Label & \multicolumn{1}{c|}{$\ell_\infty$} & 1358 & 41.9 & 1.99 \\
\multicolumn{1}{c|}{} & \multicolumn{1}{l|}{HSJ}                     & Label & \multicolumn{1}{c|}{$\ell_\infty$} & 2149 & 36.5 & 1.92 \\
\multicolumn{1}{c|}{} & \multicolumn{1}{l|}{Opt}                     & Label & \multicolumn{1}{c|}{$\ell_\infty$} & 1871 & 73.0 & 0.19 \\
\multicolumn{1}{c|}{} & \multicolumn{1}{l|}{RayS}                    & Label & \multicolumn{1}{c|}{$\ell_\infty$} & 721  & 7.3  & 4.29 \\
\multicolumn{1}{c|}{} & \multicolumn{1}{l|}{SignFlip}                & Label & \multicolumn{1}{c|}{$\ell_\infty$} & 2240 & 24.4 & 3.36 \\
\multicolumn{1}{c|}{} & \multicolumn{1}{l|}{SignOPT}                 & Label & \multicolumn{1}{c|}{$\ell_\infty$} & 832  & 69.3 & 0.15 \\
\multicolumn{1}{c|}{} & \multicolumn{1}{l|}{PointWise}               & Label & \multicolumn{1}{c|}{Opt.}          & 3460 & 9.3  & 4.38 \\
\multicolumn{1}{c|}{} & \multicolumn{1}{l|}{SparseEvo}               & Label & \multicolumn{1}{c|}{Opt.}          & 8691 & 9.1  & 5.10 \\ \cline{2-7} 
\multicolumn{1}{c|}{} & \multicolumn{1}{l|}{\textbf{CA (sssp)}}       & Label & \multicolumn{1}{c|}{Opt.}          & 548  & 21.2 & 4.29 \\
\multicolumn{1}{c|}{} & \multicolumn{1}{l|}{\textbf{CA (bin search)}} & Label & \multicolumn{1}{c|}{Opt.}          & 412  & 9.8  & 6.31 \\ \hline
\multicolumn{1}{c|}{\multirow{14}{*}{\begin{tabular}[c]{@{}c@{}}\rotatebox[origin=c]{90}{\begin{minipage}{4cm}
\centering \bf $\mathbf{\ell_2}$ Adversarial Training \\ (Clean Accuracy: $\mathbf{59.2\%}$) \end{minipage}} \end{tabular}}} &
  \multicolumn{1}{l|}{Bandit} &
  Score &
  \multicolumn{1}{c|}{$\ell_2$} &
  860 &
  1.5 &
  2.44 \\
\multicolumn{1}{c|}{} & \multicolumn{1}{l|}{NES}                     & Score & \multicolumn{1}{c|}{$\ell_2$}      & 3535 & 9.5  & 0.99 \\
\multicolumn{1}{c|}{} & \multicolumn{1}{l|}{Simple}                  & Score & \multicolumn{1}{c|}{$\ell_2$}      & 4062 & 2.1  & 1.29 \\
\multicolumn{1}{c|}{} & \multicolumn{1}{l|}{Square}                  & Score & \multicolumn{1}{c|}{$\ell_2$}      & 991  & 4.6  & 2.95 \\
\multicolumn{1}{c|}{} & \multicolumn{1}{l|}{ZOSignSGD}               & Score & \multicolumn{1}{c|}{$\ell_2$}      & 3505 & 15.1 & 0.77 \\
\multicolumn{1}{c|}{} & \multicolumn{1}{l|}{Boundary}                & Label & \multicolumn{1}{c|}{$\ell_2$}      & 771  & 40.7 & 1.19 \\
\multicolumn{1}{c|}{} & \multicolumn{1}{l|}{GeoDA}                   & Label & \multicolumn{1}{c|}{$\ell_2$}      & 1506 & 14.3 & 2.85 \\
\multicolumn{1}{c|}{} & \multicolumn{1}{l|}{HSJ}                     & Label & \multicolumn{1}{c|}{$\ell_2$}      & 1332 & 5.1  & 3.53 \\
\multicolumn{1}{c|}{} & \multicolumn{1}{l|}{Opt}                     & Label & \multicolumn{1}{c|}{$\ell_2$}      & 2890 & 41.6 & 2.39 \\
\multicolumn{1}{c|}{} & \multicolumn{1}{l|}{SignOPT}                 & Label & \multicolumn{1}{c|}{$\ell_2$}      & 1766 & 33.6 & 2.76 \\
\multicolumn{1}{c|}{} & \multicolumn{1}{l|}{PointWise}               & Label & \multicolumn{1}{c|}{Opt.}          & 4845 & 0.6  & 5.36 \\
\multicolumn{1}{c|}{} & \multicolumn{1}{l|}{SparseEvo}               & Label & \multicolumn{1}{c|}{Opt.}          & 9697 & 0.4  & 6.03 \\ \cline{2-7} 
\multicolumn{1}{c|}{} & \multicolumn{1}{l|}{\textbf{CA (sssp)}}       & Label & \multicolumn{1}{c|}{Opt.}          & 809  & 20.4 & 6.06 \\
\multicolumn{1}{c|}{} & \multicolumn{1}{l|}{\textbf{CA (bin search)}} & Label & \multicolumn{1}{c|}{Opt.}          & 461  & 0.0  & 8.18 \\ \hline
\end{tabular}%
}
\end{table}

We consider both $\ell_\infty$ and $\ell_2$ perturbations to generate adversarial examples, and TRADES respectively uses $\ell_2$ or  $\ell_\infty$  adversarial examples for adversarial training. We set the perturbation size to be \(\ell_\infty = 0.1\) and \(\ell_2 = 5\), following  \cite{TRADES}. 
We then evaluate all attacks against TRADES. The results on CIFAR10 are presented in Table \ref{tab:AT_cifar10_resnet}\footnote{It is computationally intensive and time-consuming to train TRADES on CIFAR100 and ImageNet}. We observe our attack requires much less query number than the empirical attacks. Also, our attack can achieve $100\%$ attack success rate (with the binary search localization), but at the cost of a relatively larger perturbation size.

\subsection{Ablation Study}
\label{section:exp ablation}

In this section, we explore in-depth our certifiable attack---we study its performance with varying noise variances, ASP thresholds, localization and shifting methods, and noise PDFs. We mainly show results on the image datasets and defer results on the audio dataset to Appendix \ref{apd: additional exp}, where similar performance can be observed. 

\subsubsection{Attack Performance on Different Noise Variances}
\label{sec: exp diff variance}
Table \ref{tab:diff variance} shows the performance of our attack with varying noise variances used in $\varphi$. We have the following key observations: 
1) As the variance increases, the $\ell_2$ perturbation size increases, since larger variance results in larger noise. 
2) The $\ell_2$ mean distance tends to decrease as the variance increases. This could be because that larger variance covers a larger decision space, and without moving the mean far away from the clean input, we can easily find a large portion of adversarial samples under the distribution with a large variance. 
3) As the variance increases, the number of RPQ decreases. This is because the larger variance usually leads to a larger shifting step. It takes fewer iterations to move to the decision boundary when the variance increases. 4) Finally, a larger certified accuracy means that it is easier to determine the \randAE. The results on CIFAR10 show that it is easier to find a small area of adversarial examples than a large area of adversarial examples. On ImageNet, we observe nearly $100\%$ certified accuracy, which means it is relatively easy to find the adversarial examples on datasets with a large number of classes (since $999$ out of $1,000$ classes in ImageNet are all false classes) or with high feature dimension.

\begin{table}[!t]
\small
  \caption{Attack performance of our certifiable attack with varying Gaussian noise variances $\sigma$ ($p=90\%$) 
  }
    \centering
    \resizebox{0.47\textwidth}{!}{%
    \begin{tabular}{c|c| c c c c}
    \hline
         & $\sigma$ & Dist. $\ell_2$ &Mean Dist. $\ell_2$ & \# RPQ & Certified Acc.  \\
    \hline
         \multirow{3}*{\rotatebox{90}{\tiny CIFAR10}} 
                                                &0.10  & 7.39 &3.96  &18.34 &94.17\% \\ 
                                                &0.25  & 12.95 &2.34  &14.35 &91.21\% \\
                                                &0.50  & 19.41 &0.43 &11.38  &90.00\% \\
    \hline
             \multirow{3}*{\rotatebox{90}{\tiny  ImageNet}} 
                                                     & 0.10 &41.80 &16.78  &32.55  &99.80\%  \\ 
                                                     & 0.25 &87.47 &16.78 &17.02 &99.60\%  \\
                                                     & 0.50 &135.47 &2.27 &8.31  &100.00\% \\
                                                     \hline
    \end{tabular}}
    \label{tab:diff variance}
\vspace{-0.05in}
\end{table}

\begin{table}[!t]
\small
  \caption{Attack performance of our certifiable attack with varying $p$ under the Gaussian variance $\sigma=0.25$}\vspace{-0.1in}
    \centering
    \resizebox{0.47\textwidth}{!}{%
    \begin{tabular}{c|c|c c c c}
    \hline
         & p & Dist. $\ell_2$ &Mean Dist. $\ell_2$ & \# RPQ & Certified Acc.  \\
        \hline
             \multirow{6}{*}[-1mm]{\rotatebox{90}{CIFAR10}} 
                                                     & 50\% &12.65 &1.63  &9.34  &97.17\%  \\ 
                                                     & 60\% &12.72 &1.86 &11.09 &95.85\%  \\
                                                     & 70\% &12.80 &2.05 &11.94  &94.72\% \\
                                                     & 80\% &12.87 &2.18 &12.37  &93.17\% \\
                                                     & 90\% &12.95 &2.34 &14.35 &91.21\%  \\
                                                     & 95\% &13.09 &2.65 &15.93  &90.37\% \\
    \hline
             \multirow{6}{*}[-1mm]{\rotatebox{90}{ImageNet}} 
                                                     & 50\% &85.88 &9.89  &12.85  &100.00\%  \\ 
                                                     & 60\% &86.20 &11.30 &13.63 &100.00\%  \\
                                                     & 70\% &86.45 &12.64 &14.33  &100.00\% \\
                                                     & 80\% &87.03 &14.64 &16.02  &100.00\% \\
                                                     & 90\% &87.47 &16.78 &17.02 &99.60\%  \\
                                                     & 95\% &88.42 &19.98 &19.81  &100.00\% \\   
    \hline
    \end{tabular}}
    \label{tab:diff p}
\vspace{-0.05in}
\end{table}

\begin{table}[]
\centering
\caption{Attack performance of our certifiable attack on different localization/refinement algorithms ($\sigma=0.25$, $p=90\%$)}
\vspace{-0.05in}
\label{tab:ablation study}
\resizebox{0.47\textwidth}{!}{%
\begin{tabular}{c|c|cccc}
\hline
Localization  & Refinement & Dist. $\ell_2$ & Mean Dist. $\ell_2$ & \# RPQ & Cert. Acc.    \\ \hline
sssp          & none       & 11.46          & 1.35                & 2.30   & 92.54 \\
binary search & none       & 11.29          & 0.34                & 9.07   & 92.54 \\
random        & geo.       & 11.80          & 1.73                & 67.53  & 92.54 \\
sssp          & geo.       & 11.20          & 0.49                & 3.70   & 91.54 \\
binary search & geo.       & 11.28          & 0.27                & 10.08  & 92.53 \\ \hline
\end{tabular}%
}\vspace{-0.05in}
\end{table}

\subsubsection{Attack Performance on Different ASP Thresholds}

We study the relationship between the performance of our attack and the ASP threshold, and Table \ref{tab:diff p} shows the results. As $p$ increases, so do the $\ell_2$ perturbation size, the $\ell_2$ mean distance, and the number of RPQ. On one hand, a larger $p$ means it requires more adversarial examples to fall into the false classes. When the noise variance is fixed, the mean of the \randAE~ should be further away from the decision boundary to allow more adversarial examples to fall into the false classes. On the other hand, the smaller $p$ results in a larger shifting distance, which depends on the gap between $p$ and $\underline{p_{adv}}$ (see the Gaussian-case of Theorem \ref{thm1} in Appendix \ref{apd:thm2}). With a larger shifting distance, the required number of RPQ can be fewer. We also observe that a smaller $p$ results in a higher certified accuracy on CIFAR10, since a smaller $p$ allows more ``failed" adversarial examples. On ImageNet, the certified accuracy is consistently $\sim100\%$, no matter $p$'s value. This might still because it is much easier to find adversarial examples with a much larger number of classes.

\subsubsection{Attack Performance on Different Localization/Refinement Algorithms} 
\label{sec:ablation}

In this experiment, 
we compare our proposed Smoothed SSP and binary-search localization methods with the random localization baseline; and compare our proposed geometric shifting method with a no-shifting baseline. 
Results are shown in Table \ref{tab:ablation study}. We observe that the combination of the localization and refinement methods yields the smallest perturbation size, i.e., the smallest Dist. $\ell_2$ and Mean Dist. $\ell_2$. This demonstrates that they are both effective in improving the imperceptibility of adversarial examples.  

\vspace{0.05in}

\noindent \textbf{Visualization}. We also visualize the adversarial examples $x_{adv}$  while crafting the Certifiable Attack for Binary-search Localization (Figure \ref{fig:visual_binary}) and SSSP Localization (Figure \ref{fig:visual_sssp}). It shows that when $\sigma=0.025$, both the Binary-search and SSSP-based certifiable attack can craft imperceptible perturbations. The difference is that the binary search method starts from a random $x'$ and requires more \# RPQ to update the \randAE, while the SSSP can easily find an initial \randAE with small perturbation and thus requires fewer \# RPQ.
\begin{figure*}
    \centering
    \vspace{-0.05in}
    \includegraphics[width=0.9\textwidth]{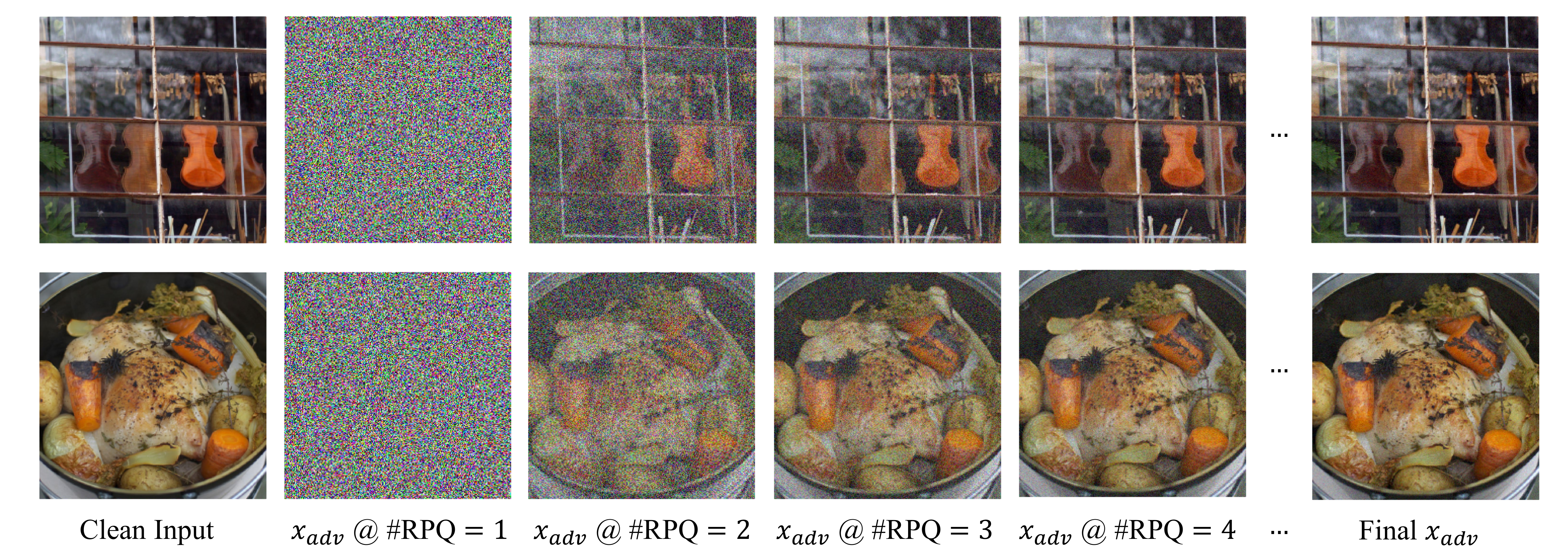}
    \vspace{-0.05in}
    \caption{Visualization of successful adversarial examples crafting by certifiable attack with binary-search localization}
    \label{fig:visual_binary}
\vspace{-0.05in}
\end{figure*}

\begin{figure}
    \centering
    \vspace{-0.05in}
    \includegraphics[width=0.44\textwidth]{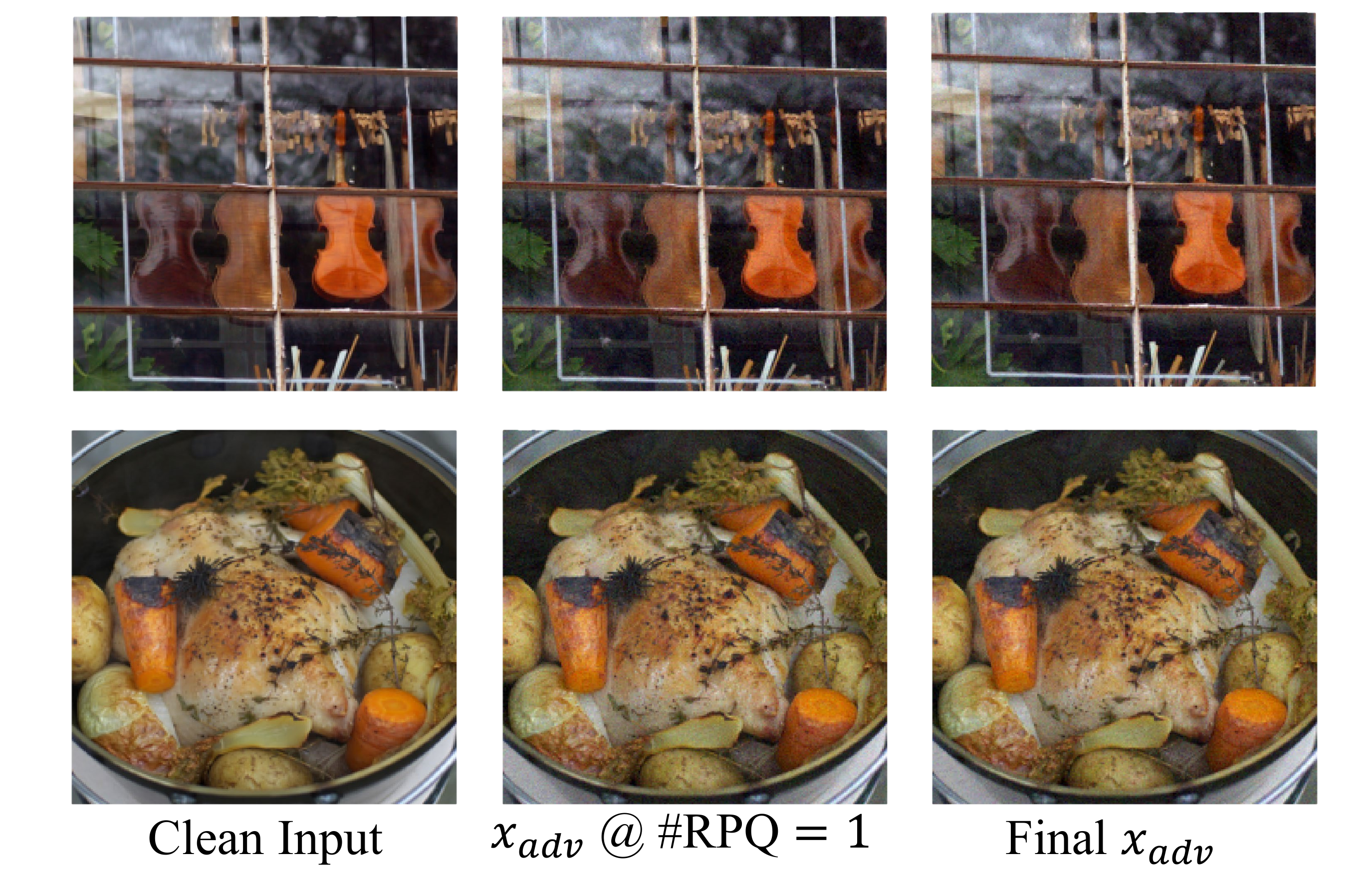}
    \vspace{-0.05in}
    \caption{Visualization of successful adversarial examples crafting by certifiable attack with SSSP localization (SSSP requires fewer \# RPQ)}
    \label{fig:visual_sssp}
\vspace{-0.05in}
\end{figure}

\vspace{-0.05in}
\subsubsection{Attack Performance on Different Noise Distributions}

Our attack can use any continuous noise distribution to craft the \randAE. Besides the Gaussian noise distribution, in this experiment, we also evaluate the performance of our certifiable attack using other noise distributions including the Cauthy distribution, Hyperbolic Secant distribution, and general normal distributions. Note that we adjust the parameters in these distributions to ensure consistent variances for a fair comparison. 

The results are presented in Table \ref{tab:diff pdf}, and the noise distributions are plotted in Figure \ref{fig:MC and pdfs} in Appendix. On both datasets, we observe the $\ell_2$ perturbation size is decreasing while the $\ell_2$ mean distance is increasing as the PDF of the noise distribution is more centralized.
This result may share a similar nature with results in Table \ref{tab:diff variance}---when the adversarial samples are more widely distributed, they tend to fall into an adversarial class (the majority of all classes). It is hard to determine which distribution is better since there is a trade-off between the perturbation size and the number of RPQ.

\begin{table*}[!t]
\small
\caption{Attack performance of our certifiable attack with different noise distributions}\vspace{-0.1in}

    \centering
    \begin{tabular}{c |c |c |c |c |c c c c}
    \hline
        & Distribution          & Density & Parameter & $\sqrt{\|\epsilon\|^2/d}$ & Dist. $\ell_2$ &Mean Dist. $\ell_2$ & \# RPQ & Certified Acc.  \\
    \hline
                 \multirow{5}{*}[-2mm]{\rotatebox{90}{CIFAR10}} 
    &     Gaussian              & $\propto e^{-|z/a|^2}$              &$a=0.25$ &0.25 & 12.95 &2.34   &14.35 &91.21\% \\
   &      Cauthy                & $\propto \frac{a^2}{z^2+a^2}$  &$a=0.01969$ &0.25 & 7.82 &4.87   &32.77 &94.12\% \\
   &Hyperbolic Secant           & $\propto sech(|z/a|)$               &$a=0.1592$ &0.25 & 12.51 &2.43  &14.59 &91.67\% \\
   &General Normal ($b=1.5$)& $\propto e^{-|z/a|^b}$          &$a=0.2909$, $b=1.5$ &0.25 & 12.74 &2.37  &14.15 &91.39\% \\
   &General Normal ($b=3.0$)& $\propto e^{-|z/a|^b}$          &$a=0.4092$, $b=3$ &0.25 & 13.16 &2.38  &14.15 &91.25\%\\ 
    \hline
                 \multirow{5}{*}[-2mm]{\rotatebox{90}{ImageNet}} 
    &     Gaussian              & $\propto e^{-|z/a|^2}$              &$a=0.25$  &0.25 & 87.47 &16.78  &17.02 &99.60\% \\
    &     Cauthy                & $\propto \frac{a^2}{z^2+a^2}$  &$a=0.01969$ &0.25 & 46.18 &23.94   &59.94 &99.60\% \\
   &Hyperbolic Secant           & $\propto sech(|z/a|)$               &$a=0.1592$ &0.25 & 85.57 &21.29  &20.89 &99.80\% \\
   &General Normal ($b=1.5$)& $\propto e^{-|z/a|^b}$          &$a=0.2909$, $b=1.5$ &0.25 & 86.69 &19.05  &17.58 &99.80\% \\
   &General Normal ($b=3.0$)& $\propto e^{-|z/a|^b}$          &$a=0.4092$, $b=3$ &0.25 & 88.51 &15.58  &14.99 &100.00\%\\ 
    \hline
    \end{tabular}
    
    \label{tab:diff pdf}
\end{table*}

\subsection{Defending against Our Certifiable Attack}
In this subsection, we discuss potential defenses and mitigation strategies against our attacks. 

\noindent {\bf Noise Detection based Defenses:} 
Our certifiable attack injects noise into the adversarial examples. Here, we suppose the adversary is aware of the noise injection and designs a detection method by training a binary classifier to distinguish the noise-injected inputs and clean inputs. Specifically, the defender (i.e., model owner) uses ResNet110 (as powerful as the target model) to train a noise detector to distinguish the inputs with and without noise. The experimental results show the noise detection rate can be as high as $99\%$ with the noise variance $\sigma=0.5$, which means this detector can be used as a strong defense against our certifiable attacks with a larger noise. However, this defense does not work when the noise scale is smaller (i.e., $\sigma=0.025$), where the detection rate is less than $1\%$. Especially, the adversary may design a novel method to hide this noise in the image texture, e.g., using the diffusion model for denoising, which may circumvent the detection. 

\vspace{0.05in}

\noindent {\bf White-Box Adaptive Defenses against Our Attack:} We assume the model owner knows the noise distribution used by our attack and performs a ``white-box" defense.  Particularly, it applies a denoiser to eliminate the injected noise, so that the adversarial examples can be restored to clean inputs. The denoiser can be deployed as a pre-processing module and is pre-trained by the model owner. Specifically,  
we use a U-Net structure~\cite{ronneberger2015u} as the denoiser and denote it as $\mathcal{D}$. Then, the loss function for the training is 
\begin{equation}
    \mathbb{E}_{\epsilon \sim \mathcal{N}(0,\sigma_d)}[||\mathcal{D}(x+\epsilon)-x||_2 + ||f(\mathcal{D}(x+\epsilon))-f(x)||_2]
\end{equation}

Taking Gaussian noise as an example (e.g., the model owner knows the Gaussian variance $\sigma=0.25$  used in the certifiable attack), we train the denoiser to eliminate Gaussian noise with $\sigma=0.25$ while evaluating the certifiable attack with Gaussian noise generated by different $\sigma$. Table \ref{tab:WB defense} shows the results. We can observe that this defense can significantly degrade the performance of a certifiable attack. Notably, by choosing the same variance $\sigma$ as the adversary, the adaptive defense can increase the Mean Dist. $\ell_2$ significantly. However, the certified accuracy is still near $90\%$. 

\begin{table}[!t]
    \centering
       \caption{White-box adaptive defense against our attack ($\sigma=0.25$, $p=90\%$) on CIFAR10}
    \begin{tabular}{c|c c c c }
    \hline
        Defense Para. & Dist. $\ell_2$ &Mean Dist. $\ell_2$ & \# RPQ & Cert. Acc.\\
    \hline
        $\sigma_{d}=0.10$  &9.99 &7.73 &34.11 &87.51\%                                                        \\
       $\sigma_{d}=0.25$     &15.40 &10.21 & 29.80 &88.31\%                                                         \\
        $\sigma_{d}=0.50$   &20.46 &8.11 & 26.52 & 86.56\%                                                     \\
    \hline
        
    \end{tabular}
 
    \label{tab:WB defense}
\end{table}

\section{Related Work}
\label{sec:RelatedWork}
\noindent\textbf{Adversarial Attack}. It aims to mislead learnt ML models by perturbing testing data with imperceptible perturbations. It can be divided into white-box attacks \cite{DBLP:journals/corr/GoodfellowSS14,DBLP:journals/corr/Moosavi-Dezfooli16,DBLP:conf/sp/Carlini017,DBLP:conf/iclr/MadryMSTV18,DBLP:conf/icml/WongSK19} and black-box attacks \cite{DBLP:conf/ccs/ChenZSYH17,DBLP:conf/icml/IlyasEAL18,DBLP:conf/iclr/ChengLCZYH19,DBLP:conf/eccv/BhagojiHLS18,DBLP:conf/ccs/DuFYCT18,DBLP:conf/iclr/BrendelRB18,DBLP:conf/ccs/ChenZSYH17,DBLP:journals/corr/abs-1712-07113,DBLP:conf/sp/ChenJW20,DBLP:conf/ccs/PapernotMGJCS17,DBLP:conf/cvpr/ShiWH19,DBLP:conf/cvpr/DongPSZ19,DBLP:conf/cvpr/NaseerKHKP20,DBLP:conf/iclr/BrendelRB18,DBLP:conf/iccv/BrunnerDTK19,DBLP:conf/icml/LiLWZG19,DBLP:conf/icml/GuoGYWW19,DBLP:conf/eccv/AndriushchenkoC20}, per the access that the adversary holds. 
White-box attacks have full access to the model parameters, and  
can leverage the gradient of the loss function w.r.t. the inputs to guide the adversarial example generation. Instead, black-box attacks only know the outputs (in the form of prediction scores or labels) of a target model via sending queries. It is widely believed that black-box attack is more practical in real-world scenarios \cite{DBLP:conf/ccs/PapernotMGJCS17,bhambri2019survey,DBLP:conf/sp/ChenJW20}. Therefore, we focus on the black-box attacks in this paper. 

\vspace{0.05in}

\noindent\textbf{Black-Box Attack}.
Existing black-box attack methods can be classified into three types: gradient estimation based \cite{DBLP:conf/ccs/ChenZSYH17,DBLP:conf/icml/IlyasEAL18,DBLP:conf/iclr/ChengLCZYH19,DBLP:conf/eccv/BhagojiHLS18,DBLP:conf/ccs/DuFYCT18,SunLH20,wang2022bandits,qu2023certified,wang2023turning}, surrogate models based \cite{DBLP:conf/ccs/PapernotMGJCS17,DBLP:conf/cvpr/ShiWH19,DBLP:conf/cvpr/DongPSZ19,DBLP:conf/cvpr/NaseerKHKP20}, or local search based algorithms \cite{DBLP:conf/iclr/BrendelRB18,DBLP:conf/iccv/BrunnerDTK19,DBLP:conf/icml/LiLWZG19,DBLP:conf/icml/GuoGYWW19,DBLP:conf/eccv/AndriushchenkoC20,mu2021a,fan2021reinforcement}.  
Gradient estimation based attack is mainly based on zero-order estimation since the true gradient is unknown~\cite{DBLP:conf/ccs/ChenZSYH17}. Surrogate model-based methods first perform white-box attacks on an offline surrogate model to generate adversarial examples, and then use these generated adversarial examples to test the target model. The attack performance largely depends on the transferability of such generated adversarial examples. Local search-based methods craft adversarial examples by searching the effective perturbation direction, e.g., Boundary Attack \cite{DBLP:conf/iclr/BrendelRB18} traverses the decision boundary to craft the \textit{least imperceptible} perturbations. 

All existing black-box attacks rely on querying the target model until finding a successful adversarial example or reaching the maximum number of queries. However, none of them can ensure the success rate of the adversarial examples that have not been queried. Further, they are shown to be easily detected/removed via adversarial detection and randomized pre/post-processing-based defenses.

\vspace{0.05in}

\noindent\textbf{Empirical Defense}. It defends against adversarial attacks without guarantees. Empirical defenses against white-box attacks can be roughly categorized into four classes. Gradient-masking defenses \cite{DBLP:conf/sp/PapernotM0JS16,DBLP:conf/iclr/XieWZRY18,DBLP:conf/iclr/DhillonALBKKA18} modify the model inference process to obstacle the gradient computation. Input-transformation defenses \cite{DBLP:conf/iclr/GuoRCM18,DBLP:conf/iclr/BuckmanRRG18,DBLP:conf/iclr/SamangoueiKC18,DBLP:conf/cvpr/LiaoLDPH018,DBLP:conf/iclr/SongKNEK18,hong2022eye} use pre-processing methods to transform the inputs so that the malicious effects caused by the perturbations can be reduced. 
Detection-based defenses \cite{DBLP:conf/icml/RothKH19,DBLP:conf/icml/Tramer22,DBLP:conf/iccv/LuIF17,DBLP:conf/ccs/MengC17,jain2022adversarial} identify features that expect to separate adversarial examples and clean examples, and train a binary classifier to detect adversarial examples. Another branch of works \cite{PIHA,OSD,IIoT-SDA,li2022blacklight} detects the adversarial examples based on the similarity of the queries, demonstrating high detection accuracy in practice. Among these, Blacklight \cite{li2022blacklight} has shown supreme detection performance without assumptions on the user accounts. These three types of defenses show certain effectiveness when they target specific known attacks, but can be broken by adaptive attacks~\cite{athalye2018obfuscated}.
Lastly, adversarial training-based defenses  \cite{DBLP:conf/iclr/MadryMSTV18,DBLP:conf/nips/ShafahiNG0DSDTG19,DBLP:conf/iclr/TramerKPGBM18,DBLP:conf/nips/TramerB19} have achieved the SOTA performance against adaptive attacks. The main idea is to augment training data with "adversarial examples", but they are reassigned the correct label. As to defend against \emph{black-box} attacks, RAND-Post \cite{chandrasekaran2020RANDpost}, RAND-Pre \cite{qin2021RAND}, Adversarial Training based TRADES \cite{DBLP:conf/iclr/MadryMSTV18}, and Blacklight \cite{li2022blacklight} 
are the SOTA in each category. Thus, we evaluated our certifiable attack under these defenses.

\vspace{0.05in}

\noindent\textbf{Certified Defense}. Certified defense \cite{DBLP:conf/cav/KatzBDJK17, DBLP:journals/constraints/FischettiJ18, DBLP:conf/icml/AnilLG19,DBLP:conf/icml/WongK18,hong2022unicr,zhang2024text} was proposed to guarantee constant classification prediction on a set of adversarial examples. Recently, randomized smoothing (RS) \cite{DBLP:conf/icml/CohenRK19} has achieved great success in the certified defense since it is the first method to certify arbitrary classifiers of any scale. Specifically, RS can guarantee the prediction if the perturbation is bounded by a distance in $\ell_p$-norm, i.e., certified radius \cite{DBLP:conf/icml/CohenRK19,teng2019ell_1,DBLP:conf/icml/YangDHSR020,hong2022unicr}. RS adds noise from a distribution (e.g., Gaussian) to the inputs and uses hypothesis testing to quantify the prediction probability. Then the bound on the perturbations (usually a $\ell_p$ norm constraint) for ensuring the consistent prediction is derived. This method is widely used in certified defense to ensure consistent and correct prediction under attack. However, in this paper, we propose to use this method to ensure consistent and wrong prediction on the \randAE, resulting in a reliable and strong certifiable attack.

\section{Conclusion}
\label{sec:Conclusion}
Certifiable attack lays a novel direction for adversarial attacks, enabling the transition from deterministic to probabilistic adversarial attacks. Compared with empirical black-box attacks, certifiable attacks share significant benefits including breaking SOTA strong detection and randomized defense, revealing consistent and severe robustness vulnerability of models, and guaranteeing the minimum ASP for numerous unique AEs without verifying via the query.

\section*{Acknowledgments}
 We sincerely thank the anonymous reviewers for their constructive comments and suggestions. This work is supported in part by the National Science Foundation (NSF) under Grants No. CNS-2308730, CNS-2302689, CNS-2319277, CMMI-2326341, ECCS-2216926, CNS-2241713, CNS-2331302 and CNS-2339686. It is also partially supported by the Cisco Research Award and the Synchrony Fellowship. 

\bibliographystyle{plain}
\balance
\bibliography{CBA}

\newpage
\appendices
\label{sec:Appendix}
\section{Proofs}
\label{apd:proofs}

\subsection{Proof of Theorem \ref{thm1}}
\label{apd:thm1 proof}
The proof of Theorem \ref{thm1} is based on the Neyman-Pearson Lemma, so we first review the Neyman-Pearson Lemma.

\begin{Lem}{(\textbf{Neyman-Pearson Lemma})}
Let $X$ and $Y$ be random variables in $\mathbb{R}^d$ with densities $\mu_X$ and $\mu_Y$. Let $f:\mathbb{R}^d \rightarrow \{0,1\}$ be a random or deterministic function. Then:

(1) If $S=\{z\in \mathbb{R}^d: \frac{\mu_Y(z)}{\mu_X(z)}\leq{t}\}$ for some $t>0$ and $\mathbb{P}(f(X)=1)\ge \mathbb{P}(X\in S)$, then $\mathbb{P}(f(Y)=1) \ge \mathbb{P}(Y\in S)$;

(2) If $S=\{z\in \mathbb{R}^d: \frac{\mu_Y(z)}{\mu_X(z)}\ge {t}\}$ for some $t>0$ and $\mathbb{P}(f(X)=1)\leq \mathbb{P}(X\in S)$, then $\mathbb{P}(f(Y)=1) \leq \mathbb{P}(Y\in S)$.
\label{NP lemma}
\end{Lem}

Let $x\in\mathbb{R}^d$ be any clean input with label $y$. 
Let noise $\epsilon$ be drawn from any continuous distribution $\varphi(0,\boldsymbol{\kappa})$. Let $x'\in \mathbb{R}^d$ be any input. Denote $X=x'+\epsilon$, and $X_\delta=x'+\delta+\epsilon$. 
Let $f:\mathbb{R}^d\rightarrow\mathbb{R}^1$ be any deterministic or random function. For each input $x$, we can consider two classes: $y$ or $\neq y$, so the problem can be considered as a binary classification problem.
Let the lower bound of randomized parallel query on $x'$ denoted as $Q(x')=\underline{p_{adv}}$. Define the half set:

\begin{equation}
    A:=\{z:\frac{\varphi(z-\delta,\boldsymbol{\kappa})}{\varphi(z,\boldsymbol{\kappa})}\leq \tau\}
\end{equation}
where the auxiliary parameter $\tau$ is picked to suffice:
\begin{equation}
    \mathbb{P}(X\in A)=\mathbb{P}[\frac{\varphi(x'+\epsilon-\delta,\boldsymbol{\kappa})}{\varphi(x'+\epsilon,\boldsymbol{\kappa})}\leq \tau]=\underline{p_{adv}}
\label{eq:10}
\end{equation}

Suppose $\underline{p_{adv}}$ and the ASP Threshold $p$ satisfy $\underline{p_{adv}}\geq p$, then 
\begin{equation}
    \mathbb{P}[f(X)\neq y] \geq p_{adv} = \mathbb{P}(X\in A)
\end{equation}

Using Neyman-Pearson Lemma (considering $X_\delta=X+\delta$ as $Y$ in Neyman-Pearson Lemma, and $\neq y$ as class $1$), we have:
\begin{equation}
    \mathbb{P}[f(X_\delta)\neq y]\geq \mathbb{P}[X_\delta\in A]
\end{equation}
which is equal to
\begin{equation}
    \mathbb{P}[f(X_\delta)\neq y]\geq \mathbb{P}[\frac{\varphi(x'+\epsilon,\boldsymbol{\kappa})}{\varphi(x'+\epsilon+\delta,\boldsymbol{\kappa})}\leq\tau]
\end{equation}

If $\mathbb{P}(\frac{\varphi(x'+\epsilon,\boldsymbol{\kappa})}{\varphi(x'+\epsilon+\delta,\boldsymbol{\kappa})}\leq\tau)\geq p$, we can guarantee that 
\begin{equation}
    \mathbb{P}[f(X_\delta)\neq y]\geq p    
\end{equation}
which means the probability of classifying $X_\delta$ as adversarial examples is greater than $p$. Therefore, the distribution of $X_\delta$ can be guaranteed to have the attack success probability larger than $p$.

Considering Eq. (\ref{eq:10}), we have $\tau=\Phi^{-1}_-(\underline{p_{adv}})$, 

where $\Phi^{-1}_-$ is the inverse CDF 
of random variable $\frac{\varphi(x'+\epsilon-\delta,\boldsymbol{\kappa})}{\varphi(x'+\epsilon,\boldsymbol{\kappa})}$. Therefore, substitute $\tau$ in $\mathbb{P}[\frac{\varphi(x'+\epsilon,\boldsymbol{\kappa})}{\varphi(x'+\epsilon+\delta,\boldsymbol{\kappa})}\leq\tau]\geq p$, we have
\begin{equation}
    \Phi_+[\Phi^{-1}_-(\underline{p_{adv}})]\geq p
\end{equation}
where $\Phi_+$ is the CDF of random variable $\frac{\varphi(x'+\epsilon,\boldsymbol{\kappa})}{\varphi(x'+\epsilon+\delta,\boldsymbol{\kappa})}$.
The ratios can be further simplified as
\begin{equation}
    \frac{\varphi(x'+\epsilon-\delta,\boldsymbol{\kappa})}{\varphi(x'+\epsilon,\boldsymbol{\kappa})}=\frac{\varphi(\epsilon-\delta,\boldsymbol{\kappa})}{\varphi(\epsilon,\boldsymbol{\kappa})}
\end{equation}
\begin{equation}
    \frac{\varphi(x'+\epsilon,\boldsymbol{\kappa})}{\varphi(x'+\epsilon+\delta,\boldsymbol{\kappa})}=\frac{\varphi(\epsilon,\boldsymbol{\kappa})}{\varphi(\epsilon+\delta,\boldsymbol{\kappa})}
\end{equation}

Now we complete the proof of Theorem \ref{thm1}.

\subsection{Certifiable Attack: Gaussian Noise}
\label{apd:thm2}
\begin{corollary}{(\textbf{Certifiable Adversarial Shifting: Gaussian Noise})} Under the same condition with Theorem~\ref{thm1}, and let $\epsilon$ be a noise drawn from Gaussian distribution $\mathcal{N}(0,\sigma)$. Then, if the randomized query on the adversarial input satisfies Eq. (\ref{thm1 condition}): 
\begin{equation}
    \mathbb{P}[f(x'+\epsilon)\neq y]\geq \underline{p_{adv}}=Q(x')\geq p
\end{equation}

Then $\mathbb{P}[f(x'+\delta+\epsilon)\neq y]\geq p$ is guaranteed for any 
shifting vector $\delta$ when
\begin{equation}
    ||\delta||_2 \le \sigma [\Phi^{-1}(\underline{p_{adv}})-\Phi^{-1}(p)]
\end{equation}
where $\Phi^{-1}$ denotes the inverse of the standard Gaussian CDF.
\label{thm2}
\end{corollary}

\begin{proof}

The Gaussian distribution is $\mu(x) \propto e^{-\frac{x^2}{2\sigma^2}}$, thus
\begin{equation}
    \frac{\mu(x-\delta)}{\mu(x)}= e^{(2x\delta-\delta^2)/(2\sigma^2)}
\end{equation}

Let $\tau:=exp((2\Phi_\sigma^{-1}(\underline{p_{adv}})\delta-\delta^2)/(2\sigma^2))$, where $\Phi^{-1}_\sigma$ denotes the inverse Gaussian CDF with variance $\sigma$. Let random variables $X:=x'+\epsilon$ and $X_\delta:=x'+\epsilon+\delta$. Then we have:
\begin{align}
    \mathbb{P}(X\in A) &=\mathbb{P}[\frac{\mu(X-\delta)}{\mu(X)}\leq \tau] \\
    &=\mathbb{P}[exp((2X\delta-\delta^2)/(2\sigma^2))] \leq \\
    &exp[(2\Phi_\sigma^{-1}(\underline{p_{adv}})\delta-\delta^2)/(2\sigma^2)]  \\
    &=\mathbb{P}[X \le \Phi_\sigma^{-1}(\underline{p_{adv}})]\\
    &=\underline{p_{adv}} \\
\end{align}

Using Neyman-Pearson Lemma, we have:
\begin{equation}
    \mathbb{P}[f(X_\delta)\neq y]\geq \mathbb{P}[(X_\delta)\in A]
\end{equation}

Since
\begin{align}
    \mathbb{P}[X_\delta\in A] &=\mathbb{P}[\frac{\mu(X)}{\mu(X+\delta)}\leq\tau] \\
        &=\mathbb{P}[exp((2X\delta+\delta^2)/(2\sigma^2))] \leq \\
    &exp[(2\Phi_\sigma^{-1}(\underline{p_{adv}})\delta-\delta^2)/(2\sigma^2)]  \\
    &=\mathbb{P}[2X\delta +\delta^2 \le (2\Phi_\sigma^{-1}(\underline{p_{adv}})\delta-\delta^2)]\\
    &=\mathbb{P}[X \le \Phi_\sigma^{-1}(\underline{p_{adv}}) - ||\delta||]\\
    &=\mathbb{P}[\frac{X}{\sigma} \le \Phi^{-1}(\underline{p_{adv}}) - \frac{||\delta||}{\sigma}]
\end{align} where $\Phi^{-1}$ denotes the inverse standard Gaussian CDF.

To guarantee that  $\mathbb{P}[f(X_\delta)\neq y]\geq p$, we need:
\begin{align}
    \mathbb{P}[X_\delta\in A] &=\mathbb{P}[\frac{X}{\sigma} \le \Phi^{-1}(\underline{p_{adv}}) - \frac{||\delta||}{\sigma}]\\
        &\ge p\\
\end{align}

which is equivalent to
\begin{equation}
    ||\delta||\le \sigma [\Phi^{-1}(\underline{p_{adv}})-\Phi^{-1}(p)]
\end{equation}

This completes the proof of Corollary \ref{thm2}.

\end{proof}

\subsection{Proof of Theorem \ref{prop:bound}}
\label{apd:theoretical bound}

If the Condition Eq. (\ref{thm1 condition}) is satisfied, we have 
\begin{equation}
    \underline{p_{adv}}\geq p
\end{equation}

For any direction $w$, our goal is to find the $\delta$ in this direction with maximum $||\delta||_2$. When $||\delta||_2=0$, we have
\begin{equation}
    \Phi_+ = \Phi_-
\end{equation}

Thus, we have

\begin{equation}
        \Phi_+[\Phi^{-1}_-(\underline{p_{adv}})]=\underline{p_{adv}}\geq p
\end{equation}

Then, we prove that when $||\delta||_2$ increase, we will get $\Phi_+[\Phi^{-1}_-(\underline{p_{adv}})]$ decrease.

Since $\Phi_-$ is the CDF of the random variable $\frac{\varphi(\epsilon-\delta,\boldsymbol{\kappa})}{\varphi(\epsilon,\boldsymbol{\kappa})}$, and $\varphi(x)$ decreases when $|x|$ increases, when $||\delta||_2 \rightarrow \infty$,  we have $\frac{\varphi(\epsilon-\delta,\boldsymbol{\kappa})}{\varphi(\epsilon,\boldsymbol{\kappa})} \rightarrow 0$, and $\Phi_-^{-1}(\underline{p_{adv}}) \rightarrow 0$. 

Since $\Phi_+$ is the CDF of the random variable $\frac{\varphi(\epsilon,\boldsymbol{\kappa})}{\varphi(\epsilon+\delta,\boldsymbol{\kappa})}$, when $||\delta||_2 \rightarrow \infty$, $\frac{\varphi(\epsilon,\boldsymbol{\kappa})}{\varphi(\epsilon+\delta,\boldsymbol{\kappa})} \rightarrow \infty$, so $\Phi_+ \rightarrow 0$. 

Therefore, when $||\delta||_2 \rightarrow \infty$, we have
\begin{equation}
\Phi_+[\Phi^{-1}_-(\underline{p_{adv}})] \rightarrow 0 < p    
\end{equation}

When $||\delta||_2 \rightarrow 0$, we have 
\begin{equation}
\Phi_+[\Phi^{-1}_-(\underline{p_{adv}})]=\underline{p_{adv}}\geq p
\end{equation}

Since $\Phi_-$ and $\Phi_+$ is continuous function, between $0$ and $\infty$, there must be some $\delta$ such that $\Phi_+[\Phi^{-1}_-(p_{adv})]= p$.

Now we prove that the binary search algorithm can always find the $\delta$ solution, then we show how to bound the adversarial attack certification. We use Monte Carlo method to estimate the $\underline{p_{adv}}$ as well as the CDFs $\Phi_-$ and $\Phi_+$. To bound the empirical CDFs, we leverage Dvoretzky-Kiefer-Wolfowitz inequality \cite{dvoretzky1956asymptotic}.

\begin{Lem}{(\textbf{Dvoretzky–Kiefer–Wolfowitz inequality} (restate))}
Let $X_1, X_2, ..., X_n$ be real-valued independent and identically distributed random variables with cumulative distribution function $F(\cdot)$, where $n \in \mathbb{N}$.Let $F_n$ denote the associated empirical distribution function defined by
\begin{equation}
    F_n(x)=\frac{1}{n} \sum_{i=1}^{n} \mathbf{1}_{\{X_i<=x\}}, x \in \mathbb{R}
\end{equation}

The Dvoretzky–Kiefer–Wolfowitz inequality bounds the probability that the random function $F_n$ differs from $F$ by more than a given constant $\Delta \in \mathbb{R}^+$ :

\begin{equation}
    \mathbb{P}[\sup_{x\in\mathbb{R}}|F_n(x)-F(x)|>\Delta] \leq 2e^{-2n\Delta^2}
\end{equation}
\label{DKW lemma restate}
\end{Lem}

Let the Monte Carlo sampling number $N_m$. Each shifting is an independent certification, and there are a lower-bound estimation and two CDF estimations in each certification. Suppose the confidence of lower-bound estimation is $(1-\alpha)$, then the certification confidence should be at least $(1-\alpha)(1-2e^{-2N_m\Delta^2})^2$.

\section{Denoising with Diffusion Models}
\label{sec:DM}

The certifiable adversarial examples sampled from \randAE~ are noise-injected inputs that  still might be perceptible when the noise is large. We further leverage the recent innovation for image synthesis, i.e., diffusion model \cite{ho2020denoising}, to denoise the adversarial examples for better imperceptibility. The key idea is to consider the noise-perturbed adversarial examples as the middle sample in the forward process of the diffusion model \cite{carlini2022certified,zhang2023diffsmooth}. This is shown to improve the imperceptibility and the diversity of the adversarial examples.

Specifically, the closed-form sampling in the forward process at timestep $t$ in \cite{ho2020denoising} can be written as:
\begin{equation}
\label{eq:diffusion}
    x_t=\sqrt{\overline{\alpha}_t}x_0+\sqrt{1-\overline{\alpha}_t}\epsilon_0
\end{equation}
where $x_0=x$ is a clean image, $\overline{\alpha}_t$  is the parameter indicating the transformation of the image in the forward process, and $\epsilon_0$ is a noise drawn from the standard normal distribution $\mathcal{N}(0,1)$. The certifiable adversarial examples when adding noise $\epsilon_0$ can be expressed as:
\begin{equation}
    x_{adv}=x'+\delta+\epsilon_0
\end{equation}

We can then consider $x_{adv}$ as the sample $x_t$ in Eq. (\ref{eq:diffusion}) by transforming $x_{adv}$ to $\sqrt{\overline{\alpha}_t}x_{adv}$ and satisfying these conditions: (1) $x'+\delta=x_0$ and (2) $\sqrt{\overline{\alpha}_t}\sigma=\sqrt{1-\overline{\alpha}_t}$. Then, we have $\overline{\alpha}_t=\frac{1}{\sigma^2+1}$ to bridge diffusion model and the certifiable attack.

By finding the corresponding time step $t$ and  $\overline{\alpha_t}$, we can leverag the reverse process $\mathcal{R}(\cdot)$ of the diffusion model to denoise  $x_{adv}$:
\begin{equation}
    x_{adv}'=\mathcal{R}(\mathcal{R}(...\mathcal{R}(\sqrt{\overline{\alpha}_t}x_{adv})))
\end{equation}

Note that the reverse denoising process can be plugged into our attack framework by simply replacing $x_{adv}$ with $x_{adv}'$ in all processes. It will not affect the guarantee since the denoising process can be part of the classification model, i.e., constructing a new target model $f'(x_{adv})=f(\mathcal{R}(\mathcal{R}(...\mathcal{R}(\sqrt{\overline{\alpha_t}} x_{adv}))))$ given any $f$. 

\section{Additional Experiments}
\label{apd: additional exp}

\subsection{Additional Experimental Settings}
\label{apd:exp settings}
Table \ref{tab:parameters} summarizes all parameter settings. 
\begin{table*}[]
\centering
\caption{Summary of all parameter settings}
\label{tab:parameters}
\resizebox{\textwidth}{!}{%
\begin{tabular}{|l|cc|cc|cccccc|ccc|ccccc|}
\hline
\multicolumn{1}{|c|}{\multirow{2}{*}{Experiments}} &
  \multicolumn{2}{c|}{General} &
  \multicolumn{2}{c|}{Random Parallel Query} &
  \multicolumn{6}{c|}{Smoothed Self-supervised Localization} &
  \multicolumn{3}{c|}{Bin-search Localization} &
  \multicolumn{5}{c|}{Refinement} \\ \cline{2-19} 
\multicolumn{1}{|c|}{} &
  $p$ &
  $\sigma$ &
  $\alpha$ &
  $N_m$ &
  $\pi_{init}$ &
  $\gamma$ &
  $N_{max}$ &
  $n_{max}$ &
  $\eta$ &
  $N_s$ &
  $N_r$ &
  $N_b$ &
  $\Omega$ &
  $M$ &
  $\eta'$ &
  $e$ &
  $e_s$ &
  $N_h$ \\ \hline
Comparison with empirical attack under Blacklight detection &
  10\% &
  0.025 &
  0.001 &
  50 &
  3/255 &
  3/255 &
  85 &
  10 &
  3/255 &
  50 &
  85 &
  15 &
  0.1 &
  20 &
  0.05 &
  0.01 &
  0.0025 &
  72 \\
Comparison with empirical attack against RAND pre-processing defense &
  10\% &
  0.025 &
  0.001 &
  50 &
  3/255 &
  3/255 &
  85 &
  10 &
  3/255 &
  50 &
  85 &
  15 &
  0.1 &
  20 &
  0.05 &
  0.01 &
  0.0025 &
  72 \\
Comparison with empirical attack against RAND post-processing Defense &
  10\% &
  0.025 &
  0.001 &
  50 &
  3/255 &
  3/255 &
  85 &
  10 &
  3/255 &
  50 &
  85 &
  15 &
  0.1 &
  20 &
  0.05 &
  0.01 &
  0.0025 &
  72 \\
Comparison with empirical attack against TRADES adversarial training &
  10\% &
  0.025 &
  0.001 &
  50 &
  3/255 &
  3/255 &
  85 &
  10 &
  3/255 &
  50 &
  85 &
  15 &
  0.1 &
  20 &
  0.05 &
  0.01 &
  0.0025 &
  72 \\
Certifiable Attack against Feature Squeezing &
  90\% &
  0.25 &
  0.001 &
  1000 &
  3/255 &
  3/255 &
  85 &
  10 &
  3/255 &
  1000 &
  85 &
  15 &
  0.1 &
  20 &
  0.05 &
  0.01 &
  0.025 &
  72 \\
Certifiable Attack against Randomized Smoothing and Adaptive Denoiser &
  90\% &
  0.25 &
  0.001 &
  1000 &
  3/255 &
  3/255 &
  85 &
  10 &
  3/255 &
  1000 &
  85 &
  15 &
  0.1 &
  20 &
  0.05 &
  0.01 &
  0.025 &
  72 \\
Ablation study: Certifiable attack vs. different noise variance &
  90\% &
  0.10 - 0.50 &
  0.001 &
  500/1000 &
  3/255 &
  3/255 &
  85 &
  10 &
  3/255 &
  500/1000 &
  85 &
  15 &
  0.1 &
  20 &
  0.05 &
  0.01 &
  0.1 $\sigma$ &
  72 \\
Ablation study: Certifiable attack vs. different ASP Threshold p &
  50 - 95\% &
  0.25 &
  0.001 &
  500/1000 &
  3/255 &
  3/255 &
  85 &
  10 &
  3/255 &
  500/1000 &
  85 &
  15 &
  0.1 &
  20 &
  0.05 &
  0.01 &
  0.025 &
  72 \\
Ablation study: Certifiable attack vs. different Localization/Shifting &
  90\% &
  0.25 &
  0.001 &
  500/1000 &
  3/255 &
  3/255 &
  85 &
  10 &
  3/255 &
  500/1000 &
  85 &
  15 &
  0.1 &
  20 &
  0.05 &
  0.01 &
  0.025 &
  72 \\
Ablation study: Certifiable attack vs. different noise PDF &
  90\% &
  -- &
  0.001 &
  500/1000 &
  3/255 &
  3/255 &
  85 &
  10 &
  3/255 &
  500/1000 &
  85 &
  15 &
  0.1 &
  20 &
  0.05 &
  0.01 &
  -- &
  72 \\
Ablation study: Certifiable attack w/ and w/o Diffusion Denoise &
  90\% &
  0.25 &
  0.001 &
  500/1000 &
  3/255 &
  3/255 &
  85 &
  10 &
  3/255 &
  500/1000 &
  85 &
  15 &
  0.1 &
  20 &
  0.05 &
  0.01 &
  0.025 &
  72 \\ \hline
\end{tabular}%
}
\end{table*}

\subsection{More Results of Black-Box Attacks against SOTA Defenses}
\subsubsection{More Results on Attacking Blacklight} 
See Results in Table \ref{tab:blacklight_cifar10_vgg}-Table \ref{tab:blacklight_cifar100_wrn}.

\subsubsection{More Results on Attacking RAND-Pre}
See Results in Table \ref{tab:cifar10_RAND_vgg}-Table \ref{tab:cifar100_RAND_wrn}.

\subsubsection{More Results on Attacking RAND-Post}
See Results in Table \ref{tab:cifar10_post_RAND_vgg}-Table \ref{tab:cifar100_post_RAND_wrn}. 

\subsection{Attacking Other Empirical and Certified Defenses}

\subsubsection{Attacking Empirical Feature Squeezing Detection}

\begin{figure}[!h]
    \centering
    \includegraphics[width=0.6\linewidth]{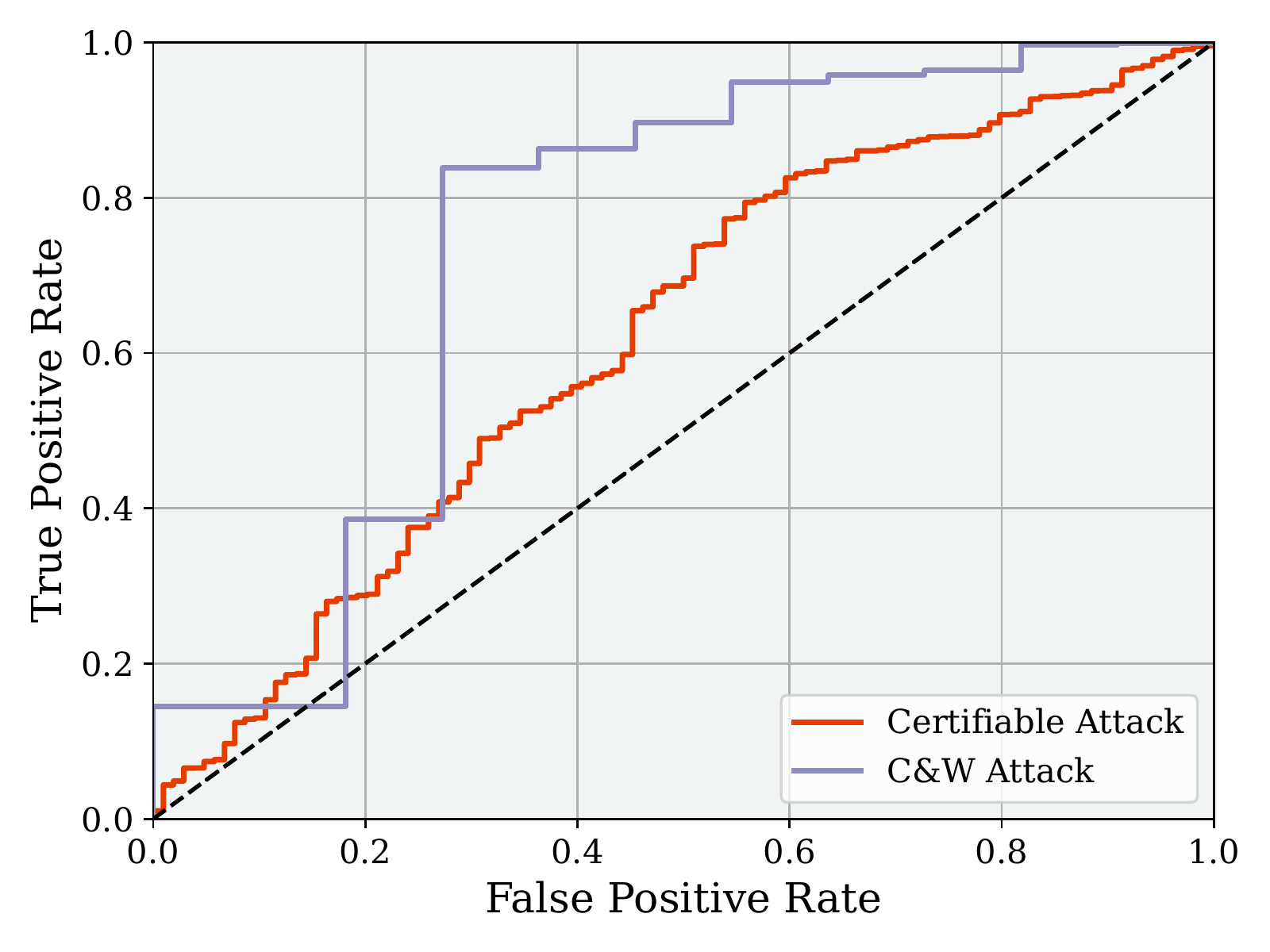}\vspace{-0.1in}
    \caption{ROC Curve of Detection Results by Feature Squeezing on CIFAR10. The ROC score is $0.38$ and $0.26$ for the C\&W attack and certifiable attack, respectively}
    \label{fig:ablation and ROC curve}
\end{figure}

We also evaluate the certifiable attack against adversarial detection. Specifically, we select the Feature Squeezing \cite{xu2017feature} method, which modifies the image and detects the adversarial examples according to the difference of model outputs. To position the performance of the detection, we compare the certifiable attack with the C\&W empirical attack \cite{DBLP:conf/sp/Carlini017}. In this experiment, Gaussian noise was adopted, and parameters are set as $\sigma=0.25$ and $p=90\%$. We draw the ROC curve in Figure \ref{fig:ablation and ROC curve}. 
As the results show, the certifiable attack is less detectable than the C\&W attack w.r.t. Feature Squeezing (with lower ROC scores), possibly because the prediction of empirical adversarial examples is less robust to image modification (empirical adversarial examples tend to be some special data points near the decision boundary). After the modification, it tends to output a different result. The outputs of certifiable adversarial examples are more consistent after the modification since their neighbors tend to be adversarial as well.

\subsubsection{Attacking Randomized Smoothing-based Certified Defense}

Randomized smoothing trains the classifier on inputs with Gaussian noises.
We evaluate the certifiable attack on the classifier trained with the same noise. Specifically, we use the Gaussian noise with $\sigma=0.12$ to $0.50$ in the classifier training and $\sigma=0.25$ in the certifiable attack. Table \ref{tab:RS defense} shows that the noise-trained classifier, especially when the model is trained with the same noise parameter as the adversary, can significantly degrade the performance of certifiable attacks. Noticeably, the smoothed training increases the perturbation sizes, especially the Mean Dist. $\ell_2$ significantly, and doubles the number of RPQ, which means the smoothing training can obstacle the certifiable attack to find a \randAE. It also reduces the certified accuracy of the certifiable attack significantly. However, the certifiable attack can still guarantee that $87.20\%$ of the test samples can be provably misclassified. Without performing the randomized smoothing certification against the certifiable attack, 
we can conclude that the certified accuracy of randomized smoothing with $\sigma=0.25$ will be at most $12.80\%$ since at least $87.20\%$ of the RandAEs are guaranteed to generate successful AEs with $90\%$ probability.

\begin{figure}[!h]
    \centering
    \includegraphics[width=0.6\linewidth]{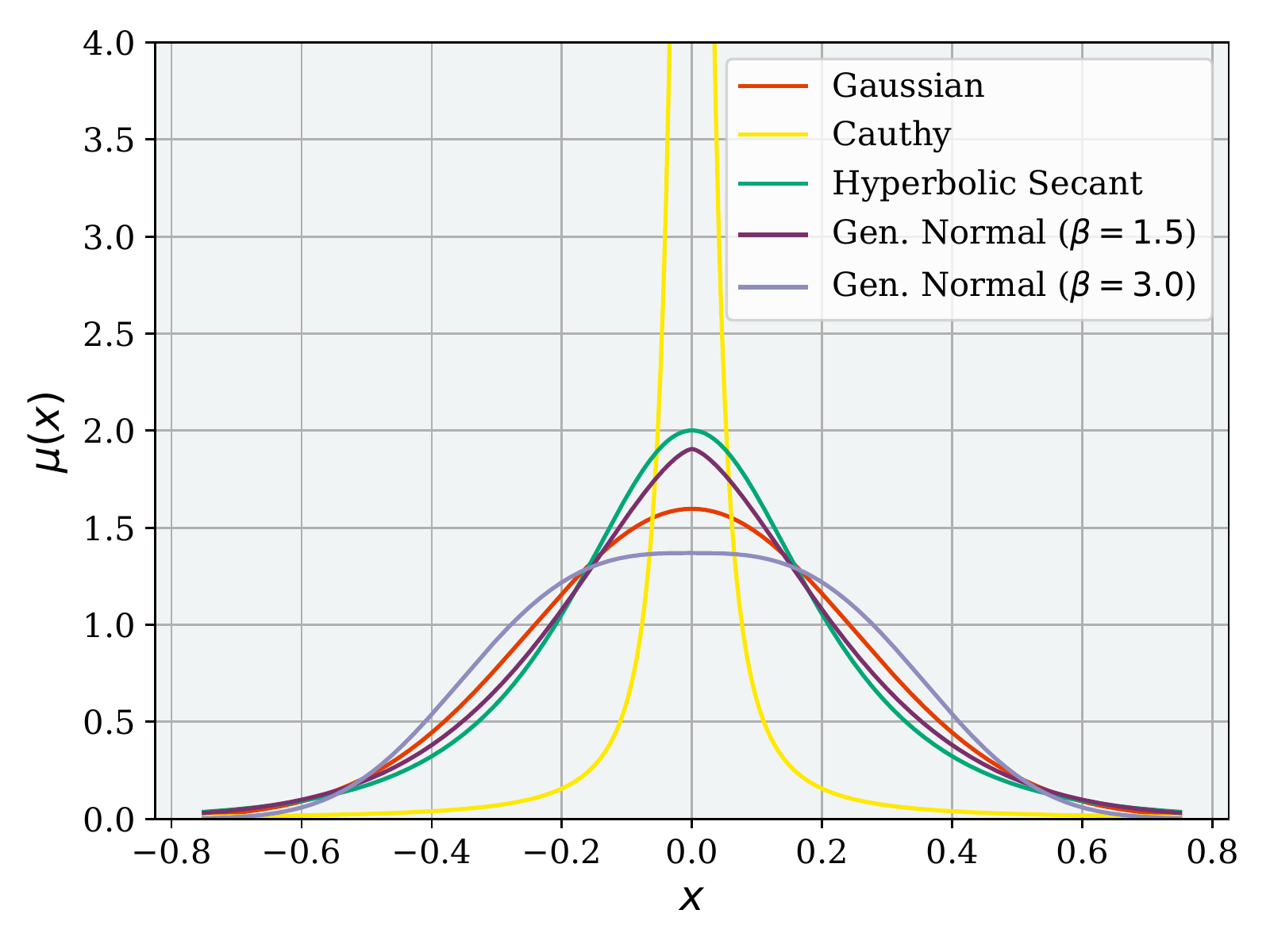}
    \caption{
    PDF of Different Noise Distributions ($\sigma=0.25$) 
    } 
    \label{fig:MC and pdfs}
\end{figure}

\begin{table}[!h]

    \centering
       \caption{RS-based defense against our attack on CIFAR10. $\sigma=0.25$, $p=90\%$}\vspace{-0.1in}
   \small
    \begin{tabular}{c|c |c | c c c}
    \hline
        Defense Para. & Dist. $\ell_2$ &Mean Dist. $\ell_2$ & \# RPQ & Cert. Acc. \\
    \hline
        none &12.95 &2.34  &14.35  & 91.21\% \\
    \hline
$\sigma_{rs}=0.12$   &13.93 &6.72  &23.93  & 88.40\% \\
     $\sigma_{rs}=0.25$  &16.07 &11.84 &33.93  & 87.20\% \\
$\sigma_{rs}=0.50$  &15.84 &10.12 &29.57  & 90.60\% \\
    \hline
        
    \end{tabular}
 
    \label{tab:RS defense}
\end{table}

\begin{table}[!t]
\caption{Performance of Certifiable Black-box Attack with Diffusion Denoise ($p=90\%$, $\sigma=0.25$). Diffusion Denoise can significantly reduce the perturbation size when the noise scale is very large.}\vspace{-0.1in}
\centering
\resizebox{0.47\textwidth}{!}{%
\begin{tabular}{c|c| c c c c}
\hline
     Dataset & denoise & Dist. $\ell_2$ &Mean Dist. $\ell_2$ & \# RPQ & Certified Acc.  \\
\hline
      \multirow{2}*{\scriptsize CIFAR10}

                                            &w/o   & 12.95 &2.34  &14.35 &91.21\% \\
                                            &w/   & 9.04 &10.38  &30.89 &91.30\% \\
\hline
         \multirow{2}*{\scriptsize  ImageNet} 

                                             & w/o   &24.32 &0.72 &5.60 &99.8\%  \\
                                             & w/    &8.48 &7.30 &12.85 &100.00\%  \\

\hline
\end{tabular}
}
\label{tab:denoise}
\end{table}

\subsection{Diffusion Model for Denoising}
\label{sec:exp diffusion}

We implement the diffusion model \cite{ho2020denoising} with the linear schedule. We train a diffusion model and an UNet with $3\times32\times32$ dimension for CIFAR10 and $3\times64\times64$ dimension for ImageNet and denoise the certified \randAE samples injected by Gaussian noise with $\sigma=0.25$. The experimental results are shown in Table \ref{tab:denoise}. Although the diffusion denoise increases the number of queries and Mean Dist. $\ell_2$, the perturbation of AE samples is significantly reduced due to the denoise. It is worth noting that the AEs generated by the diffusion model are unique due to the stochastic reverse process. This difference enables the certifiable attack to generate diverse AEs while ensuring the ASP guarantee.

\clearpage
\begin{table}[H]
\centering
\caption{Attack performance under Blacklight detection on VGG and CIFAR10 (Clean Accuracy: $90.3\%$)}
\label{tab:blacklight_cifar10_vgg}
\resizebox{0.47\textwidth}{!}{%
\begin{tabular}{lcccccccc}
\hline
Attack &
  \begin{tabular}[c]{@{}c@{}}Query\\ Type\end{tabular} &
  \begin{tabular}[c]{@{}c@{}}Pert.\\ Type\end{tabular} &
  \begin{tabular}[c]{@{}c@{}}Det.\\ Rate \%\end{tabular} &
  \begin{tabular}[c]{@{}c@{}}\# Q\\ to Det.\end{tabular} &
  \begin{tabular}[c]{@{}c@{}}Det.\\ Cov. \%\end{tabular} &
  \begin{tabular}[c]{@{}c@{}}Model\\ Acc.\end{tabular} &
  \# Q &
  \begin{tabular}[c]{@{}c@{}}Dist.\\ $\ell_2$\end{tabular} \\ \hline
Bandit                  & Score & \multicolumn{1}{c|}{$\ell_\infty$} & 100.0 & 1.0      & \multicolumn{1}{c|}{69.2}  & 0.0  & 59   & 4.77 \\
NES                     & Score & \multicolumn{1}{c|}{$\ell_\infty$} & 100.0 & 8.6      & \multicolumn{1}{c|}{21.2}  & 0.0  & 264  & 1.33 \\
Parsimonious            & Score & \multicolumn{1}{c|}{$\ell_\infty$} & 100.0 & 2.0      & \multicolumn{1}{c|}{96.7}  & 0.0  & 107  & 4.90 \\
Sign                    & Score & \multicolumn{1}{c|}{$\ell_\infty$} & 100.0 & 2.0      & \multicolumn{1}{c|}{92.4}  & 0.0  & 83   & 4.90 \\
Square                  & Score & \multicolumn{1}{c|}{$\ell_\infty$} & 100.0 & 2.0      & \multicolumn{1}{c|}{75.9}  & 0.0  & 25   & 4.90 \\
ZOSignSGD               & Score & \multicolumn{1}{c|}{$\ell_\infty$} & 100.0 & 2.0      & \multicolumn{1}{c|}{50.6}  & 0.0  & 267  & 1.29 \\
GeoDA                   & Label & \multicolumn{1}{c|}{$\ell_\infty$} & 100.0 & 1.0      & \multicolumn{1}{c|}{89.7}  & 0.2  & 377  & 3.08 \\
HSJ                     & Label & \multicolumn{1}{c|}{$\ell_\infty$} & 100.0 & 7.7      & \multicolumn{1}{c|}{92.2}  & 0.0  & 353  & 2.94 \\
Opt                     & Label & \multicolumn{1}{c|}{$\ell_\infty$} & 100.0 & 8.6      & \multicolumn{1}{c|}{83.4}  & 37.1 & 2571 & 1.04 \\
RayS                    & Label & \multicolumn{1}{c|}{$\ell_\infty$} & 100.0 & 5.8      & \multicolumn{1}{c|}{78.6}  & 0.0  & 226  & 4.82 \\
SignFlip                & Label & \multicolumn{1}{c|}{$\ell_\infty$} & 100.0 & 8.7      & \multicolumn{1}{c|}{56.0}  & 0.0  & 68   & 3.91 \\
SignOPT                 & Label & \multicolumn{1}{c|}{$\ell_\infty$} & 100.0 & 8.6      & \multicolumn{1}{c|}{89.1}  & 21.0 & 1112 & 1.15 \\
Bandit                  & Score & \multicolumn{1}{c|}{$\ell_2$}      & 100.0 & 1.0      & \multicolumn{1}{c|}{98.9}  & 0.0  & 121  & 2.62 \\
NES                     & Score & \multicolumn{1}{c|}{$\ell_2$}      & 100.0 & 8.5      & \multicolumn{1}{c|}{32.6}  & 0.0  & 823  & 0.53 \\
Simple                  & Score & \multicolumn{1}{c|}{$\ell_2$}      & 100.0 & 1.0      & \multicolumn{1}{c|}{99.8}  & 0.0  & 779  & 0.77 \\
Square                  & Score & \multicolumn{1}{c|}{$\ell_2$}      & 100.0 & 2.0      & \multicolumn{1}{c|}{78.7}  & 0.0  & 26   & 4.47 \\
ZOSignSGD               & Score & \multicolumn{1}{c|}{$\ell_2$}      & 100.0 & 2.0      & \multicolumn{1}{c|}{53.4}  & 0.3  & 1631 & 0.48 \\
Boundary                & Label & \multicolumn{1}{c|}{$\ell_2$}      & 100.0 & 7.6      & \multicolumn{1}{c|}{67.4}  & 21.7 & 315  & 2.70 \\
GeoDA                   & Label & \multicolumn{1}{c|}{$\ell_2$}      & 100.0 & 1.0      & \multicolumn{1}{c|}{89.7}  & 0.1  & 230  & 2.80 \\
HSJ                     & Label & \multicolumn{1}{c|}{$\ell_2$}      & 100.0 & 7.6      & \multicolumn{1}{c|}{90.9}  & 0.0  & 188  & 2.49 \\
Opt                     & Label & \multicolumn{1}{c|}{$\ell_2$}      & 100.0 & 8.6      & \multicolumn{1}{c|}{65.1}  & 32.6 & 675  & 2.39 \\
SignOPT                 & Label & \multicolumn{1}{c|}{$\ell_2$}      & 100.0 & 8.6      & \multicolumn{1}{c|}{77.7}  & 24.7 & 510  & 1.89 \\
PointWise               & Label & \multicolumn{1}{c|}{Opt.}          & 100.0 & 1.0      & \multicolumn{1}{c|}{99.5}  & 0.0  & 764  & 2.06 \\
SparseEvo               & Label & \multicolumn{1}{c|}{Opt.}          & 95.7  & 1.0      & \multicolumn{1}{c|}{100.0} & 0.0  & 9569 & 2.40 \\ \hline
\textbf{CA (sssp)}       & Label & \multicolumn{1}{c|}{Opt.}          & 0.0   & $\infty$ & \multicolumn{1}{c|}{0.0}   & 6.2  & 393  & 3.75 \\
\textbf{CA (bin search)} & Label & \multicolumn{1}{c|}{Opt.}          & 0.0   & $\infty$ & \multicolumn{1}{c|}{0.0}   & 0.0  & 473  & 5.20 \\ \hline
\end{tabular}%
}
\end{table}
\begin{table}[H]
\centering
\caption{Attack performance under Blacklight detection on ResNet and CIFAR10 (Clean Accuracy: $92.1\%$)}
\label{tab:blacklight_cifar10_resnet}
\resizebox{0.47\textwidth}{!}{%
\begin{tabular}{lcccccccc}
\hline
Attack &
  \begin{tabular}[c]{@{}c@{}}Query\\ Type\end{tabular} &
  \begin{tabular}[c]{@{}c@{}}Pert.\\ Type\end{tabular} &
  \begin{tabular}[c]{@{}c@{}}Det. \\ Rate \%\end{tabular} &
  \begin{tabular}[c]{@{}c@{}}\# Q\\ to Det.\end{tabular} &
  \begin{tabular}[c]{@{}c@{}}Det.\\ Cov. \%\end{tabular} &
  \begin{tabular}[c]{@{}c@{}}Model\\ Acc.\end{tabular} &
  \# Q &
  \begin{tabular}[c]{@{}c@{}}Dist.\\ $\ell_2$\end{tabular} \\ \hline
Bandit                  & Score & \multicolumn{1}{c|}{$\ell_\infty$} & 100.0 & 1.0      & \multicolumn{1}{c|}{67.9}  & 0.0  & 32   & 4.90 \\
NES                     & Score & \multicolumn{1}{c|}{$\ell_\infty$} & 100.0 & 8.5      & \multicolumn{1}{c|}{20.7}  & 0.0  & 256  & 1.36 \\
Parsimonious            & Score & \multicolumn{1}{c|}{$\ell_\infty$} & 100.0 & 2.0      & \multicolumn{1}{c|}{96.3}  & 0.0  & 83   & 5.01 \\
Sign                    & Score & \multicolumn{1}{c|}{$\ell_\infty$} & 100.0 & 2.0      & \multicolumn{1}{c|}{92.2}  & 0.0  & 94   & 4.99 \\
Square                  & Score & \multicolumn{1}{c|}{$\ell_\infty$} & 100.0 & 2.0      & \multicolumn{1}{c|}{71.6}  & 0.0  & 17   & 4.99 \\
ZOSignSGD               & Score & \multicolumn{1}{c|}{$\ell_\infty$} & 100.0 & 2.0      & \multicolumn{1}{c|}{50.6}  & 0.0  & 248  & 1.30 \\
GeoDA                   & Label & \multicolumn{1}{c|}{$\ell_\infty$} & 100.0 & 1.0      & \multicolumn{1}{c|}{89.6}  & 9.8  & 215  & 2.84 \\
HSJ                     & Label & \multicolumn{1}{c|}{$\ell_\infty$} & 100.0 & 361.1    & \multicolumn{1}{c|}{85.2}  & 0.5  & 683  & 3.17 \\
Opt                     & Label & \multicolumn{1}{c|}{$\ell_\infty$} & 89.0  & 9.0      & \multicolumn{1}{c|}{85.3}  & 50.5 & 2290 & 0.86 \\
RayS                    & Label & \multicolumn{1}{c|}{$\ell_\infty$} & 100.0 & 5.8      & \multicolumn{1}{c|}{78.9}  & 0.0  & 251  & 4.82 \\
SignFlip                & Label & \multicolumn{1}{c|}{$\ell_\infty$} & 99.5  & 246.8    & \multicolumn{1}{c|}{56.3}  & 0.5  & 389  & 4.22 \\
SignOPT                 & Label & \multicolumn{1}{c|}{$\ell_\infty$} & 92.5  & 8.4      & \multicolumn{1}{c|}{88.4}  & 31.0 & 1270 & 1.08 \\
Bandit                  & Score & \multicolumn{1}{c|}{$\ell_2$}      & 100.0 & 1.0      & \multicolumn{1}{c|}{98.7}  & 0.0  & 109  & 2.65 \\
NES                     & Score & \multicolumn{1}{c|}{$\ell_2$}      & 100.0 & 8.1      & \multicolumn{1}{c|}{32.5}  & 0.0  & 762  & 0.53 \\
Simple                  & Score & \multicolumn{1}{c|}{$\ell_2$}      & 100.0 & 1.0      & \multicolumn{1}{c|}{99.7}  & 0.0  & 703  & 0.78 \\
Square                  & Score & \multicolumn{1}{c|}{$\ell_2$}      & 100.0 & 2.0      & \multicolumn{1}{c|}{74.2}  & 0.0  & 21   & 4.55 \\
ZOSignSGD               & Score & \multicolumn{1}{c|}{$\ell_2$}      & 100.0 & 2.0      & \multicolumn{1}{c|}{53.3}  & 0.0  & 1340 & 0.45 \\
Boundary                & Label & \multicolumn{1}{c|}{$\ell_2$}      & 100.0 & 341.4    & \multicolumn{1}{c|}{71.5}  & 35.4 & 438  & 2.38 \\
GeoDA                   & Label & \multicolumn{1}{c|}{$\ell_2$}      & 100.0 & 1.0      & \multicolumn{1}{c|}{89.7}  & 8.3  & 274  & 2.65 \\
HSJ                     & Label & \multicolumn{1}{c|}{$\ell_2$}      & 100.0 & 358.4    & \multicolumn{1}{c|}{82.1}  & 0.4  & 497  & 2.75 \\
Opt                     & Label & \multicolumn{1}{c|}{$\ell_2$}      & 90.1  & 9.9      & \multicolumn{1}{c|}{64.4}  & 40.0 & 660  & 2.34 \\
SignOPT                 & Label & \multicolumn{1}{c|}{$\ell_2$}      & 88.1  & 8.5      & \multicolumn{1}{c|}{75.6}  & 35.3 & 414  & 1.67 \\
PointWise               & Label & \multicolumn{1}{c|}{Opt.}          & 91.3  & 1.0      & \multicolumn{1}{c|}{99.6}  & 8.0  & 888  & 1.92 \\
SparseEvo               & Label & \multicolumn{1}{c|}{Opt.}          & 87.9  & 1.0      & \multicolumn{1}{c|}{100.0} & 8.8  & 8796 & 2.57 \\ \hline
\textbf{CA (sssp)}       & Label & \multicolumn{1}{c|}{Opt.}          & 0.0   & $\infty$ & \multicolumn{1}{c|}{0.0}   & 8.3  & 437  & 3.95 \\
\textbf{CA (bin search)} & Label & \multicolumn{1}{c|}{Opt.}          & 0.0   & $\infty$ & \multicolumn{1}{c|}{0.0}   & 10.7 & 421  & 4.09 \\ \hline
\end{tabular}%
}
\end{table}
\begin{table}[H]
\centering
\caption{Attack performance under Blacklight detection on ResNeXt and CIFAR10 (Clean Accuracy: $94.9\%$)}
\label{tab:blacklight_cifar10_resnext}
\resizebox{0.47\textwidth}{!}{%
\begin{tabular}{lcccccccc}
\hline
Attack &
  \begin{tabular}[c]{@{}c@{}}Query\\ Type\end{tabular} &
  \begin{tabular}[c]{@{}c@{}}Pert.\\ Type\end{tabular} &
  \begin{tabular}[c]{@{}c@{}}Det.\\ Rate \%\end{tabular} &
  \begin{tabular}[c]{@{}c@{}}\# Q\\ to Det.\end{tabular} &
  \begin{tabular}[c]{@{}c@{}}Det.\\ Cov. \%\end{tabular} &
  \begin{tabular}[c]{@{}c@{}}Model\\ Acc.\end{tabular} &
  \# Q &
  \begin{tabular}[c]{@{}c@{}}Dist.\\ $\ell_2$\end{tabular} \\ \hline
Bandit                  & Score                & \multicolumn{1}{c|}{$\ell_\infty$} & 100.0 & 1.0      & \multicolumn{1}{c|}{67.3}  & 0.0  & 27   & 5.07 \\
NES                     & Score                & \multicolumn{1}{c|}{$\ell_\infty$} & 100.0 & 8.5      & \multicolumn{1}{c|}{20.4}  & 0.0  & 213  & 1.28 \\
Parsimonious            & Score                & \multicolumn{1}{c|}{$\ell_\infty$} & 100.0 & 2.0      & \multicolumn{1}{c|}{97.3}  & 0.0  & 104  & 5.16 \\
Sign                    & Score                & \multicolumn{1}{c|}{$\ell_\infty$} & 100.0 & 2.0      & \multicolumn{1}{c|}{93.6}  & 0.0  & 75   & 5.15 \\
Square                  & Score                & \multicolumn{1}{c|}{$\ell_\infty$} & 100.0 & 2.0      & \multicolumn{1}{c|}{69.7}  & 0.0  & 15   & 5.15 \\
ZOSignSGD               & Score                & \multicolumn{1}{c|}{$\ell_\infty$} & 100.0 & 2.0      & \multicolumn{1}{c|}{50.6}  & 0.0  & 222  & 1.27 \\
GeoDA                   & Label                & \multicolumn{1}{c|}{$\ell_\infty$} & 100.0 & 1.0      & \multicolumn{1}{c|}{89.9}  & 10.1 & 197  & 2.22 \\
HSJ                     & Label                & \multicolumn{1}{c|}{$\ell_\infty$} & 100.0 & 148.1    & \multicolumn{1}{c|}{85.9}  & 0.0  & 383  & 2.65 \\
Opt                     & Label                & \multicolumn{1}{c|}{$\ell_\infty$} & 88.4  & 8.5      & \multicolumn{1}{c|}{78.7}  & 43.3 & 1644 & 0.87 \\
RayS                    & Label                & \multicolumn{1}{c|}{$\ell_\infty$} & 100.0 & 5.4      & \multicolumn{1}{c|}{81.4}  & 0.0  & 391  & 4.77 \\
SignFlip                & Label                & \multicolumn{1}{c|}{$\ell_\infty$} & 100.0 & 165.8    & \multicolumn{1}{c|}{41.5}  & 0.0  & 205  & 3.43 \\
SignOPT                 & Label                & \multicolumn{1}{c|}{$\ell_\infty$} & 92.7  & 8.6      & \multicolumn{1}{c|}{92.5}  & 26.0 & 780  & 0.92 \\
Bandit                  & Score                & \multicolumn{1}{c|}{$\ell_2$}      & 100.0 & 1.0      & \multicolumn{1}{c|}{98.8}  & 0.0  & 110  & 2.81 \\
NES                     & Score                & \multicolumn{1}{c|}{$\ell_2$}      & 100.0 & 8.8      & \multicolumn{1}{c|}{32.4}  & 0.0  & 608  & 0.48 \\
Simple                  & Score                & \multicolumn{1}{c|}{$\ell_2$}      & 100.0 & 1.0      & \multicolumn{1}{c|}{99.7}  & 0.0  & 595  & 0.71 \\
Square                  & Score                & \multicolumn{1}{c|}{$\ell_2$}      & 100.0 & 2.0      & \multicolumn{1}{c|}{74.2}  & 0.0  & 20   & 4.69 \\
ZOSignSGD               & Score                & \multicolumn{1}{c|}{$\ell_2$}      & 100.0 & 2.0      & \multicolumn{1}{c|}{53.4}  & 0.6  & 1376 & 0.46 \\
Boundary                & Label                & \multicolumn{1}{c|}{$\ell_2$}      & 100.0 & 161.4    & \multicolumn{1}{c|}{56.2}  & 14.8 & 186  & 2.71 \\
GeoDA                   & Label                & \multicolumn{1}{c|}{$\ell_2$}      & 100.0 & 1.0      & \multicolumn{1}{c|}{90.3}  & 10.2 & 165  & 2.13 \\
HSJ                     & Label                & \multicolumn{1}{c|}{$\ell_2$}      & 100.0 & 138.3    & \multicolumn{1}{c|}{83.9}  & 0.0  & 287  & 2.16 \\
Opt                     & Label                & \multicolumn{1}{c|}{$\ell_2$}      & 88.4  & 8.6      & \multicolumn{1}{c|}{60.8}  & 37.2 & 480  & 1.65 \\
SignOPT                 & Label                & \multicolumn{1}{c|}{$\ell_2$}      & 88.4  & 8.6      & \multicolumn{1}{c|}{87.5}  & 32.7 & 387  & 1.17 \\
PointWise               & Label                & \multicolumn{1}{c|}{Opt.}          & 97.6  & 1.0      & \multicolumn{1}{c|}{97.6}  & 2.3  & 1084 & 2.01 \\
SparseEvo               & Label                & \multicolumn{1}{c|}{Opt.}          & 95.0  & 1.0      & \multicolumn{1}{c|}{100.0} & 2.6  & 9506 & 3.11 \\ \hline
\textbf{CA (sssp)}       & Label                & \multicolumn{1}{c|}{Opt.}          & 0.0   & $\infty$ & \multicolumn{1}{c|}{0.0}   & 8.3  & 437  & 3.95 \\
\textbf{CA (bin search)} & Label                & \multicolumn{1}{c|}{Opt.}          & 0.0   & $\infty$ & \multicolumn{1}{c|}{0.0}   & 10.7 & 421  & 4.09 \\ \hline
                        & \multicolumn{1}{l}{} & \multicolumn{1}{l}{}               &       &          &                            &      &      &     
\end{tabular}%
}
\end{table}
\begin{table}[H]
\centering
\caption{Attack performance under Blacklight detection on WRN and CIFAR10 (Clean Accuracy: $96.1\%$)}
\label{tab:blacklight_cifar10_wrn}
\resizebox{0.47\textwidth}{!}{%
\begin{tabular}{lcccccccc}
\hline
Attack &
  \begin{tabular}[c]{@{}c@{}}Query\\ Type\end{tabular} &
  \begin{tabular}[c]{@{}c@{}}Pert.\\ Type\end{tabular} &
  \begin{tabular}[c]{@{}c@{}}Det.\\ Rate \%\end{tabular} &
  \begin{tabular}[c]{@{}c@{}}\# Q\\ to Det.\end{tabular} &
  \begin{tabular}[c]{@{}c@{}}Det.\\ Cov. \%\end{tabular} &
  \begin{tabular}[c]{@{}c@{}}Model\\ Acc.\end{tabular} &
  \# Q &
  \begin{tabular}[c]{@{}c@{}}Dist.\\ $\ell_2$\end{tabular} \\ \hline
Bandit                  & Score                & \multicolumn{1}{c|}{$\ell_\infty$} & 100.0 & 1.0      & \multicolumn{1}{c|}{67.4}  & 0.0  & 47   & 5.09 \\
NES                     & Score                & \multicolumn{1}{c|}{$\ell_\infty$} & 100.0 & 8.7      & \multicolumn{1}{c|}{20.7}  & 0.0  & 295  & 1.44 \\
Parsimonious            & Score                & \multicolumn{1}{c|}{$\ell_\infty$} & 100.0 & 2.0      & \multicolumn{1}{c|}{97.6}  & 0.2  & 140  & 5.22 \\
Sign                    & Score                & \multicolumn{1}{c|}{$\ell_\infty$} & 100.0 & 2.0      & \multicolumn{1}{c|}{94.6}  & 0.0  & 130  & 5.21 \\
Square                  & Score                & \multicolumn{1}{c|}{$\ell_\infty$} & 100.0 & 2.0      & \multicolumn{1}{c|}{74.5}  & 0.0  & 21   & 5.21 \\
ZOSignSGD               & Score                & \multicolumn{1}{c|}{$\ell_\infty$} & 100.0 & 2.0      & \multicolumn{1}{c|}{50.6}  & 0.2  & 300  & 1.42 \\
GeoDA                   & Label                & \multicolumn{1}{c|}{$\ell_\infty$} & 100.0 & 1.0      & \multicolumn{1}{c|}{89.6}  & 0.1  & 323  & 2.83 \\
HSJ                     & Label                & \multicolumn{1}{c|}{$\ell_\infty$} & 100.0 & 7.3      & \multicolumn{1}{c|}{91.9}  & 0.0  & 280  & 2.69 \\
Opt                     & Label                & \multicolumn{1}{c|}{$\ell_\infty$} & 97.0  & 11.3     & \multicolumn{1}{c|}{81.1}  & 35.0 & 2179 & 1.15 \\
RayS                    & Label                & \multicolumn{1}{c|}{$\ell_\infty$} & 100.0 & 4.9      & \multicolumn{1}{c|}{81.5}  & 0.0  & 309  & 4.87 \\
SignFlip                & Label                & \multicolumn{1}{c|}{$\ell_\infty$} & 100.0 & 8.3      & \multicolumn{1}{c|}{53.0}  & 0.0  & 72   & 3.75 \\
SignOPT                 & Label                & \multicolumn{1}{c|}{$\ell_\infty$} & 88.4  & 8.5      & \multicolumn{1}{c|}{89.4}  & 28.0 & 843  & 0.97 \\
Bandit                  & Score                & \multicolumn{1}{c|}{$\ell_2$}      & 100.0 & 1.0      & \multicolumn{1}{c|}{98.8}  & 0.0  & 136  & 2.71 \\
NES                     & Score                & \multicolumn{1}{c|}{$\ell_2$}      & 100.0 & 8.2      & \multicolumn{1}{c|}{32.5}  & 0.0  & 863  & 0.57 \\
Simple                  & Score                & \multicolumn{1}{c|}{$\ell_2$}      & 100.0 & 1.0      & \multicolumn{1}{c|}{99.8}  & 0.0  & 813  & 0.83 \\
Square                  & Score                & \multicolumn{1}{c|}{$\ell_2$}      & 100.0 & 2.0      & \multicolumn{1}{c|}{73.2}  & 0.0  & 20   & 4.75 \\
ZOSignSGD               & Score                & \multicolumn{1}{c|}{$\ell_2$}      & 100.0 & 2.0      & \multicolumn{1}{c|}{53.3}  & 2.3  & 1803 & 0.54 \\
Boundary                & Label                & \multicolumn{1}{c|}{$\ell_2$}      & 100.0 & 7.3      & \multicolumn{1}{c|}{62.2}  & 16.7 & 376  & 2.85 \\
GeoDA                   & Label                & \multicolumn{1}{c|}{$\ell_2$}      & 100.0 & 1.0      & \multicolumn{1}{c|}{89.6}  & 0.0  & 186  & 2.63 \\
HSJ                     & Label                & \multicolumn{1}{c|}{$\ell_2$}      & 100.0 & 7.4      & \multicolumn{1}{c|}{91.0}  & 0.0  & 185  & 2.27 \\
Opt                     & Label                & \multicolumn{1}{c|}{$\ell_2$}      & 97.3  & 10.7     & \multicolumn{1}{c|}{66.3}  & 35.1 & 678  & 2.14 \\
SignOPT                 & Label                & \multicolumn{1}{c|}{$\ell_2$}      & 91.7  & 8.5      & \multicolumn{1}{c|}{80.3}  & 33.0 & 470  & 1.57 \\
PointWise               & Label                & \multicolumn{1}{c|}{Opt.}          & 99.9  & 1.0      & \multicolumn{1}{c|}{99.5}  & 0.1  & 800  & 1.96 \\
SparseEvo               & Label                & \multicolumn{1}{c|}{Opt.}          & 97.7  & 1.0      & \multicolumn{1}{c|}{100.0} & 0.2  & 9772 & 2.83 \\ \hline
\textbf{CA (sssp)}       & Label                & \multicolumn{1}{c|}{Opt.}          & 0.0   & $\infty$ & \multicolumn{1}{c|}{0.0}   & 6.8  & 417  & 3.7  \\
\textbf{CA (bin search)} & Label                & \multicolumn{1}{c|}{Opt.}          & 0.0   & $\infty$ & \multicolumn{1}{c|}{0.0}   & 0.0  & 461  & 5.05 \\ \hline
                        & \multicolumn{1}{l}{} & \multicolumn{1}{l}{}               &       &          &                            &      &      &     
\end{tabular}%
}
\end{table}
\clearpage
\begin{table}[H]
\centering
\caption{Attack performance under Blacklight detection on VGG and CIFAR100 (Clean Accuracy: $68.6\%$)}
\label{tab:blacklight_cifar100_vgg}
\resizebox{0.47\textwidth}{!}{%
\begin{tabular}{lcccccccc}
\hline
Attack &
  \begin{tabular}[c]{@{}c@{}}Query\\ Type\end{tabular} &
  \begin{tabular}[c]{@{}c@{}}Pert.\\ Type\end{tabular} &
  \begin{tabular}[c]{@{}c@{}}Det.\\ Rate \%\end{tabular} &
  \begin{tabular}[c]{@{}c@{}}\# Q\\ to Det.\end{tabular} &
  \begin{tabular}[c]{@{}c@{}}Det.\\ Cov. \%\end{tabular} &
  \begin{tabular}[c]{@{}c@{}}Model\\ Acc.\end{tabular} &
  \# Q &
  \begin{tabular}[c]{@{}c@{}}Dist.\\ $\ell_2$\end{tabular} \\ \hline
Bandit                  & Score                & \multicolumn{1}{c|}{$\ell_\infty$} & 100.0 & 1.0      & \multicolumn{1}{c|}{61.5}  & 0.0  & 22   & 3.66 \\
NES                     & Score                & \multicolumn{1}{c|}{$\ell_\infty$} & 100.0 & 8.6      & \multicolumn{1}{c|}{21.6}  & 0.0  & 178  & 0.80 \\
Parsimonious            & Score                & \multicolumn{1}{c|}{$\ell_\infty$} & 100.0 & 2.0      & \multicolumn{1}{c|}{94.0}  & 0.0  & 73   & 3.70 \\
Sign                    & Score                & \multicolumn{1}{c|}{$\ell_\infty$} & 100.0 & 2.0      & \multicolumn{1}{c|}{84.6}  & 0.0  & 44   & 3.68 \\
Square                  & Score                & \multicolumn{1}{c|}{$\ell_\infty$} & 100.0 & 2.0      & \multicolumn{1}{c|}{64.7}  & 0.0  & 9    & 3.68 \\
ZOSignSGD               & Score                & \multicolumn{1}{c|}{$\ell_\infty$} & 100.0 & 2.0      & \multicolumn{1}{c|}{50.7}  & 0.1  & 184  & 0.79 \\
GeoDA                   & Label                & \multicolumn{1}{c|}{$\ell_\infty$} & 100.0 & 1.0      & \multicolumn{1}{c|}{89.8}  & 0.0  & 154  & 2.07 \\
HSJ                     & Label                & \multicolumn{1}{c|}{$\ell_\infty$} & 100.0 & 7.2      & \multicolumn{1}{c|}{91.9}  & 0.0  & 232  & 2.16 \\
Opt                     & Label                & \multicolumn{1}{c|}{$\ell_\infty$} & 100.0 & 8.6      & \multicolumn{1}{c|}{74.7}  & 20.8 & 1463 & 0.97 \\
RayS                    & Label                & \multicolumn{1}{c|}{$\ell_\infty$} & 100.0 & 6.6      & \multicolumn{1}{c|}{71.7}  & 0.0  & 120  & 4.40 \\
SignFlip                & Label                & \multicolumn{1}{c|}{$\ell_\infty$} & 100.0 & 8.7      & \multicolumn{1}{c|}{44.9}  & 0.0  & 30   & 2.88 \\
SignOPT                 & Label                & \multicolumn{1}{c|}{$\ell_\infty$} & 100.0 & 8.5      & \multicolumn{1}{c|}{92.3}  & 23.0 & 699  & 0.84 \\
Bandit                  & Score                & \multicolumn{1}{c|}{$\ell_2$}      & 100.0 & 1.0      & \multicolumn{1}{c|}{97.9}  & 0.0  & 73   & 1.35 \\
NES                     & Score                & \multicolumn{1}{c|}{$\ell_2$}      & 100.0 & 8.7      & \multicolumn{1}{c|}{32.4}  & 0.0  & 553  & 0.31 \\
Simple                  & Score                & \multicolumn{1}{c|}{$\ell_2$}      & 100.0 & 1.0      & \multicolumn{1}{c|}{99.4}  & 0.0  & 507  & 0.44 \\
Square                  & Score                & \multicolumn{1}{c|}{$\ell_2$}      & 100.0 & 2.0      & \multicolumn{1}{c|}{64.0}  & 0.0  & 6    & 3.35 \\
ZOSignSGD               & Score                & \multicolumn{1}{c|}{$\ell_2$}      & 100.0 & 2.0      & \multicolumn{1}{c|}{53.6}  & 0.4  & 1137 & 0.29 \\
Boundary                & Label                & \multicolumn{1}{c|}{$\ell_2$}      & 100.0 & 7.3      & \multicolumn{1}{c|}{54.0}  & 6.3  & 100  & 2.49 \\
GeoDA                   & Label                & \multicolumn{1}{c|}{$\ell_2$}      & 100.0 & 1.0      & \multicolumn{1}{c|}{90.2}  & 0.0  & 125  & 1.95 \\
HSJ                     & Label                & \multicolumn{1}{c|}{$\ell_2$}      & 100.0 & 7.3      & \multicolumn{1}{c|}{91.4}  & 0.0  & 147  & 1.74 \\
Opt                     & Label                & \multicolumn{1}{c|}{$\ell_2$}      & 100.0 & 8.6      & \multicolumn{1}{c|}{63.3}  & 19.8 & 563  & 1.51 \\
SignOPT                 & Label                & \multicolumn{1}{c|}{$\ell_2$}      & 99.4  & 8.5      & \multicolumn{1}{c|}{88.7}  & 14.4 & 446  & 1.13 \\
PointWise               & Label                & \multicolumn{1}{c|}{Opt.}          & 100.0 & 1.0      & \multicolumn{1}{c|}{98.9}  & 0.0  & 300  & 1.42 \\
SparseEvo               & Label                & \multicolumn{1}{c|}{Opt.}          & 86.5  & 1.0      & \multicolumn{1}{c|}{100.0} & 0.0  & 8647 & 1.88 \\ \hline
\textbf{CA (sssp)}       & Label                & \multicolumn{1}{c|}{Opt.}          & 0.0   & $\infty$ & \multicolumn{1}{c|}{0.0}   & 1.7  & 187  & 2.34 \\
\textbf{CA (bin search)} & Label                & \multicolumn{1}{c|}{Opt.}          & 0.0   & $\infty$ & \multicolumn{1}{c|}{0.0}   & 0.0  & 457  & 2.70 \\ \hline
                        & \multicolumn{1}{l}{} & \multicolumn{1}{l}{}               &       &          &                            &      &      &     
\end{tabular}%
}
\end{table}
\begin{table}[H]
\centering
\caption{Attack performance under Blacklight detection on ResNet and CIFAR100 (Clean Accuracy: $66.8\%$)}
\label{tab:blacklight_cifar100_resnet}
\resizebox{0.47\textwidth}{!}{%
\begin{tabular}{lcccccccc}
\hline
Attack &
  \begin{tabular}[c]{@{}c@{}}Query\\ Type\end{tabular} &
  \begin{tabular}[c]{@{}c@{}}Pert.\\ Type\end{tabular} &
  \begin{tabular}[c]{@{}c@{}}Det.\\ Rate \%\end{tabular} &
  \begin{tabular}[c]{@{}c@{}}\# Q\\ to Det.\end{tabular} &
  \begin{tabular}[c]{@{}c@{}}Det.\\ Cov. \%\end{tabular} &
  \begin{tabular}[c]{@{}c@{}}Model\\ Acc.\end{tabular} &
  \# Q &
  \begin{tabular}[c]{@{}c@{}}Dist.\\ $\ell_2$\end{tabular} \\ \hline
Bandit                  & Score                & \multicolumn{1}{c|}{$\ell_\infty$} & 100.0 & 1.0      & \multicolumn{1}{c|}{62.5}  & 0.0  & 10   & 3.56 \\
NES                     & Score                & \multicolumn{1}{c|}{$\ell_\infty$} & 100.0 & 8.5      & \multicolumn{1}{c|}{21.3}  & 0.0  & 141  & 0.72 \\
Parsimonious            & Score                & \multicolumn{1}{c|}{$\ell_\infty$} & 100.0 & 2.0      & \multicolumn{1}{c|}{93.5}  & 0.0  & 49   & 3.61 \\
Sign                    & Score                & \multicolumn{1}{c|}{$\ell_\infty$} & 100.0 & 2.0      & \multicolumn{1}{c|}{82.4}  & 0.0  & 31   & 3.59 \\
Square                  & Score                & \multicolumn{1}{c|}{$\ell_\infty$} & 100.0 & 2.0      & \multicolumn{1}{c|}{66.2}  & 0.0  & 8    & 3.58 \\
ZOSignSGD               & Score                & \multicolumn{1}{c|}{$\ell_\infty$} & 100.0 & 2.0      & \multicolumn{1}{c|}{50.7}  & 0.0  & 148  & 0.70 \\
GeoDA                   & Label                & \multicolumn{1}{c|}{$\ell_\infty$} & 100.0 & 1.0      & \multicolumn{1}{c|}{89.5}  & 0.0  & 152  & 2.21 \\
HSJ                     & Label                & \multicolumn{1}{c|}{$\ell_\infty$} & 100.0 & 7.3      & \multicolumn{1}{c|}{91.8}  & 0.0  & 214  & 2.23 \\
Opt                     & Label                & \multicolumn{1}{c|}{$\ell_\infty$} & 100.0 & 8.5      & \multicolumn{1}{c|}{75.1}  & 25.8 & 1442 & 0.96 \\
RayS                    & Label                & \multicolumn{1}{c|}{$\ell_\infty$} & 100.0 & 6.6      & \multicolumn{1}{c|}{70.3}  & 0.0  & 109  & 4.02 \\
SignFlip                & Label                & \multicolumn{1}{c|}{$\ell_\infty$} & 100.0 & 8.7      & \multicolumn{1}{c|}{49.0}  & 0.0  & 39   & 3.28 \\
SignOPT                 & Label                & \multicolumn{1}{c|}{$\ell_\infty$} & 100.0 & 8.5      & \multicolumn{1}{c|}{89.5}  & 17.0 & 749  & 0.88 \\
Bandit                  & Score                & \multicolumn{1}{c|}{$\ell_2$}      & 100.0 & 1.0      & \multicolumn{1}{c|}{97.8}  & 0.0  & 63   & 1.25 \\
NES                     & Score                & \multicolumn{1}{c|}{$\ell_2$}      & 100.0 & 8.7      & \multicolumn{1}{c|}{32.5}  & 0.0  & 404  & 0.27 \\
Simple                  & Score                & \multicolumn{1}{c|}{$\ell_2$}      & 100.0 & 1.0      & \multicolumn{1}{c|}{99.5}  & 0.0  & 371  & 0.39 \\
Square                  & Score                & \multicolumn{1}{c|}{$\ell_2$}      & 100.0 & 2.0      & \multicolumn{1}{c|}{64.4}  & 0.0  & 6    & 3.26 \\
ZOSignSGD               & Score                & \multicolumn{1}{c|}{$\ell_2$}      & 100.0 & 2.0      & \multicolumn{1}{c|}{53.5}  & 0.1  & 747  & 0.23 \\
Boundary                & Label                & \multicolumn{1}{c|}{$\ell_2$}      & 100.0 & 7.3      & \multicolumn{1}{c|}{56.4}  & 9.1  & 105  & 2.67 \\
GeoDA                   & Label                & \multicolumn{1}{c|}{$\ell_2$}      & 100.0 & 1.0      & \multicolumn{1}{c|}{89.9}  & 0.0  & 130  & 2.02 \\
HSJ                     & Label                & \multicolumn{1}{c|}{$\ell_2$}      & 100.0 & 7.3      & \multicolumn{1}{c|}{91.2}  & 0.0  & 151  & 1.81 \\
Opt                     & Label                & \multicolumn{1}{c|}{$\ell_2$}      & 100.0 & 8.5      & \multicolumn{1}{c|}{63.2}  & 21.6 & 567  & 1.58 \\
SignOPT                 & Label                & \multicolumn{1}{c|}{$\ell_2$}      & 100.0 & 8.5      & \multicolumn{1}{c|}{85.8}  & 17.3 & 472  & 1.20 \\
PointWise               & Label                & \multicolumn{1}{c|}{Opt.}          & 100.0 & 1.0      & \multicolumn{1}{c|}{99.2}  & 0.0  & 468  & 1.52 \\
SparseEvo               & Label                & \multicolumn{1}{c|}{Opt.}          & 90.9  & 1.0      & \multicolumn{1}{c|}{100.0} & 0.0  & 9088 & 2.18 \\ \hline
\textbf{CA (sssp)}       & Label                & \multicolumn{1}{c|}{Opt.}          & 0.0   & $\infty$ & \multicolumn{1}{c|}{0.0}   & 1.4  & 191  & 2.39 \\
\textbf{CA (bin search)} & Label                & \multicolumn{1}{c|}{Opt.}          & 0.0   & $\infty$ & \multicolumn{1}{c|}{0.0}   & 0.0  & 458  & 3.05 \\ \hline
                        & \multicolumn{1}{l}{} & \multicolumn{1}{l}{}               &       &          &                            &      &      &     
\end{tabular}%
}
\end{table}
\begin{table}[H]
\centering
\caption{Attack performance under Blacklight detection on ResNeXt and CIFAR100 (Clean Accuracy: $80.0\%$)}
\label{tab:blacklight_cifar100_resnext}
\resizebox{0.47\textwidth}{!}{%
\begin{tabular}{lcccccccc}
\hline
Attack &
  \begin{tabular}[c]{@{}c@{}}Query\\ Type\end{tabular} &
  \begin{tabular}[c]{@{}c@{}}Pert.\\ Type\end{tabular} &
  \begin{tabular}[c]{@{}c@{}}Det.\\ Rate \%\end{tabular} &
  \begin{tabular}[c]{@{}c@{}}\# Q\\ to Det.\end{tabular} &
  \begin{tabular}[c]{@{}c@{}}Det.\\ Cov. \%\end{tabular} &
  \begin{tabular}[c]{@{}c@{}}Model\\ Acc.\end{tabular} &
  \# Q &
  \begin{tabular}[c]{@{}c@{}}Dist.\\ $\ell_2$\end{tabular} \\ \hline
Bandit                  & Score                & \multicolumn{1}{c|}{$\ell_\infty$} & 100.0 & 1.0      & \multicolumn{1}{c|}{62.3} & 0.0  & 13   & 4.28 \\
NES                     & Score                & \multicolumn{1}{c|}{$\ell_\infty$} & 100.0 & 8.6      & \multicolumn{1}{c|}{21.6} & 0.0  & 145  & 0.87 \\
Parsimonious            & Score                & \multicolumn{1}{c|}{$\ell_\infty$} & 100.0 & 2.0      & \multicolumn{1}{c|}{95.5} & 0.0  & 73   & 4.32 \\
Sign                    & Score                & \multicolumn{1}{c|}{$\ell_\infty$} & 100.0 & 2.0      & \multicolumn{1}{c|}{87.7} & 0.0  & 47   & 4.30 \\
Square                  & Score                & \multicolumn{1}{c|}{$\ell_\infty$} & 100.0 & 2.0      & \multicolumn{1}{c|}{63.7} & 0.0  & 7    & 4.30 \\
ZOSignSGD               & Score                & \multicolumn{1}{c|}{$\ell_\infty$} & 100.0 & 2.0      & \multicolumn{1}{c|}{50.7} & 0.0  & 158  & 0.89 \\
GeoDA                   & Label                & \multicolumn{1}{c|}{$\ell_\infty$} & 100.0 & 1.0      & \multicolumn{1}{c|}{90.1} & 0.1  & 136  & 1.91 \\
HSJ                     & Label                & \multicolumn{1}{c|}{$\ell_\infty$} & 100.0 & 7.3      & \multicolumn{1}{c|}{92.0} & 0.0  & 226  & 1.96 \\
Opt                     & Label                & \multicolumn{1}{c|}{$\ell_\infty$} & 99.0  & 8.6      & \multicolumn{1}{c|}{73.5} & 32.1 & 1158 & 0.84 \\
RayS                    & Label                & \multicolumn{1}{c|}{$\ell_\infty$} & 100.0 & 6.5      & \multicolumn{1}{c|}{75.2} & 0.0  & 175  & 4.44 \\
SignFlip                & Label                & \multicolumn{1}{c|}{$\ell_\infty$} & 100.0 & 8.8      & \multicolumn{1}{c|}{42.2} & 0.0  & 27   & 2.52 \\
SignOPT                 & Label                & \multicolumn{1}{c|}{$\ell_\infty$} & 98.8  & 8.5      & \multicolumn{1}{c|}{92.7} & 14.0 & 616  & 0.79 \\
Bandit                  & Score                & \multicolumn{1}{c|}{$\ell_2$}      & 100.0 & 1.0      & \multicolumn{1}{c|}{98.0} & 0.0  & 69   & 1.62 \\
NES                     & Score                & \multicolumn{1}{c|}{$\ell_2$}      & 100.0 & 8.7      & \multicolumn{1}{c|}{31.7} & 0.0  & 420  & 0.33 \\
Simple                  & Score                & \multicolumn{1}{c|}{$\ell_2$}      & 100.0 & 1.0      & \multicolumn{1}{c|}{99.5} & 0.0  & 409  & 0.48 \\
Square                  & Score                & \multicolumn{1}{c|}{$\ell_2$}      & 100.0 & 2.0      & \multicolumn{1}{c|}{65.8} & 0.0  & 8    & 3.92 \\
ZOSignSGD               & Score                & \multicolumn{1}{c|}{$\ell_2$}      & 100.0 & 2.0      & \multicolumn{1}{c|}{53.5} & 0.0  & 1029 & 0.33 \\
Boundary                & Label                & \multicolumn{1}{c|}{$\ell_2$}      & 100.0 & 7.3      & \multicolumn{1}{c|}{51.1} & 5.7  & 53   & 2.26 \\
GeoDA                   & Label                & \multicolumn{1}{c|}{$\ell_2$}      & 100.0 & 1.0      & \multicolumn{1}{c|}{90.6} & 0.1  & 120  & 1.84 \\
HSJ                     & Label                & \multicolumn{1}{c|}{$\ell_2$}      & 100.0 & 7.3      & \multicolumn{1}{c|}{91.7} & 0.0  & 146  & 1.53 \\
Opt                     & Label                & \multicolumn{1}{c|}{$\ell_2$}      & 99.0  & 8.6      & \multicolumn{1}{c|}{62.1} & 21.7 & 537  & 1.46 \\
SignOPT                 & Label                & \multicolumn{1}{c|}{$\ell_2$}      & 98.9  & 8.6      & \multicolumn{1}{c|}{89.5} & 16.5 & 432  & 1.07 \\
PointWise               & Label                & \multicolumn{1}{c|}{Opt.}          & 100.0 & 1.0      & \multicolumn{1}{c|}{99.4} & 0.0  & 766  & 1.80 \\
SparseEvo               & Label                & \multicolumn{1}{c|}{Opt.}          & 93.8  & 1.0      & \multicolumn{1}{c|}{100}  & 0.0  & 9376 & 2.53 \\ \hline
\textbf{CA (sssp)}       & Label                & \multicolumn{1}{c|}{Opt.}          & 0.0   & $\infty$ & \multicolumn{1}{c|}{0.0}  & 2.1  & 174  & 2.31 \\
\textbf{CA (bin search)} & Label                & \multicolumn{1}{c|}{Opt.}          & 0.0   & $\infty$ & \multicolumn{1}{c|}{0.0}  & 0.0  & 459  & 2.61 \\ \hline
                        & \multicolumn{1}{l}{} & \multicolumn{1}{l}{}               &       &          &                           &      &      &     
\end{tabular}%
}
\end{table}
\begin{table}[H]
\centering
\caption{Attack performance under Blacklight detection on WRN and CIFAR100 (Clean Accuracy: $79.4\%$)}
\label{tab:blacklight_cifar100_wrn}
\resizebox{0.47\textwidth}{!}{%
\begin{tabular}{lcccccccc}
\hline
Attack &
  \begin{tabular}[c]{@{}c@{}}Query\\ Type\end{tabular} &
  \begin{tabular}[c]{@{}c@{}}Pert.\\ Type\end{tabular} &
  \begin{tabular}[c]{@{}c@{}}Det.\\ Rate \%\end{tabular} &
  \begin{tabular}[c]{@{}c@{}}\# Q\\ to Det.\end{tabular} &
  \begin{tabular}[c]{@{}c@{}}Det.\\ Cov. \%\end{tabular} &
  \begin{tabular}[c]{@{}c@{}}Model\\ Acc.\end{tabular} &
  \# Q &
  \begin{tabular}[c]{@{}c@{}}Dist.\\ $\ell_2$\end{tabular} \\ \hline
Bandit                  & Score                & \multicolumn{1}{c|}{$\ell_\infty$} & 100.0 & 1.0      & \multicolumn{1}{c|}{62.1}  & 0.0  & 13   & 4.25 \\
NES                     & Score                & \multicolumn{1}{c|}{$\ell_\infty$} & 100.0 & 8.8      & \multicolumn{1}{c|}{21.7}  & 0.0  & 151  & 0.87 \\
Parsimonious            & Score                & \multicolumn{1}{c|}{$\ell_\infty$} & 100.0 & 2.0      & \multicolumn{1}{c|}{95.2}  & 0.0  & 75   & 4.29 \\
Sign                    & Score                & \multicolumn{1}{c|}{$\ell_\infty$} & 100.0 & 2.0      & \multicolumn{1}{c|}{87.2}  & 0.0  & 54   & 4.26 \\
Square                  & Score                & \multicolumn{1}{c|}{$\ell_\infty$} & 100.0 & 2.0      & \multicolumn{1}{c|}{65.1}  & 0.0  & 10   & 4.26 \\
ZOSignSGD               & Score                & \multicolumn{1}{c|}{$\ell_\infty$} & 100.0 & 2.0      & \multicolumn{1}{c|}{50.6}  & 0.0  & 159  & 0.87 \\
GeoDA                   & Label                & \multicolumn{1}{c|}{$\ell_\infty$} & 100.0 & 1.0      & \multicolumn{1}{c|}{90.0}  & 0.0  & 139  & 1.93 \\
HSJ                     & Label                & \multicolumn{1}{c|}{$\ell_\infty$} & 100.0 & 7.3      & \multicolumn{1}{c|}{91.9}  & 0.0  & 185  & 1.89 \\
Opt                     & Label                & \multicolumn{1}{c|}{$\ell_\infty$} & 100.0 & 8.6      & \multicolumn{1}{c|}{73.8}  & 24.1 & 1269 & 0.93 \\
RayS                    & Label                & \multicolumn{1}{c|}{$\ell_\infty$} & 100.0 & 6.0      & \multicolumn{1}{c|}{74.1}  & 0.0  & 165  & 4.20 \\
SignFlip                & Label                & \multicolumn{1}{c|}{$\ell_\infty$} & 100.0 & 8.7      & \multicolumn{1}{c|}{43.1}  & 0.0  & 29   & 2.70 \\
SignOPT                 & Label                & \multicolumn{1}{c|}{$\ell_\infty$} & 100.0 & 8.6      & \multicolumn{1}{c|}{92.9}  & 19.0 & 624  & 0.75 \\
Bandit                  & Score                & \multicolumn{1}{c|}{$\ell_2$}      & 100.0 & 1.0      & \multicolumn{1}{c|}{97.7}  & 0.0  & 70   & 1.54 \\
NES                     & Score                & \multicolumn{1}{c|}{$\ell_2$}      & 100.0 & 8.9      & \multicolumn{1}{c|}{31.8}  & 0.0  & 442  & 0.32 \\
Simple                  & Score                & \multicolumn{1}{c|}{$\ell_2$}      & 100.0 & 1.0      & \multicolumn{1}{c|}{99.4}  & 0.0  & 437  & 0.48 \\
Square                  & Score                & \multicolumn{1}{c|}{$\ell_2$}      & 100.0 & 2.0      & \multicolumn{1}{c|}{65.2}  & 0.0  & 7    & 3.88 \\
ZOSignSGD               & Score                & \multicolumn{1}{c|}{$\ell_2$}      & 100.0 & 2.0      & \multicolumn{1}{c|}{53.6}  & 0.5  & 964  & 0.31 \\
Boundary                & Label                & \multicolumn{1}{c|}{$\ell_2$}      & 100.0 & 7.2      & \multicolumn{1}{c|}{52.0}  & 4.8  & 81   & 2.46 \\
GeoDA                   & Label                & \multicolumn{1}{c|}{$\ell_2$}      & 100.0 & 1.0      & \multicolumn{1}{c|}{90.4}  & 0.0  & 120  & 1.86 \\
HSJ                     & Label                & \multicolumn{1}{c|}{$\ell_2$}      & 100.0 & 7.3      & \multicolumn{1}{c|}{91.5}  & 0.0  & 143  & 1.55 \\
Opt                     & Label                & \multicolumn{1}{c|}{$\ell_2$}      & 100.0 & 8.6      & \multicolumn{1}{c|}{65.2}  & 22.2 & 586  & 1.35 \\
SignOPT                 & Label                & \multicolumn{1}{c|}{$\ell_2$}      & 99.9  & 8.5      & \multicolumn{1}{c|}{89.0}  & 16.7 & 456  & 1.04 \\
PointWise               & Label                & \multicolumn{1}{c|}{Opt.}          & 100.0 & 1.0      & \multicolumn{1}{c|}{99.1}  & 0.0  & 469  & 1.57 \\
SparseEvo               & Label                & \multicolumn{1}{c|}{Opt.}          & 91.4  & 1.0      & \multicolumn{1}{c|}{100.0} & 0.0  & 9145 & 2.16 \\ \hline
\textbf{CA (sssp)}       & Label                & \multicolumn{1}{c|}{Opt.}          & 0.0   & $\infty$ & \multicolumn{1}{c|}{0.0}   & 1.6  & 218  & 2.56 \\
\textbf{CA (bin search)} & Label                & \multicolumn{1}{c|}{Opt.}          & 0.0   & $\infty$ & \multicolumn{1}{c|}{0.0}   & 0.0  & 458  & 2.65 \\ \hline
                        & \multicolumn{1}{l}{} & \multicolumn{1}{l}{}               &       &          &                            &      &      &     
\end{tabular}%
}
\end{table}
\clearpage
\begin{table}[H]
\centering
\caption{Attack performance under RAND Pre-processing Defense on VGG and CIFAR10 (Clean Accuracy: $87.7\%$)}
\label{tab:cifar10_RAND_vgg}
\resizebox{0.47\textwidth}{!}{%
\begin{tabular}{lccccc}
\hline
Attack &
  \begin{tabular}[c]{@{}c@{}}Query\\ Type\end{tabular} &
  \begin{tabular}[c]{@{}c@{}}Perturbation\\ Type\end{tabular} &
  \# Query &
  Model Acc. &
  Dist. $\ell_2$ \\ \hline
\multicolumn{1}{l|}{Bandit}            & Score & \multicolumn{1}{c|}{$\ell_\infty$} & 116  & 36.1 & 4.68 \\
\multicolumn{1}{l|}{NES}               & Score & \multicolumn{1}{c|}{$\ell_\infty$} & 612  & 50.7 & 1.69 \\
\multicolumn{1}{l|}{Parsimonious}      & Score & \multicolumn{1}{c|}{$\ell_\infty$} & 565  & 63.7 & 4.79 \\
\multicolumn{1}{l|}{Sign}              & Score & \multicolumn{1}{c|}{$\ell_\infty$} & 257  & 55.3 & 4.75 \\
\multicolumn{1}{l|}{Square}            & Score & \multicolumn{1}{c|}{$\ell_\infty$} & 91   & 32.9 & 4.80 \\
\multicolumn{1}{l|}{ZOSignSGD}         & Score & \multicolumn{1}{c|}{$\ell_\infty$} & 644  & 49.9 & 1.72 \\
\multicolumn{1}{l|}{GeoDA}             & Label & \multicolumn{1}{c|}{$\ell_\infty$} & 324  & 58.1 & 2.58 \\
\multicolumn{1}{l|}{HSJ}               & Label & \multicolumn{1}{c|}{$\ell_\infty$} & 591  & 59.7 & 2.26 \\
\multicolumn{1}{l|}{Opt}               & Label & \multicolumn{1}{c|}{$\ell_\infty$} & 957  & 87.3 & 0.25 \\
\multicolumn{1}{l|}{RayS}              & Label & \multicolumn{1}{c|}{$\ell_\infty$} & 251  & 60.1 & 4.60 \\
\multicolumn{1}{l|}{SignFlip}          & Label & \multicolumn{1}{c|}{$\ell_\infty$} & 365  & 51.1 & 3.31 \\
\multicolumn{1}{l|}{SignOPT}           & Label & \multicolumn{1}{c|}{$\ell_\infty$} & 1411 & 82.0 & 0.30 \\
\multicolumn{1}{l|}{Bandit}            & Score & \multicolumn{1}{c|}{$\ell_2$}      & 773  & 76.5 & 3.36 \\
\multicolumn{1}{l|}{NES}               & Score & \multicolumn{1}{c|}{$\ell_2$}      & 1875 & 72.2 & 0.75 \\
\multicolumn{1}{l|}{Simple}            & Score & \multicolumn{1}{c|}{$\ell_2$}      & 1693 & 89.3 & 0.14 \\
\multicolumn{1}{l|}{Square}            & Score & \multicolumn{1}{c|}{$\ell_2$}      & 100  & 33.9 & 4.33 \\
\multicolumn{1}{l|}{ZOSignSGD}         & Score & \multicolumn{1}{c|}{$\ell_2$}      & 2840 & 78.3 & 0.63 \\
\multicolumn{1}{l|}{Boundary}          & Label & \multicolumn{1}{c|}{$\ell_2$}      & 43   & 54.4 & 2.46 \\
\multicolumn{1}{l|}{GeoDA}             & Label & \multicolumn{1}{c|}{$\ell_2$}      & 299  & 62.0 & 2.35 \\
\multicolumn{1}{l|}{HSJ}               & Label & \multicolumn{1}{c|}{$\ell_2$}      & 735  & 56.5 & 3.20 \\
\multicolumn{1}{l|}{Opt}               & Label & \multicolumn{1}{c|}{$\ell_2$}      & 807  & 69.4 & 1.92 \\
\multicolumn{1}{l|}{SignOPT}           & Label & \multicolumn{1}{c|}{$\ell_2$}      & 586  & 50.9 & 2.76 \\
\multicolumn{1}{l|}{PointWise}         & Label & \multicolumn{1}{c|}{Optimized}  & 2813 & 83.0 & 2.02 \\
\multicolumn{1}{l|}{SparseEvo}         & Label & \multicolumn{1}{c|}{Optimized}  & 9492 & 80.8 & 1.62 \\ \hline
\multicolumn{1}{l|}{\textbf{CA (sssp)}} & Label & \multicolumn{1}{c|}{Optimized}  & 359  & 5.7  & 3.53 \\
\multicolumn{1}{l|}{\textbf{CA (bin search)}} &
  Label &
  \multicolumn{1}{c|}{Optimized} &
  473 &
  0.0 &
  4.94 \\ \hline
\end{tabular}%
}
\end{table}
\begin{table}[H]
\centering
\caption{Attack performance under RAND Pre-processing Defense on ResNet and CIFAR10 (Clean Accuracy $92.3\%$)}
\label{tab:cifar10_RAND_resnet}
\resizebox{0.47\textwidth}{!}{%
\begin{tabular}{lccccc}
\hline
Attack &
  \begin{tabular}[c]{@{}c@{}}Query\\ Type\end{tabular} &
  \begin{tabular}[c]{@{}c@{}}Perturbation\\ Type\end{tabular} &
  \# Query &
  Model Acc. &
  Dist. $\ell_2$ \\ \hline
\multicolumn{1}{l|}{Bandit}            & Score & \multicolumn{1}{c|}{$\ell_\infty$} & 52   & 27.2 & 4.88 \\
\multicolumn{1}{l|}{NES}               & Score & \multicolumn{1}{c|}{$\ell_\infty$} & 605  & 49.9 & 1.84 \\
\multicolumn{1}{l|}{Parsimonious}      & Score & \multicolumn{1}{c|}{$\ell_\infty$} & 254  & 50.1 & 4.97 \\
\multicolumn{1}{l|}{Sign}              & Score & \multicolumn{1}{c|}{$\ell_\infty$} & 394  & 58.3 & 4.85 \\
\multicolumn{1}{l|}{Square}            & Score & \multicolumn{1}{c|}{$\ell_\infty$} & 32   & 21.4 & 4.95 \\
\multicolumn{1}{l|}{ZOSignSGD}         & Score & \multicolumn{1}{c|}{$\ell_\infty$} & 674  & 51.4 & 1.86 \\
\multicolumn{1}{l|}{GeoDA}             & Label & \multicolumn{1}{c|}{$\ell_\infty$} & 265  & 62.8 & 2.53 \\
\multicolumn{1}{l|}{HSJ}               & Label & \multicolumn{1}{c|}{$\ell_\infty$} & 641  & 64.6 & 2.12 \\
\multicolumn{1}{l|}{Opt}               & Label & \multicolumn{1}{c|}{$\ell_\infty$} & 655  & 88.2 & 0.18 \\
\multicolumn{1}{l|}{RayS}              & Label & \multicolumn{1}{c|}{$\ell_\infty$} & 306  & 62.9 & 4.75 \\
\multicolumn{1}{l|}{SignFlip}          & Label & \multicolumn{1}{c|}{$\ell_\infty$} & 902  & 56.5 & 3.33 \\
\multicolumn{1}{l|}{SignOPT}           & Label & \multicolumn{1}{c|}{$\ell_\infty$} & 1086 & 85.6 & 0.24 \\
\multicolumn{1}{l|}{Bandit}            & Score & \multicolumn{1}{c|}{$\ell_2$}      & 747  & 76.0 & 3.45 \\
\multicolumn{1}{l|}{NES}               & Score & \multicolumn{1}{c|}{$\ell_2$}      & 1808 & 71.2 & 0.77 \\
\multicolumn{1}{l|}{Simple}            & Score & \multicolumn{1}{c|}{$\ell_2$}      & 2139 & 90.4 & 0.15 \\
\multicolumn{1}{l|}{Square}            & Score & \multicolumn{1}{c|}{$\ell_2$}      & 44   & 24.2 & 4.49 \\
\multicolumn{1}{l|}{ZOSignSGD}         & Score & \multicolumn{1}{c|}{$\ell_2$}      & 2822 & 77.6 & 0.64 \\
\multicolumn{1}{l|}{Boundary}          & Label & \multicolumn{1}{c|}{$\ell_2$}      & 46   & 60.1 & 2.25 \\
\multicolumn{1}{l|}{GeoDA}             & Label & \multicolumn{1}{c|}{$\ell_2$}      & 333  & 66.7 & 2.14 \\
\multicolumn{1}{l|}{HSJ}               & Label & \multicolumn{1}{c|}{$\ell_2$}      & 1163 & 56.7 & 3.45 \\
\multicolumn{1}{l|}{Opt}               & Label & \multicolumn{1}{c|}{$\ell_2$}      & 724  & 72.8 & 1.93 \\
\multicolumn{1}{l|}{SignOPT}           & Label & \multicolumn{1}{c|}{$\ell_2$}      & 442  & 59.1 & 2.47 \\
\multicolumn{1}{l|}{PointWise}         & Label & \multicolumn{1}{c|}{Optimized}  & 3804 & 86.5 & 2.02 \\
\multicolumn{1}{l|}{SparseEvo}         & Label & \multicolumn{1}{c|}{Optimized}  & 8709 & 87.3 & 1.78 \\ \hline
\multicolumn{1}{l|}{\textbf{CA (sssp)}} & Label & \multicolumn{1}{c|}{Optimized}  & 406  & 8.1  & 3.80 \\
\multicolumn{1}{l|}{\textbf{CA (bin search)}} &
  Label &
  \multicolumn{1}{c|}{Optimized} &
  417 &
  10.8 &
  3.89 \\ \hline
\end{tabular}%
}
\end{table}
\begin{table}[H]
\centering
\caption{Attack performance under RAND Pre-processing Defense on ResNeXt and CIFAR10 (Clean Accuracy: $89.6\%$)}
\label{tab:cifar10_RAND_resnext}
\resizebox{0.47\textwidth}{!}{%
\begin{tabular}{lccccc}
\hline
Attack &
  \begin{tabular}[c]{@{}c@{}}Query\\ Type\end{tabular} &
  \begin{tabular}[c]{@{}c@{}}Perturbation\\ Type\end{tabular} &
  \# Query &
  Model Acc. &
  Dist. $\ell_2$ \\ \hline
\multicolumn{1}{l|}{Bandit}            & Score & \multicolumn{1}{c|}{$\ell_\infty$} & 54   & 25.9 & 4.81 \\
\multicolumn{1}{l|}{NES}               & Score & \multicolumn{1}{c|}{$\ell_\infty$} & 446  & 48.2 & 1.61 \\
\multicolumn{1}{l|}{Parsimonious}      & Score & \multicolumn{1}{c|}{$\ell_\infty$} & 537  & 68.1 & 4.88 \\
\multicolumn{1}{l|}{Sign}              & Score & \multicolumn{1}{c|}{$\ell_\infty$} & 454  & 63.8 & 4.75 \\
\multicolumn{1}{l|}{Square}            & Score & \multicolumn{1}{c|}{$\ell_\infty$} & 90   & 24.2 & 4.91 \\
\multicolumn{1}{l|}{ZOSignSGD}         & Score & \multicolumn{1}{c|}{$\ell_\infty$} & 487  & 49.4 & 1.64 \\
\multicolumn{1}{l|}{GeoDA}             & Label & \multicolumn{1}{c|}{$\ell_\infty$} & 197  & 58.9 & 2.43 \\
\multicolumn{1}{l|}{HSJ}               & Label & \multicolumn{1}{c|}{$\ell_\infty$} & 363  & 59.3 & 2.30 \\
\multicolumn{1}{l|}{Opt}               & Label & \multicolumn{1}{c|}{$\ell_\infty$} & 771  & 89.8 & 0.30 \\
\multicolumn{1}{l|}{RayS}              & Label & \multicolumn{1}{c|}{$\ell_\infty$} & 305  & 65.3 & 4.60 \\
\multicolumn{1}{l|}{SignFlip}          & Label & \multicolumn{1}{c|}{$\ell_\infty$} & 334  & 54.0 & 2.91 \\
\multicolumn{1}{l|}{SignOPT}           & Label & \multicolumn{1}{c|}{$\ell_\infty$} & 895  & 84.9 & 0.33 \\
\multicolumn{1}{l|}{Bandit}            & Score & \multicolumn{1}{c|}{$\ell_2$}      & 407  & 75.1 & 3.03 \\
\multicolumn{1}{l|}{NES}               & Score & \multicolumn{1}{c|}{$\ell_2$}      & 1611 & 74.6 & 0.70 \\
\multicolumn{1}{l|}{Simple}            & Score & \multicolumn{1}{c|}{$\ell_2$}      & 1692 & 90.2 & 0.14 \\
\multicolumn{1}{l|}{Square}            & Score & \multicolumn{1}{c|}{$\ell_2$}      & 123  & 34.0 & 4.41 \\
\multicolumn{1}{l|}{ZOSignSGD}         & Score & \multicolumn{1}{c|}{$\ell_2$}      & 2582 & 80.9 & 0.59 \\
\multicolumn{1}{l|}{Boundary}          & Label & \multicolumn{1}{c|}{$\ell_2$}      & 18   & 52.9 & 2.50 \\
\multicolumn{1}{l|}{GeoDA}             & Label & \multicolumn{1}{c|}{$\ell_2$}      & 146  & 58.9 & 2.32 \\
\multicolumn{1}{l|}{HSJ}               & Label & \multicolumn{1}{c|}{$\ell_2$}      & 593  & 57.7 & 2.86 \\
\multicolumn{1}{l|}{Opt}               & Label & \multicolumn{1}{c|}{$\ell_2$}      & 694  & 76.1 & 1.58 \\
\multicolumn{1}{l|}{SignOPT}           & Label & \multicolumn{1}{c|}{$\ell_2$}      & 337  & 67.1 & 1.91 \\
\multicolumn{1}{l|}{PointWise}         & Label & \multicolumn{1}{c|}{Optimized}  & 4516 & 84.5 & 2.18 \\
\multicolumn{1}{l|}{SparseEvo}         & Label & \multicolumn{1}{c|}{Optimized}  & 9632 & 87.4 & 1.85 \\ \hline
\multicolumn{1}{l|}{\textbf{CA (sssp)}} & Label & \multicolumn{1}{c|}{Optimized}  & 259  & 9.8  & 2.77 \\
\multicolumn{1}{l|}{\textbf{CA (bin search)}} &
  Label &
  \multicolumn{1}{c|}{Optimized} &
  416 &
  10.6 &
  2.75 \\ \hline
\end{tabular}%
}
\end{table}
\begin{table}[H]
\centering
\caption{Attack performance under RAND Pre-processing Defense on WRN and CIFAR10 (Clean Accuracy: $92.6\%$)}
\label{tab:cifar10_RAND_wrn}
\resizebox{0.47\textwidth}{!}{%
\begin{tabular}{lccccc}
\hline
Attack &
  \begin{tabular}[c]{@{}c@{}}Query\\ Type\end{tabular} &
  \begin{tabular}[c]{@{}c@{}}Perturbation\\ Type\end{tabular} &
  \# Query &
  Model Acc. &
  Dist. $\ell_2$ \\ \hline
\multicolumn{1}{l|}{Bandit}            & Score & \multicolumn{1}{c|}{$\ell_\infty$} & 55   & 28.0 & 4.92 \\
\multicolumn{1}{l|}{NES}               & Score & \multicolumn{1}{c|}{$\ell_\infty$} & 600  & 51.6 & 1.78 \\
\multicolumn{1}{l|}{Parsimonious}      & Score & \multicolumn{1}{c|}{$\ell_\infty$} & 695  & 71.4 & 5.04 \\
\multicolumn{1}{l|}{Sign}              & Score & \multicolumn{1}{c|}{$\ell_\infty$} & 627  & 70.6 & 4.82 \\
\multicolumn{1}{l|}{Square}            & Score & \multicolumn{1}{c|}{$\ell_\infty$} & 142  & 28.6 & 5.00 \\
\multicolumn{1}{l|}{ZOSignSGD}         & Score & \multicolumn{1}{c|}{$\ell_\infty$} & 642  & 52.3 & 1.80 \\
\multicolumn{1}{l|}{GeoDA}             & Label & \multicolumn{1}{c|}{$\ell_\infty$} & 202  & 60.1 & 2.54 \\
\multicolumn{1}{l|}{HSJ}               & Label & \multicolumn{1}{c|}{$\ell_\infty$} & 468  & 61.4 & 2.30 \\
\multicolumn{1}{l|}{Opt}               & Label & \multicolumn{1}{c|}{$\ell_\infty$} & 1009 & 90.5 & 0.24 \\
\multicolumn{1}{l|}{RayS}              & Label & \multicolumn{1}{c|}{$\ell_\infty$} & 454  & 69.2 & 4.62 \\
\multicolumn{1}{l|}{SignFlip}          & Label & \multicolumn{1}{c|}{$\ell_\infty$} & 280  & 52.4 & 3.04 \\
\multicolumn{1}{l|}{SignOPT}           & Label & \multicolumn{1}{c|}{$\ell_\infty$} & 1145 & 86.4 & 0.28 \\
\multicolumn{1}{l|}{Bandit}            & Score & \multicolumn{1}{c|}{$\ell_2$}      & 616  & 77.1 & 3.22 \\
\multicolumn{1}{l|}{NES}               & Score & \multicolumn{1}{c|}{$\ell_2$}      & 2074 & 77.0 & 0.82 \\
\multicolumn{1}{l|}{Simple}            & Score & \multicolumn{1}{c|}{$\ell_2$}      & 1639 & 92.0 & 0.15 \\
\multicolumn{1}{l|}{Square}            & Score & \multicolumn{1}{c|}{$\ell_2$}      & 62   & 30.9 & 4.57 \\
\multicolumn{1}{l|}{ZOSignSGD}         & Score & \multicolumn{1}{c|}{$\ell_2$}      & 2982 & 82.0 & 0.68 \\
\multicolumn{1}{l|}{Boundary}          & Label & \multicolumn{1}{c|}{$\ell_2$}      & 36   & 56.1 & 2.53 \\
\multicolumn{1}{l|}{GeoDA}             & Label & \multicolumn{1}{c|}{$\ell_2$}      & 191  & 60.9 & 2.35 \\
\multicolumn{1}{l|}{HSJ}               & Label & \multicolumn{1}{c|}{$\ell_2$}      & 627  & 57.2 & 2.97 \\
\multicolumn{1}{l|}{Opt}               & Label & \multicolumn{1}{c|}{$\ell_2$}      & 759  & 73.5 & 1.69 \\
\multicolumn{1}{l|}{SignOPT}           & Label & \multicolumn{1}{c|}{$\ell_2$}      & 527  & 60.5 & 2.34 \\
\multicolumn{1}{l|}{PointWise}         & Label & \multicolumn{1}{c|}{Optimized}  & 3721 & 87.2 & 1.90 \\
\multicolumn{1}{l|}{SparseEvo}         & Label & \multicolumn{1}{c|}{Optimized}  & 9730 & 89.7 & 1.7  \\ \hline
\multicolumn{1}{l|}{\textbf{CA (sssp)}} & Label & \multicolumn{1}{c|}{Optimized}  & 401  & 6.1  & 3.63 \\
\multicolumn{1}{l|}{\textbf{CA (bin search)}} &
  Label &
  \multicolumn{1}{c|}{Optimized} &
  461 &
  0.0 &
  4.80 \\ \hline
\end{tabular}%
}
\end{table}
\clearpage
\begin{table}[H]
\centering
\caption{Attack performance under RAND Pre-processing Defense on VGG and CIFAR100 (Clean Accuracy: $61.4\%$)}
\label{tab:cifar100_RAND_vgg}
\resizebox{0.47\textwidth}{!}{%
\begin{tabular}{lccccc}
\hline
Attack &
  \begin{tabular}[c]{@{}c@{}}Query\\ Type\end{tabular} &
  \begin{tabular}[c]{@{}c@{}}Perturbation\\ Type\end{tabular} &
  \# Query &
  Model Acc. &
  Dist.  $\ell_2$ \\ \hline
\multicolumn{1}{l|}{Bandit}            & Score & \multicolumn{1}{c|}{$\ell_\infty$} & 9    & 8.8  & 3.28 \\
\multicolumn{1}{l|}{NES}               & Score & \multicolumn{1}{c|}{$\ell_\infty$} & 326  & 34.4 & 0.91 \\
\multicolumn{1}{l|}{Parsimonious}      & Score & \multicolumn{1}{c|}{$\ell_\infty$} & 218  & 37.1 & 3.26 \\
\multicolumn{1}{l|}{Sign}              & Score & \multicolumn{1}{c|}{$\ell_\infty$} & 138  & 32.5 & 3.26 \\
\multicolumn{1}{l|}{Square}            & Score & \multicolumn{1}{c|}{$\ell_\infty$} & 25   & 9.0  & 3.31 \\
\multicolumn{1}{l|}{ZOSignSGD}         & Score & \multicolumn{1}{c|}{$\ell_\infty$} & 376  & 32.9 & 0.93 \\
\multicolumn{1}{l|}{GeoDA}             & Label & \multicolumn{1}{c|}{$\ell_\infty$} & 139  & 40.5 & 2.36 \\
\multicolumn{1}{l|}{HSJ}               & Label & \multicolumn{1}{c|}{$\ell_\infty$} & 301  & 43.7 & 2.12 \\
\multicolumn{1}{l|}{Opt}               & Label & \multicolumn{1}{c|}{$\ell_\infty$} & 908  & 66.1 & 0.38 \\
\multicolumn{1}{l|}{RayS}              & Label & \multicolumn{1}{c|}{$\ell_\infty$} & 127  & 46.4 & 4.03 \\
\multicolumn{1}{l|}{SignFlip}          & Label & \multicolumn{1}{c|}{$\ell_\infty$} & 194  & 43.9 & 2.64 \\
\multicolumn{1}{l|}{SignOPT}           & Label & \multicolumn{1}{c|}{$\ell_\infty$} & 997  & 60.3 & 0.46 \\
\multicolumn{1}{l|}{Bandit}            & Score & \multicolumn{1}{c|}{$\ell_2$}      & 187  & 43.9 & 1.43 \\
\multicolumn{1}{l|}{NES}               & Score & \multicolumn{1}{c|}{$\ell_2$}      & 1215 & 48.5 & 0.39 \\
\multicolumn{1}{l|}{Simple}            & Score & \multicolumn{1}{c|}{$\ell_2$}      & 1335 & 60.4 & 0.09 \\
\multicolumn{1}{l|}{Square}            & Score & \multicolumn{1}{c|}{$\ell_2$}      & 37   & 10.0 & 2.99 \\
\multicolumn{1}{l|}{ZOSignSGD}         & Score & \multicolumn{1}{c|}{$\ell_2$}      & 1804 & 54.8 & 0.33 \\
\multicolumn{1}{l|}{Boundary}          & Label & \multicolumn{1}{c|}{$\ell_2$}      & 26   & 41.0 & 2.48 \\
\multicolumn{1}{l|}{GeoDA}             & Label & \multicolumn{1}{c|}{$\ell_2$}      & 150  & 43.5 & 2.20 \\
\multicolumn{1}{l|}{HSJ}               & Label & \multicolumn{1}{c|}{$\ell_2$}      & 278  & 46.3 & 2.41 \\
\multicolumn{1}{l|}{Opt}               & Label & \multicolumn{1}{c|}{$\ell_2$}      & 765  & 56.5 & 1.44 \\
\multicolumn{1}{l|}{SignOPT}           & Label & \multicolumn{1}{c|}{$\ell_2$}      & 390  & 51.1 & 1.93 \\
\multicolumn{1}{l|}{PointWise}         & Label & \multicolumn{1}{c|}{Optimized}  & 956  & 51.3 & 1.07 \\
\multicolumn{1}{l|}{SparseEvo}         & Label & \multicolumn{1}{c|}{Optimized}  & 8793 & 53.2 & 0.95 \\ \hline
\multicolumn{1}{l|}{\textbf{CA (sssp)}} & Label & \multicolumn{1}{c|}{Optimized}  & 168  & 0.9  & 2.25 \\
\multicolumn{1}{l|}{\textbf{CA (bin search)}} &
  Label &
  \multicolumn{1}{c|}{Optimized} &
  457 &
  0.0 &
  2.52 \\ \hline
\end{tabular}%
}
\end{table}
\begin{table}[H]
\centering
\caption{Attack performance under RAND Pre-processing Defense on ResNet and CIFAR100 (Clean Accuracy: $62.4\%$)}
\label{tab:cifar100_RAND_resnet}
\resizebox{0.47\textwidth}{!}{%
\begin{tabular}{lccccc}
\hline
Attack &
  \begin{tabular}[c]{@{}c@{}}Query\\ Type\end{tabular} &
  \begin{tabular}[c]{@{}c@{}}Perturbation\\ Type\end{tabular} &
  \# Query &
  Model Acc. &
  Dist. $\ell_2$ \\ \hline
\multicolumn{1}{l|}{Bandit}       & Score & \multicolumn{1}{c|}{$\ell_\infty$} & 13   & 10.2 & 3.36 \\
\multicolumn{1}{l|}{NES}          & Score & \multicolumn{1}{c|}{$\ell_\infty$} & 279  & 36.4 & 0.92 \\
\multicolumn{1}{l|}{Parsimonious} & Score & \multicolumn{1}{c|}{$\ell_\infty$} & 85   & 31.5 & 3.39 \\
\multicolumn{1}{l|}{Sign}         & Score & \multicolumn{1}{c|}{$\ell_\infty$} & 221  & 33.1 & 3.33 \\
\multicolumn{1}{l|}{Square}       & Score & \multicolumn{1}{c|}{$\ell_\infty$} & 12   & 9.7  & 3.48 \\
\multicolumn{1}{l|}{ZOSignSGD}    & Score & \multicolumn{1}{c|}{$\ell_\infty$} & 330  & 34.9 & 0.94 \\
\multicolumn{1}{l|}{GeoDA}        & Label & \multicolumn{1}{c|}{$\ell_\infty$} & 220  & 40.5 & 2.34 \\
\multicolumn{1}{l|}{HSJ}          & Label & \multicolumn{1}{c|}{$\ell_\infty$} & 472  & 43.6 & 2.17 \\
\multicolumn{1}{l|}{Opt}          & Label & \multicolumn{1}{c|}{$\ell_\infty$} & 907  & 62.7 & 0.32 \\
\multicolumn{1}{l|}{RayS}         & Label & \multicolumn{1}{c|}{$\ell_\infty$} & 119  & 44.0 & 3.97 \\
\multicolumn{1}{l|}{SignFlip}     & Label & \multicolumn{1}{c|}{$\ell_\infty$} & 383  & 40.0 & 2.93 \\
\multicolumn{1}{l|}{SignOPT}      & Label & \multicolumn{1}{c|}{$\ell_\infty$} & 912  & 57.9 & 0.43 \\
\multicolumn{1}{l|}{Bandit}       & Score & \multicolumn{1}{c|}{$\ell_2$}      & 291  & 47.8 & 1.47 \\
\multicolumn{1}{l|}{NES}          & Score & \multicolumn{1}{c|}{$\ell_2$}      & 964  & 51.0 & 0.36 \\
\multicolumn{1}{l|}{Simple}       & Score & \multicolumn{1}{c|}{$\ell_2$}      & 1152 & 62.8 & 0.09 \\
\multicolumn{1}{l|}{Square}       & Score & \multicolumn{1}{c|}{$\ell_2$}      & 6    & 10.0 & 3.05 \\
\multicolumn{1}{l|}{ZOSignSGD}    & Score & \multicolumn{1}{c|}{$\ell_2$}      & 1514 & 55.2 & 0.30 \\
\multicolumn{1}{l|}{Boundary}     & Label & \multicolumn{1}{c|}{$\ell_2$}      & 16   & 37.7 & 2.52 \\
\multicolumn{1}{l|}{GeoDA}        & Label & \multicolumn{1}{c|}{$\ell_2$}      & 256  & 44.3 & 2.14 \\
\multicolumn{1}{l|}{HSJ}          & Label & \multicolumn{1}{c|}{$\ell_2$}      & 381  & 43.8 & 2.56 \\
\multicolumn{1}{l|}{Opt}          & Label & \multicolumn{1}{c|}{$\ell_2$}      & 806  & 54.5 & 1.46 \\
\multicolumn{1}{l|}{SignOPT}      & Label & \multicolumn{1}{c|}{$\ell_2$}      & 452  & 46.0 & 1.96 \\
\multicolumn{1}{l|}{PointWise}    & Label & \multicolumn{1}{c|}{Optimized}  & 2051 & 52.4 & 1.21 \\
\multicolumn{1}{l|}{SparseEvo}    & Label & \multicolumn{1}{c|}{Optimized}  & 9043 & 56.3 & 1.12 \\ \hline
\multicolumn{1}{l|}{\textbf{CA (sssp)}} &
  Label &
  \multicolumn{1}{c|}{Optimized} &
  167 &
  1.6 &
  2.27 \\
\multicolumn{1}{l|}{\textbf{CA (bin search)}} &
  Label &
  \multicolumn{1}{c|}{Optimized} &
  459 &
  0.0 &
  2.81 \\ \hline
\end{tabular}%
}
\end{table}
\begin{table}[H]
\centering
\caption{Attack performance under RAND Pre-processing Defense on ResNeXt and CIFAR100 (Clean Accuracy: $65.0\%$)}
\label{tab:cifar100_RAND_resnext}
\resizebox{0.47\textwidth}{!}{%
\begin{tabular}{lccccc}
\hline
Attack &
  \begin{tabular}[c]{@{}c@{}}Query\\ Type\end{tabular} &
  \begin{tabular}[c]{@{}c@{}}Perturbation\\ Type\end{tabular} &
  \# Query &
  Model Acc. &
  Dist. $\ell_2$ \\ \hline
\multicolumn{1}{l|}{Bandit}            & Score & \multicolumn{1}{c|}{$\ell_\infty$} & 20   & 9.2  & 3.43 \\
\multicolumn{1}{l|}{NES}               & Score & \multicolumn{1}{c|}{$\ell_\infty$} & 311  & 35.3 & 0.95 \\
\multicolumn{1}{l|}{Parsimonious}      & Score & \multicolumn{1}{c|}{$\ell_\infty$} & 359  & 40.7 & 3.54 \\
\multicolumn{1}{l|}{Sign}              & Score & \multicolumn{1}{c|}{$\ell_\infty$} & 309  & 38.2 & 3.44 \\
\multicolumn{1}{l|}{Square}            & Score & \multicolumn{1}{c|}{$\ell_\infty$} & 14   & 8.1  & 3.45 \\
\multicolumn{1}{l|}{ZOSignSGD}         & Score & \multicolumn{1}{c|}{$\ell_\infty$} & 348  & 33.2 & 0.97 \\
\multicolumn{1}{l|}{GeoDA}             & Label & \multicolumn{1}{c|}{$\ell_\infty$} & 152  & 48.6 & 2.06 \\
\multicolumn{1}{l|}{HSJ}               & Label & \multicolumn{1}{c|}{$\ell_\infty$} & 184  & 51.3 & 1.94 \\
\multicolumn{1}{l|}{Opt}               & Label & \multicolumn{1}{c|}{$\ell_\infty$} & 951  & 73.9 & 0.44 \\
\multicolumn{1}{l|}{RayS}              & Label & \multicolumn{1}{c|}{$\ell_\infty$} & 203  & 56.7 & 4.23 \\
\multicolumn{1}{l|}{SignFlip}          & Label & \multicolumn{1}{c|}{$\ell_\infty$} & 132  & 49.3 & 2.32 \\
\multicolumn{1}{l|}{SignOPT}           & Label & \multicolumn{1}{c|}{$\ell_\infty$} & 798  & 71.8 & 0.45 \\
\multicolumn{1}{l|}{Bandit}            & Score & \multicolumn{1}{c|}{$\ell_2$}      & 188  & 45.7 & 1.49 \\
\multicolumn{1}{l|}{NES}               & Score & \multicolumn{1}{c|}{$\ell_2$}      & 1117 & 54.3 & 0.41 \\
\multicolumn{1}{l|}{Simple}            & Score & \multicolumn{1}{c|}{$\ell_2$}      & 1607 & 65.1 & 0.09 \\
\multicolumn{1}{l|}{Square}            & Score & \multicolumn{1}{c|}{$\ell_2$}      & 31   & 13.9 & 3.19 \\
\multicolumn{1}{l|}{ZOSignSGD}         & Score & \multicolumn{1}{c|}{$\ell_2$}      & 1964 & 58.3 & 0.35 \\
\multicolumn{1}{l|}{Boundary}          & Label & \multicolumn{1}{c|}{$\ell_2$}      & 17   & 49.7 & 2.21 \\
\multicolumn{1}{l|}{GeoDA}             & Label & \multicolumn{1}{c|}{$\ell_2$}      & 118  & 52.3 & 2.00 \\
\multicolumn{1}{l|}{HSJ}               & Label & \multicolumn{1}{c|}{$\ell_2$}      & 301  & 51.9 & 2.19 \\
\multicolumn{1}{l|}{Opt}               & Label & \multicolumn{1}{c|}{$\ell_2$}      & 741  & 65.1 & 1.35 \\
\multicolumn{1}{l|}{SignOPT}           & Label & \multicolumn{1}{c|}{$\ell_2$}      & 399  & 56.7 & 1.76 \\
\multicolumn{1}{l|}{PointWise}         & Label & \multicolumn{1}{c|}{Optimized}  & 1884 & 55.9 & 1.20 \\
\multicolumn{1}{l|}{SparseEvo}         & Label & \multicolumn{1}{c|}{Optimized}  & 9364 & 61.2 & 1.03 \\ \hline
\multicolumn{1}{l|}{\textbf{CA (sssp)}} & Label & \multicolumn{1}{c|}{Optimized}  & 159  & 1.4  & 2.21 \\
\multicolumn{1}{l|}{\textbf{CA (bin search)}} &
  Label &
  \multicolumn{1}{c|}{Optimized} &
  459 &
  0.0 &
  2.43 \\ \hline
\end{tabular}%
}
\end{table}
\begin{table}[H]
\centering
\caption{Attack performance under RAND Pre-processing Defense on WRN and CIFAR100 (Clean Accuracy: $65.5\%$)}
\label{tab:cifar100_RAND_wrn}
\resizebox{0.47\textwidth}{!}{%
\begin{tabular}{lccccc}
\hline
Attack &
  \begin{tabular}[c]{@{}c@{}}Query\\ Type\end{tabular} &
  \begin{tabular}[c]{@{}c@{}}Perturbation\\ Type\end{tabular} &
  \# Query &
  Model Acc. &
  Dist. $\ell_2$ \\ \hline
\multicolumn{1}{l|}{Bandit}            & Score & \multicolumn{1}{c|}{$\ell_\infty$} & 12   & 11.1 & 3.52 \\
\multicolumn{1}{l|}{NES}               & Score & \multicolumn{1}{c|}{$\ell_\infty$} & 334  & 36.9 & 0.98 \\
\multicolumn{1}{l|}{Parsimonious}      & Score & \multicolumn{1}{c|}{$\ell_\infty$} & 282  & 43.2 & 3.60 \\
\multicolumn{1}{l|}{Sign}              & Score & \multicolumn{1}{c|}{$\ell_\infty$} & 144  & 43.9 & 3.51 \\
\multicolumn{1}{l|}{Square}            & Score & \multicolumn{1}{c|}{$\ell_\infty$} & 21   & 11.9 & 3.57 \\
\multicolumn{1}{l|}{ZOSignSGD}         & Score & \multicolumn{1}{c|}{$\ell_\infty$} & 383  & 36.9 & 0.99 \\
\multicolumn{1}{l|}{GeoDA}             & Label & \multicolumn{1}{c|}{$\ell_\infty$} & 136  & 52.1 & 2.21 \\
\multicolumn{1}{l|}{HSJ}               & Label & \multicolumn{1}{c|}{$\ell_\infty$} & 312  & 52.1 & 2.09 \\
\multicolumn{1}{l|}{Opt}               & Label & \multicolumn{1}{c|}{$\ell_\infty$} & 866  & 74.4 & 0.37 \\
\multicolumn{1}{l|}{RayS}              & Label & \multicolumn{1}{c|}{$\ell_\infty$} & 212  & 56.1 & 4.07 \\
\multicolumn{1}{l|}{SignFlip}          & Label & \multicolumn{1}{c|}{$\ell_\infty$} & 146  & 50.0 & 2.47 \\
\multicolumn{1}{l|}{SignOPT}           & Label & \multicolumn{1}{c|}{$\ell_\infty$} & 929  & 69.6 & 0.39 \\
\multicolumn{1}{l|}{Bandit}            & Score & \multicolumn{1}{c|}{$\ell_2$}      & 175  & 50.2 & 1.53 \\
\multicolumn{1}{l|}{NES}               & Score & \multicolumn{1}{c|}{$\ell_2$}      & 1131 & 55.3 & 0.41 \\
\multicolumn{1}{l|}{Simple}            & Score & \multicolumn{1}{c|}{$\ell_2$}      & 1268 & 65.2 & 0.09 \\
\multicolumn{1}{l|}{Square}            & Score & \multicolumn{1}{c|}{$\ell_2$}      & 20   & 13.1 & 3.19 \\
\multicolumn{1}{l|}{ZOSignSGD}         & Score & \multicolumn{1}{c|}{$\ell_2$}      & 1702 & 57.2 & 0.34 \\
\multicolumn{1}{l|}{Boundary}          & Label & \multicolumn{1}{c|}{$\ell_2$}      & 19   & 48.1 & 2.49 \\
\multicolumn{1}{l|}{GeoDA}             & Label & \multicolumn{1}{c|}{$\ell_2$}      & 171  & 50.7 & 2.17 \\
\multicolumn{1}{l|}{HSJ}               & Label & \multicolumn{1}{c|}{$\ell_2$}      & 250  & 50.3 & 2.34 \\
\multicolumn{1}{l|}{Opt}               & Label & \multicolumn{1}{c|}{$\ell_2$}      & 753  & 65.7 & 1.39 \\
\multicolumn{1}{l|}{SignOPT}           & Label & \multicolumn{1}{c|}{$\ell_2$}      & 395  & 56.0 & 1.81 \\
\multicolumn{1}{l|}{PointWise}         & Label & \multicolumn{1}{c|}{Optimized}  & 1617 & 55.1 & 1.07 \\
\multicolumn{1}{l|}{SparseEvo}         & Label & \multicolumn{1}{c|}{Optimized}  & 9157 & 59.1 & 0.85 \\ \hline
\multicolumn{1}{l|}{\textbf{CA (sssp)}} & Label & \multicolumn{1}{c|}{Optimized}  & 180  & 1.2  & 2.37 \\
\multicolumn{1}{l|}{\textbf{CA (bin search)}} &
  Label &
  \multicolumn{1}{c|}{Optimized} &
  457 &
  0.0 &
  2.49 \\ \hline
\end{tabular}%
}
\end{table}
\clearpage
\begin{table}[H]
\centering
\caption{Attack performance under RAND Post-processing Defense on VGG and CIFAR10 (Clean Accuracy: $90.5\%$)}
\label{tab:cifar10_post_RAND_vgg}
\resizebox{0.47\textwidth}{!}{%
\begin{tabular}{lccccc}
\hline
Attack &
  \begin{tabular}[c]{@{}c@{}}Query\\ Type\end{tabular} &
  \begin{tabular}[c]{@{}c@{}}Perturbation\\ Type\end{tabular} &
  \# Query &
  Model Acc. &
  Dist. $\ell_2$ \\ \hline
\multicolumn{1}{l|}{Bandit}       & Score & \multicolumn{1}{c|}{$\ell_\infty$} & 192  & 10.9 & 4.83 \\
\multicolumn{1}{l|}{NES}          & Score & \multicolumn{1}{c|}{$\ell_\infty$} & 364  & 0.0  & 1.43 \\
\multicolumn{1}{l|}{Parsimonious} & Score & \multicolumn{1}{c|}{$\ell_\infty$} & 215  & 8.8  & 4.90 \\
\multicolumn{1}{l|}{Sign}         & Score & \multicolumn{1}{c|}{$\ell_\infty$} & 134  & 1.2  & 4.88 \\
\multicolumn{1}{l|}{Square}       & Score & \multicolumn{1}{c|}{$\ell_\infty$} & 37   & 0.4  & 4.89 \\
\multicolumn{1}{l|}{ZOSignSGD}    & Score & \multicolumn{1}{c|}{$\ell_\infty$} & 389  & 0.1  & 1.40 \\
\multicolumn{1}{l|}{GeoDA}        & Label & \multicolumn{1}{c|}{$\ell_\infty$} & 371  & 52.9 & 2.69 \\
\multicolumn{1}{l|}{HSJ}          & Label & \multicolumn{1}{c|}{$\ell_\infty$} & 518  & 53.7 & 2.43 \\
\multicolumn{1}{l|}{Opt}          & Label & \multicolumn{1}{c|}{$\ell_\infty$} & 970  & 78.0 & 0.31 \\
\multicolumn{1}{l|}{RayS}         & Label & \multicolumn{1}{c|}{$\ell_\infty$} & 235  & 45.0 & 4.77 \\
\multicolumn{1}{l|}{SignFlip}     & Label & \multicolumn{1}{c|}{$\ell_\infty$} & 76   & 17.2 & 3.91 \\
\multicolumn{1}{l|}{SignOPT}      & Label & \multicolumn{1}{c|}{$\ell_\infty$} & 1177 & 34.2 & 1.13 \\
\multicolumn{1}{l|}{Bandit}       & Score & \multicolumn{1}{c|}{$\ell_2$}      & 999  & 25.6 & 3.78 \\
\multicolumn{1}{l|}{NES}          & Score & \multicolumn{1}{c|}{$\ell_2$}      & 1137 & 0.1  & 0.59 \\
\multicolumn{1}{l|}{Simple}       & Score & \multicolumn{1}{c|}{$\ell_2$}      & 3138 & 67.1 & 0.21 \\
\multicolumn{1}{l|}{Square}       & Score & \multicolumn{1}{c|}{$\ell_2$}      & 41   & 0.2  & 4.47 \\
\multicolumn{1}{l|}{ZOSignSGD}    & Score & \multicolumn{1}{c|}{$\ell_2$}      & 1955 & 1.6  & 0.53 \\
\multicolumn{1}{l|}{Boundary}     & Label & \multicolumn{1}{c|}{$\ell_2$}      & 57   & 52.3 & 2.14 \\
\multicolumn{1}{l|}{GeoDA}        & Label & \multicolumn{1}{c|}{$\ell_2$}      & 545  & 52.7 & 2.69 \\
\multicolumn{1}{l|}{HSJ}          & Label & \multicolumn{1}{c|}{$\ell_2$}      & 234  & 45.4 & 2.90 \\
\multicolumn{1}{l|}{Opt}          & Label & \multicolumn{1}{c|}{$\ell_2$}      & 641  & 53.3 & 2.02 \\
\multicolumn{1}{l|}{SignOPT}      & Label & \multicolumn{1}{c|}{$\ell_2$}      & 504  & 34.8 & 2.00 \\
\multicolumn{1}{l|}{PointWise}    & Label & \multicolumn{1}{c|}{Optimized}  & 1290 & 89.8 & 2.02 \\
\multicolumn{1}{l|}{SparseEvo}    & Label & \multicolumn{1}{c|}{Optimized}  & 9425 & 38.3 & 2.27 \\ \hline
\multicolumn{1}{l|}{\textbf{CA (sssp)}} &
  Label &
  \multicolumn{1}{c|}{Optimized} &
  400 &
  5.7 &
  3.75 \\
\multicolumn{1}{l|}{\textbf{CA (bin search)}} &
  Label &
  \multicolumn{1}{c|}{Optimized} &
  473 &
  0.0 &
  5.19 \\ \hline
\end{tabular}%
}
\end{table}
\begin{table}[H]
\centering
\caption{Attack performance under RAND Post-processing Defense on ResNet and CIFAR10 (Clean Accuracy: $92.0\%$)}
\label{tab:cifar10_post_RAND_resnet}
\resizebox{0.47\textwidth}{!}{%
\begin{tabular}{lccccc}
\hline
Attack &
  \begin{tabular}[c]{@{}c@{}}Query\\ Type\end{tabular} &
  \begin{tabular}[c]{@{}c@{}}Perturbation\\ Type\end{tabular} &
  \# Query &
  Model Acc. &
  Dist. $\ell_2$ \\ \hline
\multicolumn{1}{l|}{Bandit}       & Score & \multicolumn{1}{c|}{$\ell_\infty$} & 153  & 10.3 & 4.93 \\
\multicolumn{1}{l|}{NES}          & Score & \multicolumn{1}{c|}{$\ell_\infty$} & 277  & 0.0  & 1.39 \\
\multicolumn{1}{l|}{Parsimonious} & Score & \multicolumn{1}{c|}{$\ell_\infty$} & 104  & 0.2  & 5.01 \\
\multicolumn{1}{l|}{Sign}         & Score & \multicolumn{1}{c|}{$\ell_\infty$} & 144  & 0.7  & 4.98 \\
\multicolumn{1}{l|}{Square}       & Score & \multicolumn{1}{c|}{$\ell_\infty$} & 18   & 0.0  & 4.99 \\
\multicolumn{1}{l|}{ZOSignSGD}    & Score & \multicolumn{1}{c|}{$\ell_\infty$} & 274  & 0.0  & 1.35 \\
\multicolumn{1}{l|}{GeoDA}        & Label & \multicolumn{1}{c|}{$\ell_\infty$} & 380  & 60.2 & 2.71 \\
\multicolumn{1}{l|}{HSJ}          & Label & \multicolumn{1}{c|}{$\ell_\infty$} & 792  & 60.7 & 2.21 \\
\multicolumn{1}{l|}{Opt}          & Label & \multicolumn{1}{c|}{$\ell_\infty$} & 658  & 84.9 & 0.21 \\
\multicolumn{1}{l|}{RayS}         & Label & \multicolumn{1}{c|}{$\ell_\infty$} & 252  & 46.8 & 4.79 \\
\multicolumn{1}{l|}{SignFlip}     & Label & \multicolumn{1}{c|}{$\ell_\infty$} & 442  & 17.0 & 4.26 \\
\multicolumn{1}{l|}{SignOPT}      & Label & \multicolumn{1}{c|}{$\ell_\infty$} & 1095 & 42.0 & 1.00 \\
\multicolumn{1}{l|}{Bandit}       & Score & \multicolumn{1}{c|}{$\ell_2$}      & 673  & 20.7 & 3.76 \\
\multicolumn{1}{l|}{NES}          & Score & \multicolumn{1}{c|}{$\ell_2$}      & 840  & 0.0  & 0.54 \\
\multicolumn{1}{l|}{Simple}       & Score & \multicolumn{1}{c|}{$\ell_2$}      & 5021 & 64.4 & 0.30 \\
\multicolumn{1}{l|}{Square}       & Score & \multicolumn{1}{c|}{$\ell_2$}      & 19   & 0.0  & 4.55 \\
\multicolumn{1}{l|}{ZOSignSGD}    & Score & \multicolumn{1}{c|}{$\ell_2$}      & 1443 & 0.1  & 0.47 \\
\multicolumn{1}{l|}{Boundary}     & Label & \multicolumn{1}{c|}{$\ell_2$}      & 85   & 64.4 & 1.68 \\
\multicolumn{1}{l|}{GeoDA}        & Label & \multicolumn{1}{c|}{$\ell_2$}      & 545  & 57.6 & 2.54 \\
\multicolumn{1}{l|}{HSJ}          & Label & \multicolumn{1}{c|}{$\ell_2$}      & 558  & 48.5 & 3.32 \\
\multicolumn{1}{l|}{Opt}          & Label & \multicolumn{1}{c|}{$\ell_2$}      & 593  & 59.4 & 1.98 \\
\multicolumn{1}{l|}{SignOPT}      & Label & \multicolumn{1}{c|}{$\ell_2$}      & 397  & 41.3 & 1.90 \\
\multicolumn{1}{l|}{PointWise}    & Label & \multicolumn{1}{c|}{unrestricted}  & 1735 & 91.1 & 1.90 \\
\multicolumn{1}{l|}{SparseEvo}    & Label & \multicolumn{1}{c|}{unrestricted}  & 8665 & 63.2 & 2.41 \\ \hline
\multicolumn{1}{l|}{\textbf{CA(sssp)}} &
  Label &
  \multicolumn{1}{c|}{unrestricted} &
  425 &
  9.1 &
  3.90 \\
\multicolumn{1}{l|}{\textbf{CA(bin search)}} &
  Label &
  \multicolumn{1}{c|}{unrestricted} &
  421 &
  10.7 &
  4.09 \\ \hline
\end{tabular}%
}
\end{table}
\begin{table}[H]
\centering
\caption{Attack performance under RAND Post-processing Defense on ResNeXt and CIFAR10 (Clean Accuracy: $94.8\%$)}
\label{tab:cifar10_post_RAND_resnext}
\resizebox{0.47\textwidth}{!}{%
\begin{tabular}{lccccc}
\hline
Attack &
  \begin{tabular}[c]{@{}c@{}}Query\\ Type\end{tabular} &
  \begin{tabular}[c]{@{}c@{}}Perturbation\\ Type\end{tabular} &
  \# Query &
  Model Acc. &
  Dist. $\ell_2$ \\ \hline
\multicolumn{1}{l|}{Bandit}       & Score & \multicolumn{1}{c|}{$\ell_\infty$} & 141  & 4.7  & 5.08 \\
\multicolumn{1}{l|}{NES}          & Score & \multicolumn{1}{c|}{$\ell_\infty$} & 237  & 0.0  & 1.32 \\
\multicolumn{1}{l|}{Parsimonious} & Score & \multicolumn{1}{c|}{$\ell_\infty$} & 221  & 2.1  & 5.15 \\
\multicolumn{1}{l|}{Sign}         & Score & \multicolumn{1}{c|}{$\ell_\infty$} & 133  & 0.0  & 5.14 \\
\multicolumn{1}{l|}{Square}       & Score & \multicolumn{1}{c|}{$\ell_\infty$} & 23   & 0.4  & 5.14 \\
\multicolumn{1}{l|}{ZOSignSGD}    & Score & \multicolumn{1}{c|}{$\ell_\infty$} & 249  & 0.0  & 1.31 \\
\multicolumn{1}{l|}{GeoDA}        & Label & \multicolumn{1}{c|}{$\ell_\infty$} & 282  & 53.1 & 2.52 \\
\multicolumn{1}{l|}{HSJ}          & Label & \multicolumn{1}{c|}{$\ell_\infty$} & 417  & 57.4 & 2.41 \\
\multicolumn{1}{l|}{Opt}          & Label & \multicolumn{1}{c|}{$\ell_\infty$} & 683  & 81.6 & 0.36 \\
\multicolumn{1}{l|}{RayS}         & Label & \multicolumn{1}{c|}{$\ell_\infty$} & 484  & 52.5 & 4.72 \\
\multicolumn{1}{l|}{SignFlip}     & Label & \multicolumn{1}{c|}{$\ell_\infty$} & 287  & 28.1 & 3.35 \\
\multicolumn{1}{l|}{SignOPT}      & Label & \multicolumn{1}{c|}{$\ell_\infty$} & 760  & 45.3 & 0.91 \\
\multicolumn{1}{l|}{Bandit}       & Score & \multicolumn{1}{c|}{$\ell_2$}      & 574  & 14.2 & 3.69 \\
\multicolumn{1}{l|}{NES}          & Score & \multicolumn{1}{c|}{$\ell_2$}      & 694  & 0.0  & 0.50 \\
\multicolumn{1}{l|}{Simple}       & Score & \multicolumn{1}{c|}{$\ell_2$}      & 4075 & 61.7 & 0.29 \\
\multicolumn{1}{l|}{Square}       & Score & \multicolumn{1}{c|}{$\ell_2$}      & 22   & 0.0  & 4.68 \\
\multicolumn{1}{l|}{ZOSignSGD}    & Score & \multicolumn{1}{c|}{$\ell_2$}      & 1459 & 0.6  & 0.47 \\
\multicolumn{1}{l|}{Boundary}     & Label & \multicolumn{1}{c|}{$\ell_2$}      & 20   & 51.1 & 2.36 \\
\multicolumn{1}{l|}{GeoDA}        & Label & \multicolumn{1}{c|}{$\ell_2$}      & 206  & 54.5 & 2.48 \\
\multicolumn{1}{l|}{HSJ}          & Label & \multicolumn{1}{c|}{$\ell_2$}      & 414  & 48.6 & 2.84 \\
\multicolumn{1}{l|}{Opt}          & Label & \multicolumn{1}{c|}{$\ell_2$}      & 542  & 61.2 & 1.70 \\
\multicolumn{1}{l|}{SignOPT}      & Label & \multicolumn{1}{c|}{$\ell_2$}      & 357  & 44.2 & 1.50 \\
\multicolumn{1}{l|}{PointWise}    & Label & \multicolumn{1}{c|}{Optimized}  & 2519 & 94.6 & 2.05 \\
\multicolumn{1}{l|}{SparseEvo}    & Label & \multicolumn{1}{c|}{Optimized}  & 9537 & 78.4 & 2.85 \\ \hline
\multicolumn{1}{l|}{\textbf{CA (sssp)}} &
  Label &
  \multicolumn{1}{c|}{Optimized} &
  276 &
  9.3 &
  2.92 \\
\multicolumn{1}{l|}{\textbf{CA (bin search)}} &
  Label &
  \multicolumn{1}{c|}{Optimized} &
  434 &
  10.0 &
  3.05 \\ \hline
\end{tabular}%
}
\end{table}
\begin{table}[H]
\centering
\caption{Attack performance under RAND Post-processing Defense on WRN and CIFAR10 (Clean Accuracy: $96.1\%$)}
\label{tab:cifar10_post_RAND_wrn}
\resizebox{0.47\textwidth}{!}{%
\begin{tabular}{lccccc}
\hline
Attack &
  \begin{tabular}[c]{@{}c@{}}Query\\ Type\end{tabular} &
  \begin{tabular}[c]{@{}c@{}}Perturbation\\ Type\end{tabular} &
  \# Query &
  Model Acc. &
  Dist. $\ell_2$ \\ \hline
\multicolumn{1}{l|}{Bandit}            & Score & \multicolumn{1}{c|}{$\ell_\infty$} & 157  & 6.2  & 5.12 \\
\multicolumn{1}{l|}{NES}               & Score & \multicolumn{1}{c|}{$\ell_\infty$} & 464  & 0.2  & 1.59 \\
\multicolumn{1}{l|}{Parsimonious}      & Score & \multicolumn{1}{c|}{$\ell_\infty$} & 390  & 21.1 & 5.22 \\
\multicolumn{1}{l|}{Sign}              & Score & \multicolumn{1}{c|}{$\ell_\infty$} & 330  & 12.4 & 5.02 \\
\multicolumn{1}{l|}{Square}            & Score & \multicolumn{1}{c|}{$\ell_\infty$} & 36   & 0.5  & 5.21 \\
\multicolumn{1}{l|}{ZOSignSGD}         & Score & \multicolumn{1}{c|}{$\ell_\infty$} & 464  & 0.4  & 1.56 \\
\multicolumn{1}{l|}{GeoDA}             & Label & \multicolumn{1}{c|}{$\ell_\infty$} & 186  & 56.9 & 2.58 \\
\multicolumn{1}{l|}{HSJ}               & Label & \multicolumn{1}{c|}{$\ell_\infty$} & 364  & 56.0 & 2.37 \\
\multicolumn{1}{l|}{Opt}               & Label & \multicolumn{1}{c|}{$\ell_\infty$} & 1022 & 84.8 & 0.31 \\
\multicolumn{1}{l|}{RayS}              & Label & \multicolumn{1}{c|}{$\ell_\infty$} & 334  & 51.5 & 4.81 \\
\multicolumn{1}{l|}{SignFlip}          & Label & \multicolumn{1}{c|}{$\ell_\infty$} & 164  & 22.8 & 3.72 \\
\multicolumn{1}{l|}{SignOPT}           & Label & \multicolumn{1}{c|}{$\ell_\infty$} & 968  & 39.2 & 0.98 \\
\multicolumn{1}{l|}{Bandit}            & Score & \multicolumn{1}{c|}{$\ell_2$}      & 727  & 17.5 & 3.81 \\
\multicolumn{1}{l|}{NES}               & Score & \multicolumn{1}{c|}{$\ell_2$}      & 1328 & 1.6  & 0.68 \\
\multicolumn{1}{l|}{Simple}            & Score & \multicolumn{1}{c|}{$\ell_2$}      & 2536 & 74.7 & 0.21 \\
\multicolumn{1}{l|}{Square}            & Score & \multicolumn{1}{c|}{$\ell_2$}      & 26   & 0.4  & 4.75 \\
\multicolumn{1}{l|}{ZOSignSGD}         & Score & \multicolumn{1}{c|}{$\ell_2$}      & 2143 & 5.5  & 0.60 \\
\multicolumn{1}{l|}{Boundary}          & Label & \multicolumn{1}{c|}{$\ell_2$}      & 35   & 54.0 & 2.35 \\
\multicolumn{1}{l|}{GeoDA}             & Label & \multicolumn{1}{c|}{$\ell_2$}      & 267  & 58.3 & 2.53 \\
\multicolumn{1}{l|}{HSJ}               & Label & \multicolumn{1}{c|}{$\ell_2$}      & 360  & 48.8 & 2.82 \\
\multicolumn{1}{l|}{Opt}               & Label & \multicolumn{1}{c|}{$\ell_2$}      & 620  & 59.7 & 1.89 \\
\multicolumn{1}{l|}{SignOPT}           & Label & \multicolumn{1}{c|}{$\ell_2$}      & 476  & 43.4 & 1.65 \\
\multicolumn{1}{l|}{PointWise}         & Label & \multicolumn{1}{c|}{Optimized}  & 1727 & 95.8 & 1.89 \\
\multicolumn{1}{l|}{SparseEvo}         & Label & \multicolumn{1}{c|}{Optimized}  & 9720 & 59.9 & 2.67 \\ \hline
\multicolumn{1}{l|}{\textbf{CA (sssp)}} & Label & \multicolumn{1}{c|}{Optimized}  & 399  & 7.3  & 3.65 \\
\multicolumn{1}{l|}{\textbf{CA (bin search)}} &
  Label &
  \multicolumn{1}{c|}{Optimized} &
  460 &
  0.0 &
  4.94 \\ \hline
\end{tabular}%
}
\end{table}
\clearpage
\begin{table}[H]
\centering
\caption{Attack performance under RAND Post-processing Defense on VGG and CIFAR100 (Clean Accuracy: $68.5\%$)}
\label{tab:cifar100_post_RAND_vgg}
\resizebox{0.47\textwidth}{!}{%
\begin{tabular}{lccccc}
\hline
Attack &
  \begin{tabular}[c]{@{}c@{}}Query\\ Type\end{tabular} &
  \begin{tabular}[c]{@{}c@{}}Perturbation\\ Type\end{tabular} &
  \# Query &
  Model Acc. &
  Dist. $\ell_2$ \\ \hline
\multicolumn{1}{l|}{Bandit}       & Score & \multicolumn{1}{c|}{$\ell_\infty$} & 35   & 1.0  & 3.67 \\
\multicolumn{1}{l|}{NES}          & Score & \multicolumn{1}{c|}{$\ell_\infty$} & 190  & 0.2  & 0.83 \\
\multicolumn{1}{l|}{Parsimonious} & Score & \multicolumn{1}{c|}{$\ell_\infty$} & 103  & 2.0  & 3.71 \\
\multicolumn{1}{l|}{Sign}         & Score & \multicolumn{1}{c|}{$\ell_\infty$} & 72   & 0.7  & 3.66 \\
\multicolumn{1}{l|}{Square}       & Score & \multicolumn{1}{c|}{$\ell_\infty$} & 8    & 0.1  & 3.66 \\
\multicolumn{1}{l|}{ZOSignSGD}    & Score & \multicolumn{1}{c|}{$\ell_\infty$} & 203  & 0.1  & 0.82 \\
\multicolumn{1}{l|}{GeoDA}        & Label & \multicolumn{1}{c|}{$\ell_\infty$} & 177  & 35.6 & 2.31 \\
\multicolumn{1}{l|}{HSJ}          & Label & \multicolumn{1}{c|}{$\ell_\infty$} & 235  & 36.7 & 2.19 \\
\multicolumn{1}{l|}{Opt}          & Label & \multicolumn{1}{c|}{$\ell_\infty$} & 756  & 53.3 & 0.48 \\
\multicolumn{1}{l|}{RayS}         & Label & \multicolumn{1}{c|}{$\ell_\infty$} & 116  & 32.2 & 4.14 \\
\multicolumn{1}{l|}{SignFlip}     & Label & \multicolumn{1}{c|}{$\ell_\infty$} & 40   & 14.4 & 2.89 \\
\multicolumn{1}{l|}{SignOPT}      & Label & \multicolumn{1}{c|}{$\ell_\infty$} & 724  & 29.4 & 0.93 \\
\multicolumn{1}{l|}{Bandit}       & Score & \multicolumn{1}{c|}{$\ell_2$}      & 139  & 2.3  & 1.86 \\
\multicolumn{1}{l|}{NES}          & Score & \multicolumn{1}{c|}{$\ell_2$}      & 650  & 0.1  & 0.33 \\
\multicolumn{1}{l|}{Simple}       & Score & \multicolumn{1}{c|}{$\ell_2$}      & 2570 & 34.7 & 0.18 \\
\multicolumn{1}{l|}{Square}       & Score & \multicolumn{1}{c|}{$\ell_2$}      & 6    & 0.1  & 3.35 \\
\multicolumn{1}{l|}{ZOSignSGD}    & Score & \multicolumn{1}{c|}{$\ell_2$}      & 1225 & 0.5  & 0.30 \\
\multicolumn{1}{l|}{Boundary}     & Label & \multicolumn{1}{c|}{$\ell_2$}      & 40   & 28.8 & 2.31 \\
\multicolumn{1}{l|}{GeoDA}        & Label & \multicolumn{1}{c|}{$\ell_2$}      & 205  & 36.3 & 2.31 \\
\multicolumn{1}{l|}{HSJ}          & Label & \multicolumn{1}{c|}{$\ell_2$}      & 169  & 32.8 & 2.11 \\
\multicolumn{1}{l|}{Opt}          & Label & \multicolumn{1}{c|}{$\ell_2$}      & 581  & 37.7 & 1.54 \\
\multicolumn{1}{l|}{SignOPT}      & Label & \multicolumn{1}{c|}{$\ell_2$}      & 409  & 27.1 & 1.32 \\
\multicolumn{1}{l|}{PointWise}    & Label & \multicolumn{1}{c|}{Optimized}  & 494  & 67.5 & 1.38 \\
\multicolumn{1}{l|}{SparseEvo}    & Label & \multicolumn{1}{c|}{Optimized}  & 8502 & 21.1 & 1.75 \\ \hline
\multicolumn{1}{l|}{\textbf{CA (sssp)}} &
  Label &
  \multicolumn{1}{c|}{Optimized} &
  186 &
  1.8 &
  2.35 \\
\multicolumn{1}{l|}{\textbf{CA (bin search)}} &
  Label &
  \multicolumn{1}{c|}{Optimized} &
  582 &
  0.0 &
  2.67 \\ \hline
\end{tabular}%
}
\end{table}
\begin{table}[H]
\centering
\caption{Attack performance under RAND Post-processing Defense on ResNet and CIFAR100 (Clean Accuracy: $67.5\%$)}
\label{tab:cifar100_post_RAND_resnet}
\resizebox{0.47\textwidth}{!}{%
\begin{tabular}{lccccc}
\hline
Attack &
  \begin{tabular}[c]{@{}c@{}}Query\\ Type\end{tabular} &
  \begin{tabular}[c]{@{}c@{}}Perturbation\\ Type\end{tabular} &
  \# Query &
  Model Acc. &
  Dist. $\ell_2$ \\ \hline
\multicolumn{1}{l|}{Bandit}       & Score & \multicolumn{1}{c|}{$\ell_\infty$} & 46   & 1.0  & 3.57 \\
\multicolumn{1}{l|}{NES}          & Score & \multicolumn{1}{c|}{$\ell_\infty$} & 152  & 0.0  & 0.75 \\
\multicolumn{1}{l|}{Parsimonious} & Score & \multicolumn{1}{c|}{$\ell_\infty$} & 56   & 0.0  & 3.63 \\
\multicolumn{1}{l|}{Sign}         & Score & \multicolumn{1}{c|}{$\ell_\infty$} & 39   & 0.0  & 3.57 \\
\multicolumn{1}{l|}{Square}       & Score & \multicolumn{1}{c|}{$\ell_\infty$} & 7    & 0.0  & 3.61 \\
\multicolumn{1}{l|}{ZOSignSGD}    & Score & \multicolumn{1}{c|}{$\ell_\infty$} & 159  & 0.0  & 0.72 \\
\multicolumn{1}{l|}{GeoDA}        & Label & \multicolumn{1}{c|}{$\ell_\infty$} & 245  & 38.7 & 2.49 \\
\multicolumn{1}{l|}{HSJ}          & Label & \multicolumn{1}{c|}{$\ell_\infty$} & 407  & 38.2 & 2.28 \\
\multicolumn{1}{l|}{Opt}          & Label & \multicolumn{1}{c|}{$\ell_\infty$} & 762  & 51.6 & 0.47 \\
\multicolumn{1}{l|}{RayS}         & Label & \multicolumn{1}{c|}{$\ell_\infty$} & 99   & 32.7 & 4.03 \\
\multicolumn{1}{l|}{SignFlip}     & Label & \multicolumn{1}{c|}{$\ell_\infty$} & 47   & 17.5 & 3.25 \\
\multicolumn{1}{l|}{SignOPT}      & Label & \multicolumn{1}{c|}{$\ell_\infty$} & 839  & 28.2 & 0.97 \\
\multicolumn{1}{l|}{Bandit}       & Score & \multicolumn{1}{c|}{$\ell_2$}      & 211  & 3.2  & 1.86 \\
\multicolumn{1}{l|}{NES}          & Score & \multicolumn{1}{c|}{$\ell_2$}      & 431  & 0.0  & 0.28 \\
\multicolumn{1}{l|}{Simple}       & Score & \multicolumn{1}{c|}{$\ell_2$}      & 3593 & 36.0 & 0.20 \\
\multicolumn{1}{l|}{Square}       & Score & \multicolumn{1}{c|}{$\ell_2$}      & 6    & 0.0  & 3.28 \\
\multicolumn{1}{l|}{ZOSignSGD}    & Score & \multicolumn{1}{c|}{$\ell_2$}      & 791  & 0.0  & 0.24 \\
\multicolumn{1}{l|}{Boundary}     & Label & \multicolumn{1}{c|}{$\ell_2$}      & 56   & 34.2 & 2.35 \\
\multicolumn{1}{l|}{GeoDA}        & Label & \multicolumn{1}{c|}{$\ell_2$}      & 387  & 35.6 & 2.38 \\
\multicolumn{1}{l|}{HSJ}          & Label & \multicolumn{1}{c|}{$\ell_2$}      & 197  & 33.2 & 2.40 \\
\multicolumn{1}{l|}{Opt}          & Label & \multicolumn{1}{c|}{$\ell_2$}      & 633  & 40.2 & 1.60 \\
\multicolumn{1}{l|}{SignOPT}      & Label & \multicolumn{1}{c|}{$\ell_2$}      & 413  & 27.0 & 1.53 \\
\multicolumn{1}{l|}{PointWise}    & Label & \multicolumn{1}{c|}{Optimized}  & 922  & 65.7 & 1.41 \\
\multicolumn{1}{l|}{SparseEvo}    & Label & \multicolumn{1}{c|}{Optimized}  & 9028 & 41.0 & 2.02 \\ \hline
\multicolumn{1}{l|}{\textbf{CA (sssp)}} &
  Label &
  \multicolumn{1}{c|}{Optimized} &
  187 &
  1.8 &
  2.36 \\
\multicolumn{1}{l|}{\textbf{CA (bin search)}} &
  Label &
  \multicolumn{1}{c|}{Optimized} &
  458 &
  0.0 &
  3.04 \\ \hline
\end{tabular}%
}
\end{table}
\begin{table}[H]
\centering
\caption{Attack performance under RAND Post-processing Defense on ResNeXt and CIFAR100 (Clean Accuracy: $79.5\%$)}
\label{tab:cifar100_post_RAND_resnext}
\resizebox{0.47\textwidth}{!}{%
\begin{tabular}{lccccc}
\hline
Attack &
  \begin{tabular}[c]{@{}c@{}}Query\\ Type\end{tabular} &
  \begin{tabular}[c]{@{}c@{}}Perturbation\\ Type\end{tabular} &
  \# Query &
  Model Acc. &
  Dist. $\ell_2$ \\ \hline
\multicolumn{1}{l|}{Bandit}       & Score & \multicolumn{1}{c|}{$\ell_\infty$} & 46   & 1.3  & 4.28 \\
\multicolumn{1}{l|}{NES}          & Score & \multicolumn{1}{c|}{$\ell_\infty$} & 165  & 0.0  & 0.92 \\
\multicolumn{1}{l|}{Parsimonious} & Score & \multicolumn{1}{c|}{$\ell_\infty$} & 117  & 0.5  & 4.33 \\
\multicolumn{1}{l|}{Sign}         & Score & \multicolumn{1}{c|}{$\ell_\infty$} & 124  & 0.1  & 4.29 \\
\multicolumn{1}{l|}{Square}       & Score & \multicolumn{1}{c|}{$\ell_\infty$} & 8    & 0.1  & 4.32 \\
\multicolumn{1}{l|}{ZOSignSGD}    & Score & \multicolumn{1}{c|}{$\ell_\infty$} & 179  & 0.0  & 0.93 \\
\multicolumn{1}{l|}{GeoDA}        & Label & \multicolumn{1}{c|}{$\ell_\infty$} & 147  & 40.7 & 2.04 \\
\multicolumn{1}{l|}{HSJ}          & Label & \multicolumn{1}{c|}{$\ell_\infty$} & 201  & 41.4 & 1.91 \\
\multicolumn{1}{l|}{Opt}          & Label & \multicolumn{1}{c|}{$\ell_\infty$} & 744  & 64.2 & 0.52 \\
\multicolumn{1}{l|}{RayS}         & Label & \multicolumn{1}{c|}{$\ell_\infty$} & 179  & 45.7 & 4.32 \\
\multicolumn{1}{l|}{SignFlip}     & Label & \multicolumn{1}{c|}{$\ell_\infty$} & 98   & 29.2 & 2.47 \\
\multicolumn{1}{l|}{SignOPT}      & Label & \multicolumn{1}{c|}{$\ell_\infty$} & 844  & 43.8 & 0.86 \\
\multicolumn{1}{l|}{Bandit}       & Score & \multicolumn{1}{c|}{$\ell_2$}      & 163  & 3.1  & 2.03 \\
\multicolumn{1}{l|}{NES}          & Score & \multicolumn{1}{c|}{$\ell_2$}      & 461  & 0.0  & 0.34 \\
\multicolumn{1}{l|}{Simple}       & Score & \multicolumn{1}{c|}{$\ell_2$}      & 3805 & 45.2 & 0.23 \\
\multicolumn{1}{l|}{Square}       & Score & \multicolumn{1}{c|}{$\ell_2$}      & 10   & 0.0  & 3.90 \\
\multicolumn{1}{l|}{ZOSignSGD}    & Score & \multicolumn{1}{c|}{$\ell_2$}      & 1054 & 0.0  & 0.33 \\
\multicolumn{1}{l|}{Boundary}     & Label & \multicolumn{1}{c|}{$\ell_2$}      & 20   & 36.6 & 2.15 \\
\multicolumn{1}{l|}{GeoDA}        & Label & \multicolumn{1}{c|}{$\ell_2$}      & 143  & 44.3 & 2.02 \\
\multicolumn{1}{l|}{HSJ}          & Label & \multicolumn{1}{c|}{$\ell_2$}      & 256  & 40.6 & 2.13 \\
\multicolumn{1}{l|}{Opt}          & Label & \multicolumn{1}{c|}{$\ell_2$}      & 590  & 49.5 & 1.39 \\
\multicolumn{1}{l|}{SignOPT}      & Label & \multicolumn{1}{c|}{$\ell_2$}      & 371  & 38.0 & 1.56 \\
\multicolumn{1}{l|}{PointWise}    & Label & \multicolumn{1}{c|}{Optimized}  & 1811 & 78.7 & 1.75 \\
\multicolumn{1}{l|}{SparseEvo}    & Label & \multicolumn{1}{c|}{Optimized}  & 9476 & 62.6 & 2.20 \\ \hline
\multicolumn{1}{l|}{\textbf{CA (sssp)}} &
  Label &
  \multicolumn{1}{c|}{Optimized} &
  179 &
  2.0 &
  2.33 \\
\multicolumn{1}{l|}{\textbf{CA (bin search)}} &
  Label &
  \multicolumn{1}{c|}{Optimized} &
  458 &
  0.0 &
  2.60 \\ \hline
\end{tabular}%
}
\end{table}
\begin{table}[H]
\centering
\caption{Attack performance under RAND Post-processing Defense on WRN and CIFAR100 (Clean Accuracy: $79.4\%$)}
\label{tab:cifar100_post_RAND_wrn}
\resizebox{0.47\textwidth}{!}{%
\begin{tabular}{lccccc}
\hline
Attack &
  \begin{tabular}[c]{@{}c@{}}Query\\ Type\end{tabular} &
  \begin{tabular}[c]{@{}c@{}}Perturbation\\ Type\end{tabular} &
  \# Query &
  Model Acc. &
  Dist. $\ell_2$ \\ \hline
\multicolumn{1}{l|}{Bandit}       & Score & \multicolumn{1}{c|}{$\ell_\infty$} & 35   & 1.8  & 4.26 \\
\multicolumn{1}{l|}{NES}          & Score & \multicolumn{1}{c|}{$\ell_\infty$} & 170  & 0.0  & 0.89 \\
\multicolumn{1}{l|}{Parsimonious} & Score & \multicolumn{1}{c|}{$\ell_\infty$} & 134  & 1.1  & 4.26 \\
\multicolumn{1}{l|}{Sign}         & Score & \multicolumn{1}{c|}{$\ell_\infty$} & 81   & 0.6  & 4.24 \\
\multicolumn{1}{l|}{Square}       & Score & \multicolumn{1}{c|}{$\ell_\infty$} & 10   & 0.0  & 4.23 \\
\multicolumn{1}{l|}{ZOSignSGD}    & Score & \multicolumn{1}{c|}{$\ell_\infty$} & 178  & 0.1  & 0.89 \\
\multicolumn{1}{l|}{GeoDA}        & Label & \multicolumn{1}{c|}{$\ell_\infty$} & 146  & 40.9 & 2.22 \\
\multicolumn{1}{l|}{HSJ}          & Label & \multicolumn{1}{c|}{$\ell_\infty$} & 224  & 42.3 & 2.03 \\
\multicolumn{1}{l|}{Opt}          & Label & \multicolumn{1}{c|}{$\ell_\infty$} & 738  & 62.1 & 0.46 \\
\multicolumn{1}{l|}{RayS}         & Label & \multicolumn{1}{c|}{$\ell_\infty$} & 168  & 41.9 & 4.22 \\
\multicolumn{1}{l|}{SignFlip}     & Label & \multicolumn{1}{c|}{$\ell_\infty$} & 50   & 24.1 & 2.70 \\
\multicolumn{1}{l|}{SignOPT}      & Label & \multicolumn{1}{c|}{$\ell_\infty$} & 761  & 36.2 & 0.89 \\
\multicolumn{1}{l|}{Bandit}       & Score & \multicolumn{1}{c|}{$\ell_2$}      & 172  & 3.6  & 2.10 \\
\multicolumn{1}{l|}{NES}          & Score & \multicolumn{1}{c|}{$\ell_2$}      & 523  & 0.1  & 0.34 \\
\multicolumn{1}{l|}{Simple}       & Score & \multicolumn{1}{c|}{$\ell_2$}      & 2954 & 37.9 & 0.21 \\
\multicolumn{1}{l|}{Square}       & Score & \multicolumn{1}{c|}{$\ell_2$}      & 15   & 0.0  & 3.89 \\
\multicolumn{1}{l|}{ZOSignSGD}    & Score & \multicolumn{1}{c|}{$\ell_2$}      & 1009 & 0.7  & 0.32 \\
\multicolumn{1}{l|}{Boundary}     & Label & \multicolumn{1}{c|}{$\ell_2$}      & 25   & 35.6 & 2.34 \\
\multicolumn{1}{l|}{GeoDA}        & Label & \multicolumn{1}{c|}{$\ell_2$}      & 179  & 42.8 & 2.21 \\
\multicolumn{1}{l|}{HSJ}          & Label & \multicolumn{1}{c|}{$\ell_2$}      & 214  & 38.0 & 2.18 \\
\multicolumn{1}{l|}{Opt}          & Label & \multicolumn{1}{c|}{$\ell_2$}      & 597  & 46.7 & 1.41 \\
\multicolumn{1}{l|}{SignOPT}      & Label & \multicolumn{1}{c|}{$\ell_2$}      & 397  & 35.0 & 1.37 \\
\multicolumn{1}{l|}{PointWise}    & Label & \multicolumn{1}{c|}{Optimized}  & 1145 & 77.1 & 1.50 \\
\multicolumn{1}{l|}{SparseEvo}    & Label & \multicolumn{1}{c|}{Optimized}  & 9145 & 49.3 & 1.94 \\ \hline
\multicolumn{1}{l|}{\textbf{CA (sssp)}} &
  Label &
  \multicolumn{1}{c|}{Optimized} &
  210 &
  1.6 &
  2.53 \\
\multicolumn{1}{l|}{\textbf{CA (bin search)}} &
  Label &
  \multicolumn{1}{c|}{Optimized} &
  457 &
  0.0 &
  2.65 \\ \hline
\end{tabular}%
}
\end{table}
\clearpage

\subsection{Experiments on Audio Classification task}
\label{apd:audio}
\noindent \textbf{Dataset}. The speaker verification task focuses on determining if a given voice sample belongs to a specific individual or ascribed identity~\cite{campbell1997speaker}. The process for this task entails comparing two voice samples using a speaker verification model and making a decision. We utilize the large-scale multi-speaker corpus LibriSpeech (``train-clean-100'' set), which comprises over 100 hours of read English voices from 251 speakers encompassing various accents, occupations, and age groups. Within this corpus, each speaker has multiple voice samples spanning from several seconds to tens of seconds, sampled at a rate of 16kHz. 

\vspace{0.05in}

\noindent \textbf{Model and Setting}. For the speaker verification task, we use two SOTA speaker verification models: 
ECAPA-TDNN~\cite{desplanques2020ecapa} pre-trained by Speechbrain~\cite{speechbrain} as the target model for the model owner and the X-vector model~\cite{garcia2019x} pre-trained by Speechbrain~\cite{speechbrain} as the feature extractor. In the experiments, we utilize the pre-trained SOTA model, which is trained on the VoxCeleb dataset~\cite{Nagrani17} and VoxCeleb2 dataset~\cite{Chung18b}, to perform the task of verifying if two voice samples are from the same speaker. We evaluate the performance of the model on 500 pairs of voice samples randomly selected from the LibriSpeech ``train-clean-100'' set. Each sample pair consists of two voice samples from a single speaker.

\begin{table}[!h]
 \small
  \caption{Performance of certifiable attack with Gaussian Noise on Audio Dataset. ($p=90\%$)}\vspace{-0.1in}
    \small
    \centering
    \resizebox{0.47\textwidth}{!}{%
    \begin{tabular}{c| c c c c}
    \hline
          $\sigma$ & Dist. $\ell_2$ &Mean Dist. $\ell_2$ & \# RPQ & Certified Acc.  \\
    \hline
                 0.05 & 18.35 & 12.67 & 5.42 & 100.00\% \\
                 0.1 & 26.71 & 12.70 & 3.72 & 100.00\% \\
                 0.15 & 35.82 & 12.63 & 3.12 & 100.00\% \\
    \hline

    \end{tabular}}
    \vspace{-.2in}
    \label{tab:diff variance audio}
\end{table}

\begin{table}[!h]
\small
  \caption{Performance of certifiable attack with Gaussian noise  $\sigma=0.25$ on LibriSpeech}\vspace{-0.1in}\small
    \centering
    \resizebox{0.47\textwidth}{!}{%
    \begin{tabular}{c|c c c c}
    \hline
          p & Dist. $\ell_2$ &Mean Dist. $\ell_2$ & \# RPQ & Certified Acc.  \\
    \hline
                50\% & 26.64 & 12.55 & 5.31 & 100.00\% \\
                60\% & 26.65 & 12.57 & 5.08 & 100.00\% \\
                70\% & 26.66 & 12.60 & 4.87 & 100.00\% \\
                80\% & 26.68 & 12.63 & 4.49 & 100.00\% \\
                90\% & 26.71 & 12.70 & 3.72 & 100.00\% \\
                95\% & 26.74 & 12.78 & 3.05 & 100.00\% \\
                                                     
    \hline
    \end{tabular}}
  
    \label{tab:diff p audio}
\end{table}

\begin{table}[H]
    \centering
        \caption{Attack performance of different localization/refinement algorithms on LibriSpeec ($\sigma=0.25$, $p=90\%$).}
    \small
   \resizebox{0.47\textwidth}{!}{%
    \begin{tabular}{c |c| c c c c}
    \hline
        Localization & Shifting & Dist. $\ell_2$ &Mean Dist. $\ell_2$ & \# RPQ & Cert. Acc.\\
    \hline
        random & none & 165.56 & 164.59 & 1.00 & 100.00\% \\
        random & geo. & 159.21 & 157.53 & 73.88 & 100.00\% \\
        SSSP & none & 26.72 & 12.74 & 1.32 & 100.00\% \\
        SSSP & geo. & 26.71 & 12.70 & 3.72 & 100.00\% \\
    \hline
    \end{tabular}}

    \label{tab:ablation study audio}
\end{table}

\end{document}